\documentclass{article}
\usepackage{iclr2027_conference,times}
\usepackage{amsmath,amssymb,amsthm,booktabs,array}
\usepackage{xurl,hyperref}
\PassOptionsToPackage{table}{xcolor}
\usepackage[most]{tcolorbox}
\usepackage{listings}
\usepackage[scaled=0.95]{inconsolata}

\usepackage[table]{xcolor}
\usepackage{adjustbox}
\usepackage{makecell}
\usepackage{graphicx}
\usepackage{float}
\usepackage{wrapfig}
\usepackage{caption}
\usepackage{microtype}
\usepackage{algorithm}
\usepackage{algpseudocode}

\definecolor{oursblue}{HTML}{F3F7FB}
\definecolor{headergray}{HTML}{F7F7F7}

\newcommand{\revised}[1]{{#1}}
\newcommand{\best}[1]{\textbf{#1}}
\newcommand{\second}[1]{\underline{#1}}
\usepackage{multirow}
\newcommand{\sd}[1]{{\scriptsize$\pm$#1}}
\newcommand{\up}{$\uparrow$}
\newcommand{\dn}{$\downarrow$}

\definecolor{PromptInk}{HTML}{263745}
\definecolor{PromptMute}{HTML}{566876}
\definecolor{PromptPaper}{HTML}{FCFDFE}
\definecolor{PromptHeader}{HTML}{F0F3F5}
\definecolor{PromptEdge}{HTML}{DCE2E7}
\newcommand{\promptheading}[2]{\textcolor{PromptMute}{\footnotesize\sffamily\bfseries #1}\hspace{0.85em}{\sffamily\bfseries #2}}
\newtcolorbox{promptbox}[2][]{%
 enhanced,breakable,lines before break=3,
 colback=PromptPaper,colframe=PromptEdge,boxrule=0.4pt,arc=0.7mm,
 left=9pt,right=9pt,top=6pt,bottom=7pt,
 before skip=10pt,after skip=8pt,
 title={#2},title after break={#2\hfill{\footnotesize\mdseries\itshape (continued)}},fonttitle=\small\sffamily\bfseries,
 coltitle=PromptInk,colbacktitle=PromptHeader,
 titlerule=0.3pt,
 #1}
\newtcolorbox{contractbox}[2][]{%
 enhanced,breakable,lines before break=3,
 colback=white,frame hidden,
 borderline north={0.7pt}{0pt}{PromptEdge},
 borderline south={0.5pt}{0pt}{PromptEdge},
 left=8pt,right=8pt,top=4pt,bottom=4pt,
 before skip=10pt,after skip=8pt,
 title={#2},title after break={#2\hfill{\footnotesize\mdseries\itshape (continued)}},fonttitle=\small\sffamily\bfseries,
 coltitle=PromptInk,colbacktitle=white,
 #1}
\lstdefinestyle{prompttext}{%
 basicstyle=\fontsize{8.6}{10.2}\selectfont\fontfamily{zi4}\selectfont,
 breaklines=true,breakatwhitespace=false,
 columns=fullflexible,keepspaces=true,showstringspaces=false,
 aboveskip=2.5pt,belowskip=3pt, xleftmargin=0pt,
}
\newcommand{\promptrole}[1]{%
 \par\addvspace{5pt}\noindent{\footnotesize\sffamily\bfseries\color{PromptMute}#1}\par\nobreak\vspace{1pt}}
\hypersetup{hidelinks}
\newtheorem{theorem}{Theorem}
\newtheorem{proposition}[theorem]{Proposition}
\newtheorem{corollary}[theorem]{Corollary}
\newtheorem{lemma}[theorem]{Lemma}

\newcommand{\Var}{\operatorname{Var}}
\newcommand{\Cov}{\operatorname{Cov}}
\newcommand{\STOP}{\mathtt{STOP}}

\newcommand{\pos}[1]{\left(#1\right)_{+}}

\title{UpliftMem: Learning Set-Level Uplift for Agent Memory Retrieval}

\author{Mengkun Liang, Haoran Qiang, Guannan Liu\thanks{Corresponding author: \texttt{liugn@buaa.edu.cn}.},  Junjie Wu \\
MIIT Key Laboratory of Data and Decision Intelligence \\
Beihang University \\
Beijing, China \\
\texttt{\{mengkun,BY2308105,liugn,wujj\}@buaa.edu.cn}
}
\iclrfinalcopy

\begin{document}
\maketitle
\lhead{Under review as a conference paper at ICLR 2027}
\begin{abstract}
\revised{Large language model (LLM) agents draw on external memory to guide new tasks, raising the question of which memories to retrieve together to improve execution. 
Learning set values requires outcome feedback, but evaluating each untested alternative costs an additional rollout. We introduce \textsc{UpliftMem}, which learns set-level execution uplift relative to the same executor without memory. Our analysis of how retrieval preferences limit feedback coverage motivates targeted probes. Under a correlated Gaussian surrogate, we derive a closed-form expected value of sample information (EVSI) criterion that prioritizes probes by their expected benefit to local retrieval decisions. The feedback trains a shared scorer with the executor frozen, and inference requires no probing. 
Across ALFWorld, WebShop, and BigCodeBench, \textsc{UpliftMem} achieves best or joint-best success rates among evaluated methods. On fixed-store MemSyco-Bench evaluations, decision accuracy increases from 41.7\% to 66.5\% with Qwen3-4B and from 44.4\% to 68.8\% with Qwen3-8B. At matched probe budgets, EVSI improves decision accuracy over random probing by 10.6 and 13.9 percentage points, respectively.}
\end{abstract}

\section{Introduction}
\label{sec:introduction}

\begingroup
\setlength{\intextsep}{6pt}
\setlength{\columnsep}{12pt}
\begin{wrapfigure}{r}{0.35\textwidth}
    \centering
    \includegraphics[width=\linewidth]{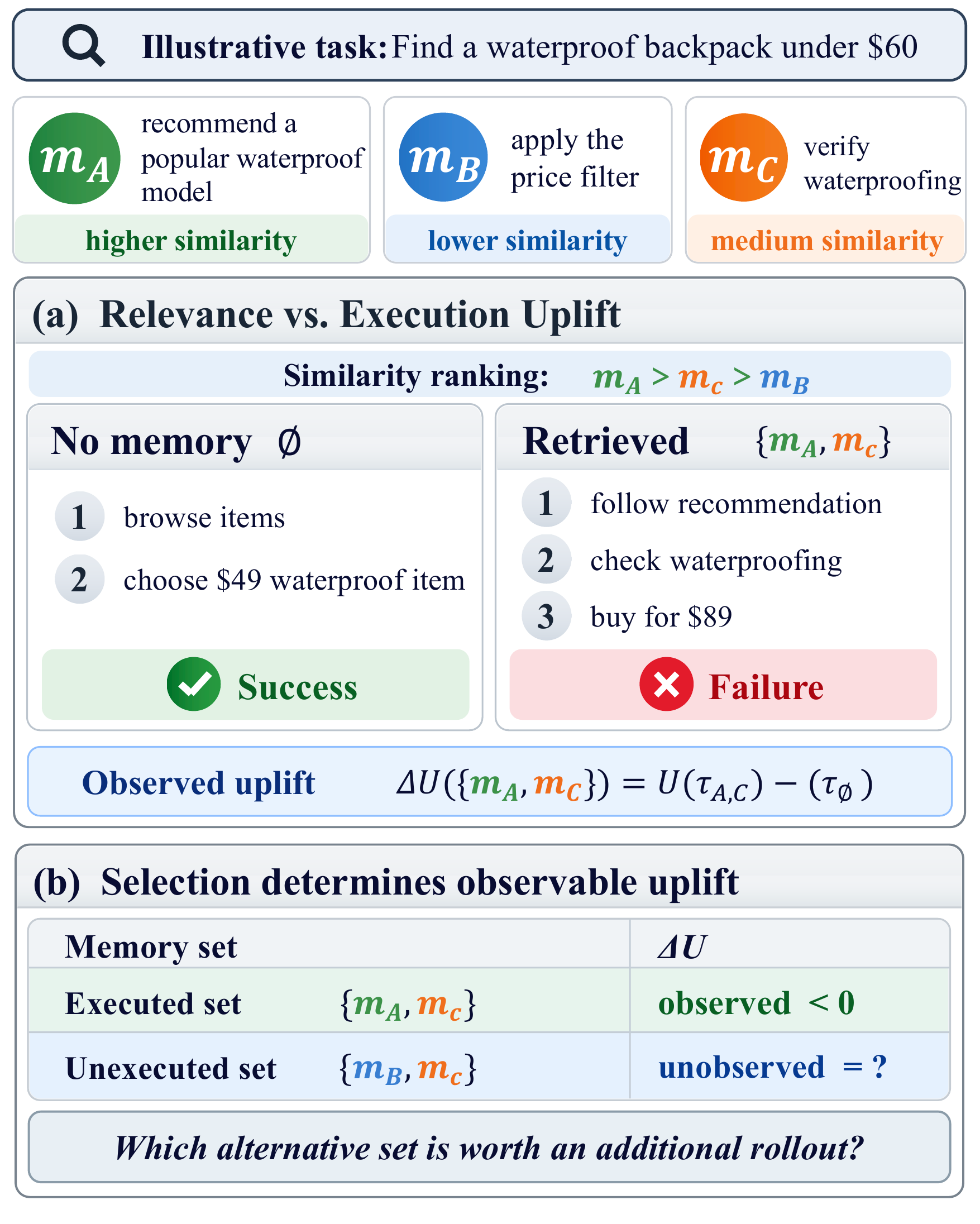}
    \captionsetup{font=scriptsize,labelfont=normalfont,labelsep=period,skip=2pt}
    \caption{Execution-based memory retrieval. (a) Relevant memories may reduce execution utility. (b) Alternative sets require additional rollouts to observe their uplift.}
    \label{fig:intro_memory}
\end{wrapfigure}

LLM-based agents retrieve past experience to guide new tasks \citep{shinn2023reflexion,zhao2024expel}, but relevant memories do not necessarily improve execution. Figure~\ref{fig:intro_memory}(a) illustrates this mismatch: when asked to find a waterproof backpack under \$60, the agent succeeds without memory by choosing a \$49 item, yet follows a retrieved recommendation to purchase an \$89 item. \revised{A less similar candidate memory directly addresses the budget constraint by suggesting a price filter.} This failure motivates learning retrieval from execution feedback, with supervision that reflects whether supplying memory improves the same agent's execution on the same task.

Building on outcome-aware and uplift-based retrieval \citep{zhang2026memrl,qu2025upliftrag,tu2026d2skill}, we define execution uplift as the expected utility difference between execution with complete memory and execution without memory under the same executor. Observed utility differences, which account for both task performance and execution cost, supply the signal, distinguishing additional memory benefit from success attributable to the executor alone. Memories may provide complementary, redundant, or conflicting guidance \citep{xu2025scarlet}. A memory's effect on execution can thus depend on the other memories used alongside it. To capture these interactions, we model execution uplift at the level of complete memory sets.

Obtaining this supervision, however, requires costly agent executions. Each new label requires executing a candidate memory set and comparing its utility with a no-memory reference. Ordinary retrieval supplies direct feedback only for the set it executes. In Figure~\ref{fig:intro_memory}(b), the negative uplift observed for $\{m_A,m_C\}$ provides no direct label for $\{m_B,m_C\}$. In other words, \emph{if a memory set is never selected, its actual execution uplift remains unverified by direct observations.} \revised{Retrieval and supervision thus form a feedback loop: current estimates influence which sets are executed, and the resulting observations are used to update those estimates.} Consequently, an underestimated set may remain rarely selected and receive little feedback to correct that estimate. Obtaining labels for untested sets requires additional rollouts, but the number of feasible combinations makes exhaustive labeling impractical under a limited budget.

\revised{The goal is therefore to learn effective memory retrieval by allocating a limited rollout budget to observations expected to improve subsequent retrieval decisions.} To address this problem, we introduce \textsc{UpliftMem}, which couples set-level uplift learning with execution-feedback acquisition under a limited rollout budget. Our theoretical analysis characterizes how retrieval preferences restrict direct-feedback coverage. This finding motivates additional training executions, called probes, that provide feedback for alternative memory sets ordinary retrieval may leave untested. \revised{Even when a set's uplift is uncertain, probing it has little decision value if the observation is unlikely to change which set is preferred.} For local memory-set comparisons, the analysis further yields a closed-form expression for the expected value of sample information (EVSI) under a correlated Gaussian surrogate \citep{frazier2009knowledge}. The resulting decision-aware acquisition rule combines resolvable uncertainty with predicted utility differences to prioritize observations expected to improve retrieval decisions. Correlations in the surrogate allow one observation to inform multiple set comparisons. The acquired outcomes then supervise a shared value model that evaluates each candidate addition in the context already selected, while no-memory execution serves as both the reference and an explicit alternative. 

With executors backed by Qwen3-4B and Qwen3-8B \citep{yang2025qwen3}, \textsc{UpliftMem} achieves the \revised{best or joint-best} success rates among the evaluated systems on both ALFWorld splits \citep{shridhar2021alfworld}, WebShop \citep{yao2022webshop}, and the overall BigCodeBench evaluation set \citep{zhuo2025bigcodebench}. Fixed-store evaluations on MemSyco-Bench \citep{xiang2026memsyco} isolate the contribution of retrieval by holding memory content and the executor unchanged. Relative to native retrieval, decision accuracy averaged over five memory stores and the three scenarios assessing when to use memory increases from 41.7\% to 66.5\% with Qwen3-4B and from 44.4\% to 68.8\% with Qwen3-8B. \revised{On the same metric, set-level valuation improves over independent scoring by 9.0 and 8.8 percentage points for the two executors, respectively. At matched probe budgets, EVSI improves over random probing by 10.6 and 13.9 percentage points, respectively.}

\par
\endgroup

\section{Related Work}
\label{sec:related_work}


\subsection{Experience-Based Agent Memory}

Agent memory has evolved from retaining execution feedback and trajectories to distilling reusable procedural and reasoning guidance \citep{shinn2023reflexion,zhao2024expel,wang2024voyager,fang2026memp,ouyang2026reasoningbank}. Recent systems further organize, revise, and selectively maintain stored experience using later interactions or learned memory-management policies \citep{xu2025amem,fang2025lightmem,zhou2025memento,yan2026memoryr1}. As memory becomes an adaptive component of the agent, its benefit increasingly depends on which parts of accumulated experience are supplied to the executor for a particular task.


\subsection{Learning to Retrieve Agent Memory}

Memory retrieval has correspondingly moved beyond semantic relevance toward downstream usefulness. Reasoning-aware and utility-based methods learn which memories are likely to improve task performance, while uplift-style and paired-execution approaches estimate contribution relative to behavior without external guidance \citep{li2026memreranker,zhang2026memrl,qu2025upliftrag,tu2026d2skill}. When several memories are used together, their effects can also depend on the surrounding context, motivating shared-context attribution and joint adaptation of retrieved knowledge \citep{xu2025scarlet,li2026skillaligner}.

Outcome-based retrieval further changes how training evidence is obtained. Direct outcome labels are available only for memory sets that are executed, while evaluating alternatives requires additional rollouts. Related work on counterfactual learning and active acquisition studies how selectively observed or costly feedback should be used for learning and decision making \citep{swaminathan2015counterfactual,roy2001optimal,houlsby2011bayesian,kirsch2019batchbald,frazier2009knowledge}. UpliftMem formulates these coupled challenges as a learning problem, where complete memory sets are valued by execution uplift and additional rollouts are allocated according to their expected value for improving retrieval decisions.

\section{Preliminaries}
\label{sec:preliminaries}

\subsection{Interactive Agent with External Memory}

Following prior memory-augmented agent architectures~\citep{sumers2024coala,zhao2024expel,wang2024voyager}, we consider an interactive agent comprising a memory controller and an executor. The controller manages an external pool of reusable experience through memory retrieval and updates, while the executor interacts with the environment to complete tasks. Within the controller, the retriever selects memory sets to support execution.

At episode $t$, the retriever receives a task context $c$ and a memory pool $\mathcal M_t$ of self-contained experience units $m$. It selects a subset $S\subseteq\mathcal M_t$ as auxiliary context for the executor. The controller then uses the resulting trajectory and execution feedback to write new memory candidates, revise them, and determine their admission to the next pool $\mathcal M_{t+1}$.

\subsection{Problem Formulation}
\label{sec:problem}
Task performance alone cannot distinguish equally successful agent executions with substantially different resource requirements \citep{kapoor2024agents}. We therefore incorporate execution cost to make resource efficiency an explicit component of utility. For an execution trajectory $\tau$, let $R(\tau)\in[0,1]$ denote task performance and $C(\tau)\in[0,C_{\max}]$ denote the cost assigned by the benchmark protocol, with $C_{\max}>0$ as the cap. We define execution utility as
\begin{equation}
U(\tau)=R(\tau)\left[
1+\gamma\mathbb I[R(\tau)\geq\rho]
\left(1-\frac{C(\tau)}{C_{\max}}\right)
\right],
\label{eq:task_utility}
\end{equation}
where $\mathbb I[\cdot]$ is the indicator function, $\gamma\geq0$ weights the efficiency bonus, and $\rho\in[0,1]$ sets the performance threshold. Performance determines the base utility, and lower cost increases the bonus once the threshold is met. 


For task $c$ executed with memory set $S$, the expected utility is $V(c,S)=\mathbb E[U(\tau(c,S))]$. The retriever uses a learnable set-value model $G_\theta(c,S)$ to predict the expected utility gain relative to no-memory execution. Its output $S_\theta(c,\mathcal M_t)$ satisfies the retrieval constraints, with the empty set allowed. For a given memory pool $\mathcal M_t$, the downstream learning objective is
\begin{equation}
\theta^*
\in\arg\max_\theta
\mathbb E\left[
V\bigl(c,S_\theta(c,\mathcal M_t)\bigr)
\right].
\label{eq:retrieval_objective}
\end{equation}

\section{Methodology}
\label{sec:method}

As shown in Figure~\ref{fig:overview}, our trainable retriever is parameterized by $G_\theta(c,S)$, which scores complete memory sets by their predicted execution uplift. Similarity-based recall supplies candidate memories, and conditional gains computed from $G_\theta$ guide sequential set construction. Because set values are learned from execution outcomes, retrieval also determines which memory sets receive direct feedback, leaving disfavored alternatives poorly observed unless additional rollouts are performed. Under a limited rollout budget, \textsc{UpliftMem} therefore uses expected value of sample information (EVSI) to probe alternative sets whose outcomes are expected to improve retrieval decisions.

The retrieved set and selected probes are executed by the same frozen agent and compared with no-memory execution to obtain uplift observations. Retrieved executions supervise both set value and the sampled retrieval path, whereas probes provide additional value observations. After value-only warm-up, these signals jointly update the shared uplift model. Completed trajectories are also distilled into memory candidates, which are validated and revised when needed before admission to the memory pool. Complete training and inference procedure is shown in Algorithm~\ref{alg:training_flow}.

\begin{figure}[t]
\centering
\includegraphics[width=0.9\linewidth]{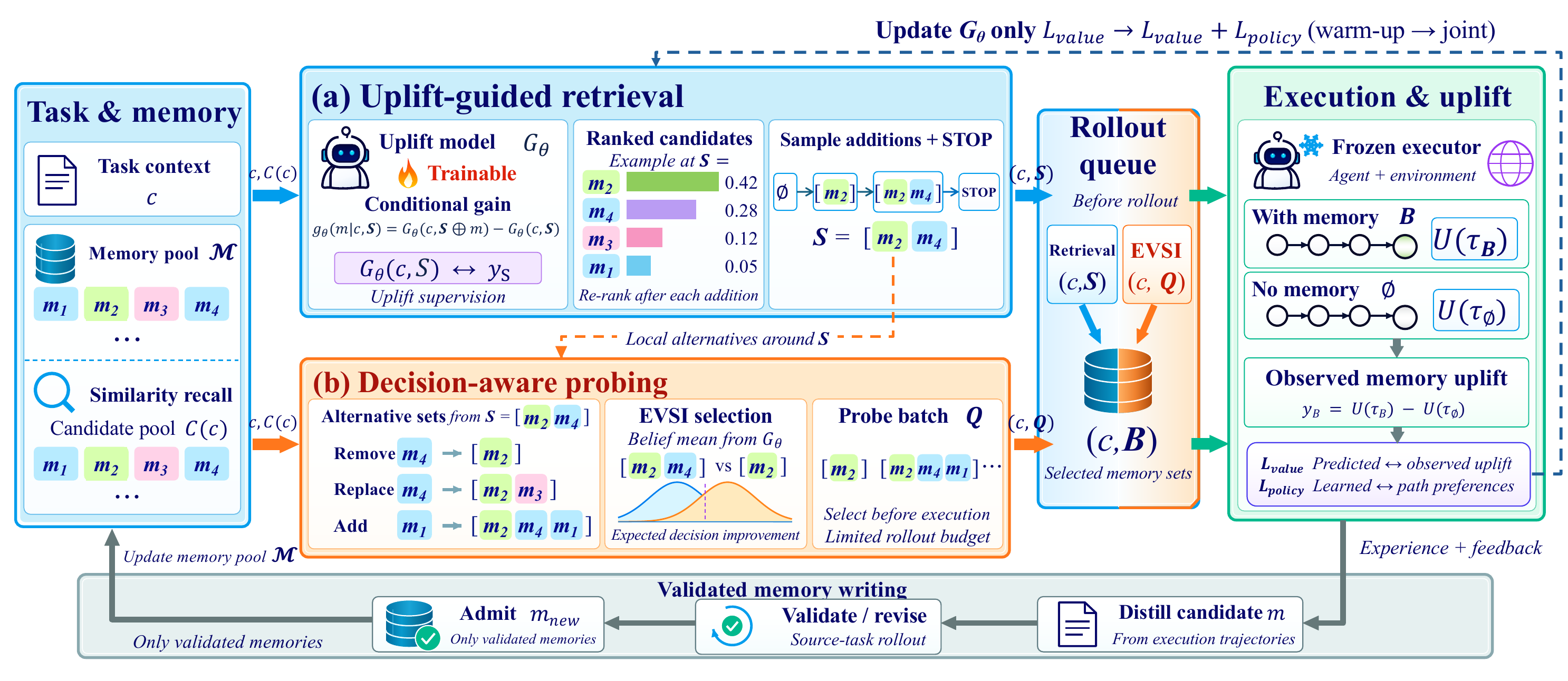}
\caption{Overview of \textsc{UpliftMem}. Shared set-value scoring guides retrieval and EVSI probing, with validated memory writing supplying additional feedback.}
\label{fig:overview}
\end{figure}

\subsection{Decision Value of Memory Sets}
\label{sec:uplift}


Motivated by outcome-aware retrieval methods that measure the benefit of external content relative to model-only behavior \citep{qu2025upliftrag,tu2026d2skill}, we use the same executor without external memory as the reference and define the uplift of a memory set \(S\) as
\begin{equation}
\Delta(c,S)=V(c,S)-V(c,\emptyset).
\label{eq:memory_uplift}
\end{equation}
Accordingly, we train a shared uplift model $G_\theta(c,S)$ to predict expected set uplift, supervised by the observed utility difference from paired executions with and without \(S\),
\begin{equation}
\label{eq:observed_uplift}
y_S = U\bigl(\tau(c,S)\bigr)-U\bigl(\tau(c,\emptyset)\bigr).
\end{equation}

Because $V(c,\emptyset)$ is independent of $\theta$, replacing $V$ with $\Delta$ in Eq.~\ref{eq:retrieval_objective} does not change the optimal parameters. The learned model $G_\theta$ approximates this uplift, with $G_\theta(c,\emptyset)=0$ as the no-memory reference.
Each observation corresponds to the complete memory set used during execution. The set-valued formulation follows
consistent memory construction and evaluation rules, with the search and scoring details provided in Appendices~\ref{app:search} and~\ref{app:scorer}.
Appendix~\ref{app:observation} distinguishes uncertainty about expected uplift from variability in its observed labels.

\subsection{Set Construction and Execution Feedback}
\label{sec:retrieval}
For each task context \(c\), similarity-based recall first constructs a candidate memory pool \(\mathcal C(c)\subseteq\mathcal M_t\). Starting from $S=\emptyset$, the retriever evaluates an additional memory through its conditional gain,
\begin{equation}
g_\theta(m\mid c,S)=G_\theta(c,S\cup\{m\})-G_\theta(c,S),
\qquad g_\theta(\STOP\mid c,S)=0.
\label{eq:conditional_gain}
\end{equation}
The action set $\mathcal A(c,S)$ contains feasible memories in $\mathcal C(c)\setminus S$ and the stopping action $\STOP$, which returns the current set. Retrieval also terminates at the maximum memory count or when no feasible addition remains. 
During training, the next action $a \in  \mathcal A$ is sampled according to the conditional gains:
\begin{equation}
\pi_\theta(a\mid c,S)
=
\frac{\exp(g_\theta(a\mid c,S)/\tau_{\mathrm{tr}})}
{\sum_{a'\in\mathcal A(c,S)}
\exp(g_\theta(a'\mid c,S)/\tau_{\mathrm{tr}})},
\label{eq:retrieval_policy}
\end{equation}
where \(\tau_{\mathrm{tr}}>0\) controls the sampling concentration. 
At test time, deterministic beam search of width two returns the highest-valued set retained across depths, including $\emptyset$. It can retain intermediate sets with negative conditional gains so that later additions can reveal complementarity. 
Appendix~\ref{app:search} specifies feasibility, deduplication, and stopping.


Under this retrieval strategy, only the selected terminal set receives direct execution feedback, while alternative memory sets may remain unobserved and lack corrective supervision. The resulting dependence between retrieval preference and feedback coverage is characterized by the following proposition.

\begin{proposition}[Selective retrieval suppresses alternative observations]
\label{prop:passive_observation}
Fix a task $c$, its candidate pool, retrieval constraints, and model parameters. Let $\mathcal P(B)$ denote all feasible retrieval paths that return a memory set $B$. For each path $\mathbf a=(a_1,\ldots,a_L)$ of $L$ sampled actions, let $a_t$ denote its $t$-th action and $S_{t-1}$ the preceding memory set. Define the cumulative action deficit as
\begin{equation}
d(\mathbf a)=\sum_{t=1}^{L}\left[\max_{a\in\mathcal A(c,S_{t-1})}g_\theta(a\mid c,S_{t-1})-g_\theta(a_t\mid c,S_{t-1})\right].
\label{eq:path_deficit}
\end{equation}
The probability $p_B$ that ordinary retrieval returns $B$ satisfies
\begin{equation}
p_B\leq\sum_{\mathbf a\in\mathcal P(B)}\exp\left(-\frac{d(\mathbf a)}{\tau_{\mathrm{tr}}}\right).
\label{eq:passive_probability_bound}
\end{equation}
\end{proposition}

The bound shows how locally disfavored actions suppress direct feedback for the sets they construct, limiting correction of their estimated values. Appendices~\ref{app:selective}--\ref{app:changing_policy} extend this result to evidence arrival, identification, and changing policies, motivating targeted executions of informative alternatives.

\subsection{Decision-Aware Probing}
\label{sec:probing}
To address limited feedback for sets disfavored by retrieval (Section~\ref{sec:retrieval}), \textsc{UpliftMem} probes alternative memory sets during training.
Under a limited rollout budget, an observation is valuable when it can improve a retrieval comparison, following the value-of-information principle \citep{howard1966information,chaloner1995bayesian}. 

Because memory-set values are statistically related, one observation can reduce uncertainty across several retrieval comparisons. Following correlated-belief knowledge-gradient models \citep{frazier2009knowledge}, we use a local Gaussian surrogate conditioned on the collected uplift observations $\mathcal D$,
\begin{equation}
\widetilde{G}(c,S)
=
G_\theta(c,S)+\varphi(c,S)^\top\xi,
\qquad
\xi\sim\mathcal N(0,\Sigma).
\label{eq:belief}
\end{equation}
Here $G_\theta$ is the conditional mean uplift prediction given \(\mathcal D\), $\varphi(c,S)\in\mathbb R^d$ is a task-set feature vector of dimension $d$, $\xi\in\mathbb R^d$ is a zero-mean Gaussian latent vector, and $\Sigma\in\mathbb R^{d\times d}$ is its positive-semidefinite covariance matrix. 
The surrogate describes uncertainty about expected uplift, while a separate noise model describes variability in execution labels. Appendix~\ref{app:observation} defines the measurement model, including shared no-memory references.

For a local retrieval decision \(j\) comparing two candidate sets \(S_j^a\) and \(S_j^b\), we define their uncertain uplift difference as \(
\widetilde{g}_j
=
\widetilde{G}(c,S_j^a)
-
\widetilde{G}(c,S_j^b),
\) with conditional mean \(\mu_j=\mathbb E[\widetilde{g}_j\mid\mathcal D]\).
A prospective probe collection \(Q\) would produce additional observations \(y_Q\). The uncertainty these observations can resolve about decision \(j\) is
\begin{equation}
s_j^2(Q)
=
\operatorname{Var}(\widetilde{g}_j\mid\mathcal D)
-
\operatorname{Var}(\widetilde{g}_j\mid\mathcal D,y_Q),
\label{eq:resolved_variance}
\end{equation}
where $\mu'_j=\mathbb E[\widetilde{g}_j\mid\mathcal D,y_Q]$ is the posterior mean used to revise the comparison. 
Appendix~\ref{app:observation} derives their pre-execution distribution under Gaussian conditioning.

The following theorem gives the expected improvement in choosing between sets in closed form, using both resolved uncertainty $s_j(Q)$ and the current comparison margin $\mu_j$.

\begin{theorem}[Value of a probe for a decision]
\label{thm:evsi}
Under the joint Gaussian model, for $s_j(Q)>0$, the expected value of sample information (EVSI) for comparison $j$ is
\begin{equation}
E_j(Q)=\mathbb E_{y_Q\mid\mathcal D}\!\left[\pos{\mu'_j}\right]-\pos{\mu_j}
=s_j(Q)\psi\!\left(\frac{|\mu_j|}{s_j(Q)}\right),
\qquad \psi(t)=\phi(t)-t\Phi(-t),
\label{eq:evsi}
\end{equation}
where $\pos{x}=\max\{x,0\}$, $\phi$ and $\Phi$ are the standard normal density and cumulative distribution functions, and $E_j(Q)=0$ when $s_j(Q)=0$.
\end{theorem}

This score combines resolvable uncertainty with the current predicted utility difference and measures expected decision improvement in utility units. It can therefore compare prospective observations without simulating their outcomes. Appendix~\ref{app:information} relates binary EVSI to information gain and distinguishes local comparisons from selecting a single set among multiple alternatives.

We aggregate EVSI over the local comparison family $\mathcal G$ as $H(Q)=\sum_{j\in\mathcal G}E_j(Q)$ to select a probe batch $Q$ from the feasible probe pool. For each candidate \(q\), its marginal acquisition value is
\begin{equation}
V_t(q)
=
\sum_j E_j(Q_t\cup\{q\})
-
\sum_j E_j(Q_t),
\label{eq:batch_acquisition}
\end{equation}
\revised{where $Q_t$ contains the probes selected after $t$ steps. Since $V_t(q)$ measures expected improvement in decision utility, we compare it with the cost $\lambda_{\mathrm{probe}}$ of acquiring additional execution}
\revised{feedback, expressed in the same utility units. EVSI-guided acquisition considers candidates satisfying}
\revised{$V_t(q)>\lambda_{\mathrm{probe}}$ and stops when no such candidate remains or the probe budget is reached.}
\revised{Consequently, the budget serves as an upper bound rather than a fixed acquisition quota. Appendix~\ref{app:batch} derives the marginal value, while Appendix~\ref{app:probe_algorithm} details the acquisition procedure.}

\subsection{Learning from Execution Feedback}
\label{sec:learning}
The retrieved set \(S\), the selected probe sets \(Q\), and validated memory candidates are executed against the same no-memory reference to provide uplift observations for training \(G_\theta\).

Regression supervision fits the uplift magnitude:
\begin{equation}
\mathcal L_{\mathrm{reg}}
=
\mathbb E\!\left[
w_{\mathrm{mag}}(y_S)
\mathtt{Huber}_{\delta_H}
\left(G_\theta(c,S)-y_S\right)
\right],
\end{equation}
where \(w_{\mathrm{mag}}\) weights observations by uplift magnitude.

Since the no-memory baseline also determines whether a memory set is beneficial or harmful, we explicitly supervise the prediction sign while excluding near-zero observations:
\begin{equation}
\mathcal L_{\mathrm{sign}}
=
\mathbb E_{|y_S|>\epsilon_{\mathrm{sign}}}
\!\left[
\mathtt{BCEWithLogits}
\left(G_\theta(c,S),\mathbb I[y_S>0]\right)
\right],
\end{equation}
where \(\mathtt{BCEWithLogits}(\cdot)\) denotes binary cross-entropy applied directly to the predicted uplift as a logit.

Retrieval further requires distinguishing competing sets for the same task, so we impose pairwise supervision on their relative uplift. For pairs satisfying
\(|y_{S_a}-y_{S_b}|>\epsilon_{\mathrm{rank}}\),
\begin{equation}
\mathcal L_{\mathrm{pair}}
=
\mathbb E_{\mathrm{pair}}
\!\left[
\mathtt{softplus}
\left(
-\mathtt{sign}(y_{S_a}-y_{S_b})
\left[
G_\theta(c,S_a)-G_\theta(c,S_b)
\right]
\right)
\right].
\end{equation}
Here,  \(\mathtt{sign}(y_{S_a}-y_{S_b})\) indicates which set has the larger observed uplift.

The overall value objective combines these complementary signals
\begin{equation}
\mathcal L_{\mathrm{value}}
=
\lambda_{\mathrm{reg}}\mathcal L_{\mathrm{reg}}
+
\lambda_{\mathrm{sign}}\mathcal L_{\mathrm{sign}}
+
\lambda_{\mathrm{pair}}\mathcal L_{\mathrm{pair}}.
\label{eq:main_value_losses}
\end{equation}
Appendix~\ref{app:losses} specifies the weighting and sample-eligibility rules.

The same parameters $\theta$ also determine the retrieval policy $\pi_\theta$ through the conditional gains in Eqs.~\ref{eq:conditional_gain} and~\ref{eq:retrieval_policy}. For path $i$ on task $c_i$, let $L_i\geq1$ be its number of sampled actions, $a_{i,t}$ its $t$-th action, and $S_{i,t-1}$ the preceding set. Using paths whose terminal sets are executed, we optimize
\begin{equation}
\mathcal L_{\mathrm{policy}}
=-\mathbb E_i\!\left[
\frac{\mathtt{sg}(\omega_i)}{L_i}
\sum_{t=1}^{L_i}
\log\pi_\theta(a_{i,t}\mid c_i,S_{i,t-1})
\right].
\label{eq:main_policy_loss}
\end{equation}
The expectation averages eligible paths, $\omega_i$ is a signed and clipped weight derived from their observed uplift, and $\mathtt{sg}(\cdot)$ denotes the stop-gradient operator. 
Positive weights reinforce sampled actions, negative weights suppress them, and zero weights contribute no policy gradient. Value supervision also uses observations without a recorded retrieval path. Appendix~\ref{app:policy_training} specifies path eligibility and outcome weights.

After value-only warm-up, both objectives jointly update $\theta$ by minimizing $\mathcal L_{\mathrm{value}}+\mathcal L_{\mathrm{policy}}$, while the executor remains frozen. Memory writing extracts and validates new candidates from completed trajectories \citep{zhao2024expel,cao2026reme}, contributing a value label even when the selected candidate is not admitted. Appendix~\ref{app:writing} gives the revision and admission rules.

\vspace{-6pt}

\section{Experiments}
\label{sec:experiments}

We evaluate task performance across complete agent workflows and then hold memory stores fixed to examine retrieval decisions. Component ablations test the learning and acquisition mechanisms. We then compare probe-acquisition strategies under increasing rollout budgets to evaluate feedback efficiency, followed by retrieval-budget experiments that examine sensitivity to the amount of available context.

\subsection{Experimental Setup}
\label{sec:setup}
We evaluate ALFWorld \citep{shridhar2021alfworld} using success rate (SR) and execution steps, WebShop \citep{yao2022webshop} using reward, SR, and interaction steps, and BigCodeBench (BCB) \citep{zhuo2025bigcodebench} using overall and hard-subset SR. MemSyco-Bench \citep{xiang2026memsyco} evaluates memory-use behavior under a dialogue-grouped 70:30 train/test split. \revised{For model selection, we hold out 50 training tasks as a fixed development set, using it to select hyperparameters and checkpoints before final test evaluation.} 

The frozen executors are Qwen3-4B and Qwen3-8B. \revised{An additional end-to-end evaluation with Ministral-14B
is reported in Appendix~\ref{app:ministral14b}.} For each executor, the retriever is parameterized by $G_\theta$, which is initialized from Qwen3-Reranker-0.6B \citep{zhang2025qwen3embedding} and adapted separately. Baselines include ReAct \citep{yao2023react}, conventional retrieval methods, MemP \citep{fang2026memp}, MemRL \citep{zhang2026memrl}, Mem0 \citep{chhikara2025mem0}, MemGPT \citep{packer2023memgpt}, A-MEM \citep{xu2025amem}, LightMem \citep{fang2025lightmem}, and \revised{ReasoningBank \citep{ouyang2026reasoningbank}}. End-to-end comparisons evaluate complete workflows, whereas MemSyco-Bench compares retrieval within identical memory stores. Appendix~\ref{app:protocol} specifies these controls and the implementation-evidence view used for BCB.

We recall 50 candidate memories and cap retrieval at $K=5$ for ALFWorld and WebShop and $K=3$ for BCB. MemSyco-Bench uses $K=5$, with additional tests at $K\in\{3,10\}$. Appendix~\ref{app:config} reports the model configuration. 
We repeat training across independent random seeds and report the average performance across runs. Detailed statistical protocols are provided in Appendix~\ref{app:metrics}.

\subsection{RQ1: End-to-End Agent Performance}
\label{sec:rq1}

Table~\ref{tab:main_results} reports end-to-end performance across three benchmarks. UpliftMem attains the highest success rates on both ALFWorld splits and WebShop with both executors. With Qwen3-4B, it achieves 76.9\% and 70.9\% success on ALFWorld seen and unseen tasks, respectively, 45.60\% on WebShop, and 39.77\% on BCB. With Qwen3-8B, it achieves 81.43\% and 79.10\% on ALFWorld, 44.00\% on WebShop, and 42.69\% on BCB, matching the highest overall BCB result. It also obtains the highest WebShop reward with both executors and the fewest steps on ALFWorld seen tasks. 
\revised{We further evaluate UpliftMem with a Ministral-14B executor to assess its applicability beyond the Qwen  (Appendix~\ref{app:ministral14b}).}
We next hold memory stores fixed to examine the contribution of retrieval decisions.

\begin{table*}[t]
\centering
\footnotesize
\caption{End-to-end agent performance with Qwen3-4B and Qwen3-8B. Success rates are percentages; values are means across independent training runs. Standard deviations appear in Appendix~\ref{app:agent_variability}.}
\label{tab:main_results}

\setlength{\tabcolsep}{2.6pt}
\renewcommand{\arraystretch}{1.10}

\begin{adjustbox}{max width=\textwidth}
\begin{tabular}{
@{}l
cccccccc
@{\hspace{5pt}}
cccccc
@{\hspace{5pt}}
cccc@{}
}
\toprule

& \multicolumn{8}{c}{\textbf{ALFWorld}}
& \multicolumn{6}{c}{\textbf{WebShop}}
& \multicolumn{4}{c}{\textbf{BCB}} \\

\cmidrule(lr){2-9}
\cmidrule(lr){10-15}
\cmidrule(lr){16-19}

& \multicolumn{2}{c}{\makecell{Seen SR$\uparrow$}}
& \multicolumn{2}{c}{\makecell{Seen Steps$\downarrow$}}
& \multicolumn{2}{c}{\makecell{Unseen SR$\uparrow$}}
& \multicolumn{2}{c}{\makecell{Unseen Steps$\downarrow$}}

& \multicolumn{2}{c}{\makecell{Reward$\uparrow$}}
& \multicolumn{2}{c}{\makecell{SR$\uparrow$}}
& \multicolumn{2}{c}{\makecell{Steps$\downarrow$}}

& \multicolumn{2}{c}{\makecell{SR$\uparrow$}}
& \multicolumn{2}{c}{\makecell{Hard SR$\uparrow$}} \\

\cmidrule(lr){2-3}
\cmidrule(lr){4-5}
\cmidrule(lr){6-7}
\cmidrule(lr){8-9}
\cmidrule(lr){10-11}
\cmidrule(lr){12-13}
\cmidrule(lr){14-15}
\cmidrule(lr){16-17}
\cmidrule(lr){18-19}

\rowcolor{headergray}
\textbf{Method}
& \textbf{4B} & \textbf{8B}
& \textbf{4B} & \textbf{8B}
& \textbf{4B} & \textbf{8B}
& \textbf{4B} & \textbf{8B}
& \textbf{4B} & \textbf{8B}
& \textbf{4B} & \textbf{8B}
& \textbf{4B} & \textbf{8B}
& \textbf{4B} & \textbf{8B}
& \textbf{4B} & \textbf{8B} \\

\midrule

ReAct
& 43.13 & 57.14
& 22.35 & 20.36
& 50.00 & 76.87
& 20.98 & 16.59
& 63.20 & \second{68.39}
& 38.00 & 41.60
& 7.09 & 5.80
& 31.87 & 40.35
& 15.91 & \second{20.45} \\

Random
& 47.14 & 65.00
& 18.54 & 19.46
& 49.25 & 71.64
& 21.73 & 17.80
& 62.22 & 52.87
& 39.20 & 29.80
& 7.26 & 6.13
& 30.70 & \best{42.69}
& 13.64 & 13.64 \\

BM25-RAG
& 59.28 & 75.71
& 18.67 & 15.31
& 52.23 & 70.15
& 19.75 & 17.13
& 57.88 & 64.87
& 37.60 & \second{41.80}
& 7.24 & 6.55
& 33.33 & 40.35
& 11.36 & 15.91 \\

Dense-RAG
& 50.00 & 67.14
& 20.82 & 17.02
& 52.99 & 68.66
& 21.05 & 17.43
& 56.97 & 64.17
& 34.80 & 38.60
& 7.40 & 6.61
& 34.50 & 39.47
& 15.91 & 11.36 \\

MemP
& 50.00 & 54.29
& 19.81 & 18.64
& 58.21 & 49.25
& 19.15 & 20.09
& 52.42 & 58.87
& 35.40 & 37.00
& 6.71 & 7.74
& 34.21 & 39.47
& 18.18 & 11.36 \\

MemRL
& 54.29 & 72.14
& 19.69 & 16.35
& 58.96 & \second{78.35}
& 19.36 & 15.44
& 53.17 & 60.78
& 35.00 & 36.40
& 6.86 & 7.38
& 32.46 & 39.77
& 18.18 & 13.64 \\

Mem0
& \second{75.71} & 70.00
& \second{15.77} & 16.57
& 67.16 & 77.61
& 16.82 & \best{15.04}
& \second{70.26} & 58.36
& \second{43.60} & 33.80
& 5.59 & \best{5.03}
& 35.09 & 40.64
& \second{20.45} & 11.36 \\

MemGPT
& 57.14 & 52.86
& 19.37 & 21.30
& 65.67 & 64.18
& 19.03 & 20.27
& 67.26 & 65.65
& 40.20 & 40.40
& 6.33 & \second{5.49}
& \second{38.01} & 41.52
& 18.18 & 11.36 \\

A-MEM
& 57.86 & 58.57
& 20.29 & 20.99
& 58.96 & 61.19
& 21.15 & 22.33
& 65.14 & 55.48
& 39.60 & 33.20
& \best{5.54} & 5.92
& 33.04 & 38.30
& 18.18 & 18.18 \\

LightMem
& 64.29 & 66.43
& 19.81 & 18.99
& \second{67.91} & 77.61
& \second{16.74} & 16.74
& 66.28 & 61.41
& 40.40 & 37.80
& 5.78 & 6.86
& 35.96 & \second{41.81}
& \best{22.73} & 18.18 \\

\revised{ReasoningBank}
& \revised{65.00} & \revised{\second{77.14}}
& \revised{17.92} & \revised{\second{15.26}}
& \revised{66.42} & \revised{74.62}
& \revised{18.40} & \revised{\second{15.26}}
& \revised{53.70} & \revised{48.48}
& \revised{32.40} & \revised{31.00}
& \revised{5.74} & \revised{5.58}
& \revised{30.41} & \revised{37.72}
& \revised{13.64} & \revised{11.36} \\

\midrule

\rowcolor{oursblue}
\textbf{UpliftMem}
& \best{76.86} & \best{81.43}
& \best{14.01} & \best{13.61}
& \best{70.91} & \best{79.10}
& \best{15.81} & 16.53
& \best{72.51} & \best{70.20}
& \best{45.60} & \best{44.00}
& \second{5.57} & 5.56
& \best{39.77} & \best{42.69}
& \second{20.45} & \best{22.73} \\

\bottomrule
\end{tabular}
\end{adjustbox}
\end{table*}

\subsection{RQ2: Controlled Retrieval Evaluation}
\label{sec:rq2}

MemSyco-Bench evaluates when memory should guide a response through Objective Fact, Scope Control, and Evidence Conflict, and how it should be used through Personalized Use and Valid Selection. Each scenario reports response accuracy alongside a behavioral measure, namely sycophancy (Syco) for the first three, correct-memory use (Correct) for Personalized Use, and outdated-memory use (Outdated) for Valid Selection.

To isolate retrieval effects, we replace each memory system's native retriever with UpliftMem while holding its post-extraction memory store and executor fixed. Figure~\ref{fig:upliftmem_heatmap} shows higher accuracy and lower sycophancy across Objective Fact, Scope Control, and Evidence Conflict for all five memory stores and both executors. Scope Control accuracy increases by 23.7--62.7 percentage points with Qwen3-4B, while Evidence Conflict accuracy increases by 29.1--54.6 points with Qwen3-8B. In Valid Selection, accuracy improves in all ten store--executor comparisons, with outdated-memory use decreasing in nine. These results show that retrieval decisions affect both when agents rely on memory and which memories they use, even with identical available memory content. Table~\ref{tab:app-main-k5} reports the corresponding absolute scores and run-to-run variability for each memory store.
\begin{figure}[t]
    \centering
    \includegraphics[
        width=0.85\linewidth
    ]{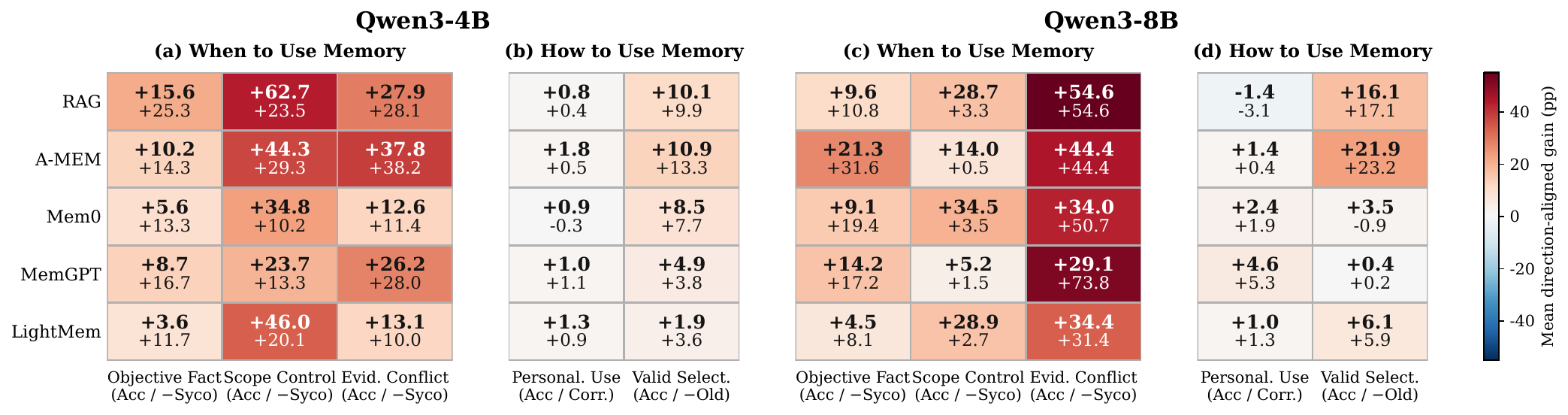}
\caption{Fixed-store retrieval gains over native retrieval on MemSyco-Bench at $K=5$ (pp). Accuracy is shown above the direction-aligned behavioral improvement.}
    \label{fig:upliftmem_heatmap}
\end{figure}

\subsection{RQ3: Ablation Study}
\label{sec:rq3}



\revised{Seven ablations isolate set-level valuation, iterative memory writing, scorer adaptation, feedback acquisition, and the learning objectives. \emph{Independent Scoring} values memories separately, \emph{w/o Iterative Writing} disables online memory updates, \emph{Base reranker} uses the pretrained scorer, and \emph{Warm-up only} omits subsequent scorer updates and probing. \emph{Random probing} replaces EVSI under a matched probe budget, while the remaining variants remove the policy or value objective.}

\revised{Across all three benchmarks and both executors, UpliftMem achieves the highest primary success rates (Table~\ref{tab:ablation}), outperforming Independent Scoring on ALFWorld seen tasks by 13.29 and 12.86 percentage points with Qwen3-4B and Qwen3-8B, respectively. Iterative memory writing contributes 7.58 and 4.29 points relative to its ablation, while warm-up-only and random-probing comparisons further support continued adaptation and EVSI-guided acquisition. Appendix~\ref{app:memory_interaction} examines non-additive memory interactions, and Appendix~\ref{app:training_stability} evaluates performance as the memory pool expands.}



\begin{table*}[t]
\centering
\footnotesize
\caption{Component ablations on ALFWorld, WebShop, and BCB with both executors.}
\label{tab:ablation}
\setlength{\tabcolsep}{2.6pt}
\renewcommand{\arraystretch}{1.10}
\begin{adjustbox}{max width=\textwidth}
\begin{tabular}{@{}lcccccccc@{\hspace{5pt}}cccccc@{\hspace{5pt}}cccc@{}}
\toprule
& \multicolumn{8}{c}{\textbf{ALFWorld}}
& \multicolumn{6}{c}{\textbf{WebShop}}
& \multicolumn{4}{c}{\textbf{BCB}} \\
\cmidrule(lr){2-9}\cmidrule(lr){10-15}\cmidrule(lr){16-19}
& \multicolumn{2}{c}{Seen SR$\uparrow$}
& \multicolumn{2}{c}{Seen Steps$\downarrow$}
& \multicolumn{2}{c}{Unseen SR$\uparrow$}
& \multicolumn{2}{c}{Unseen Steps$\downarrow$}
& \multicolumn{2}{c}{Reward$\uparrow$}
& \multicolumn{2}{c}{SR$\uparrow$}
& \multicolumn{2}{c}{Steps$\downarrow$}
& \multicolumn{2}{c}{SR$\uparrow$}
& \multicolumn{2}{c}{Hard SR$\uparrow$} \\
\cmidrule(lr){2-3}\cmidrule(lr){4-5}\cmidrule(lr){6-7}\cmidrule(lr){8-9}
\cmidrule(lr){10-11}\cmidrule(lr){12-13}\cmidrule(lr){14-15}
\cmidrule(lr){16-17}\cmidrule(lr){18-19}
\rowcolor{headergray}
\textbf{Variant}
& \textbf{4B} & \textbf{8B} & \textbf{4B} & \textbf{8B}
& \textbf{4B} & \textbf{8B} & \textbf{4B} & \textbf{8B}
& \textbf{4B} & \textbf{8B} & \textbf{4B} & \textbf{8B}
& \textbf{4B} & \textbf{8B} & \textbf{4B} & \textbf{8B}
& \textbf{4B} & \textbf{8B} \\
\midrule
Base reranker
& 68.57 & 65.00 & 15.52 & 17.31 & 63.43 & 64.93 & 17.12 & 17.33 & 67.12 & 66.81 & 41.80 & 42.00 & 6.32 & 6.44 & 33.63 & 38.89 & 18.18 & 18.18 \\
\revised{Independent Scoring}
& \revised{63.57} & \revised{68.57} & \revised{16.42} & \revised{16.09} & \revised{\second{68.66}} & \revised{70.15} & \revised{\second{16.45}} & \revised{\second{16.45}} & \revised{65.92} & \revised{69.38} & \revised{40.80} & \revised{43.40} & \revised{6.53} & \revised{5.63} & \revised{35.38} & \revised{37.43} & \revised{18.18} & \revised{18.18} \\
\revised{w/o Iterative Writing}
& \revised{69.28} & \revised{\second{77.14}} & \revised{17.02} & \revised{\second{14.06}} & \revised{64.18} & \revised{73.88} & \revised{18.20} & \revised{\best{14.46}} & \revised{\second{68.94}} & \revised{69.82} & \revised{42.00} & \revised{\second{43.60}} & \revised{6.04} & \revised{5.63} & \revised{38.01} & \revised{\second{41.23}} & \revised{\second{20.45}} & \revised{18.18} \\
Warm-up only
& 64.76 & 66.19 & 16.69 & 16.86 & 61.19 & 67.16 & 17.36 & 16.65 & 68.82 & 66.66 & \second{43.60} & 43.20 & \second{5.71} & 6.46 & 34.80 & \second{41.23} & 18.18 & 18.18 \\
Random probing
& 66.07 & 72.85 & 16.18 & 15.32 & 63.43 & \second{76.12} & 17.28 & 16.85 & 68.03 & \second{70.16} & 41.60 & \second{43.60} & 6.08 & \second{5.59} & 37.72 & 40.94 & \second{20.45} & \second{20.45} \\
w/o policy loss
& \second{71.42} & 70.71 & 16.98 & 14.85 & 67.16 & 70.90 & 16.85 & 17.43 & 67.88 & 69.54 & 41.80 & 43.40 & 5.99 & 6.78 & \second{38.89} & \second{41.23} & \best{25.00} & 18.18 \\
w/o value loss
& 70.00 & 75.00 & \second{15.21} & 14.34 & 68.65 & 73.13 & 16.74 & 17.15 & 68.61 & 68.25 & 43.40 & 43.20 & 5.72 & 6.01 & 37.72 & \second{41.23} & \second{20.45} & 18.18 \\
\midrule
\rowcolor{oursblue}
\textbf{UpliftMem (Full)}
& \best{76.86} & \best{81.43} & \best{14.01} & \best{13.61} & \best{70.91} & \best{79.10} & \best{15.81} & 16.53 & \best{72.51} & \best{70.20} & \best{45.60} & \best{44.00} & \best{5.57} & \best{5.56} & \best{39.77} & \best{42.69} & \second{20.45} & \best{22.73} \\
\bottomrule
\end{tabular}
\end{adjustbox}
\end{table*}

\revised{Fixed-store MemSyco-Bench ablations (Table~\ref{tab:ablation_memsyco}) further isolate retrieval decisions. Set-level scoring improves Decision accuracy over independent scoring by 9.0/8.8 percentage points and reduces sycophancy by 10.8/12.8 points with Qwen3-4B/8B. EVSI-guided probing improves Decision accuracy over random probing by 10.6/13.9 points and reduces sycophancy by 9.2/14.8 points. Per-store loss ablations (Appendix Tables~\ref{tab:app-abl-k5-4b} and~\ref{tab:app-abl-k5-8b}) examine the contributions of value supervision and outcome-weighted policy learning.}

\begin{table*}[htbp]
\centering
\footnotesize
\caption{MemSyco-Bench ablations at $K=5$, averaged over five memory stores. Each entry reports Qwen3-4B/Qwen3-8B.}
\label{tab:ablation_memsyco}
\setlength{\tabcolsep}{2.5pt}
\renewcommand{\arraystretch}{1.12}
\begin{tabular*}{\textwidth}{@{\extracolsep{\fill}}lcccccccc@{}}
\toprule
\rowcolor{headergray}
\textbf{Metric}
& \textbf{Native}
& \textbf{Base}
& \revised{\textbf{Independent}}
& \textbf{Warm-up}
& \textbf{Random}
& \textbf{No policy}
& \textbf{No value}
& \textbf{UpliftMem} \\
\midrule

Decision$\uparrow$
& 41.7/44.4
& 43.1/44.5
& \revised{57.5/60.0}
& 49.6/48.8
& 55.9/54.9
& 56.7/56.1
& 55.8/53.1
& \cellcolor{oursblue}\textbf{66.5/68.8}
\\

Syco.$\downarrow$
& 49.0/40.0
& 47.2/36.7
& \revised{40.3/29.3}
& 44.2/34.7
& 38.7/31.3
& 39.2/31.9
& 41.6/35.3
& \cellcolor{oursblue}\textbf{29.5/16.5}
\\

Select$\uparrow$
& 54.3/61.1
& 55.1/62.2
& \revised{57.3/61.8}
& 53.7/57.9
& 58.3/60.3
& 56.9/60.2
& 58.0/62.8
& \cellcolor{oursblue}\textbf{61.6/70.7}
\\

Old$\downarrow$
& 47.0/40.8
& 46.5/39.6
& \revised{42.6/36.6}
& 47.6/42.4
& 42.5/41.3
& 43.8/40.3
& 43.6/38.9
& \cellcolor{oursblue}\textbf{39.3/31.7}
\\

\bottomrule
\end{tabular*}
\vspace{-4pt}
\end{table*}



\subsection{RQ4: Probe Strategy Efficiency}
\label{sec:rq4}

Starting from the same warm-up checkpoint and a fixed memory store, we compare EVSI with random probing, Monte Carlo predictive standard deviation (MC-STD), and expected information gain (EIG) on WebShop. Each strategy collects additional probes and updates the retriever through \revised{low-rank adaptation (LoRA)} at budgets of $\{20,50,80,100,120,150,200\}$, followed by evaluation on the same 200 validation tasks. We measure cumulative success rate as the fraction of tasks solved by at least one checkpoint up to the current budget, alongside the mean reward of the current checkpoint.
\begin{wrapfigure}{r}{0.6\textwidth}
    \centering
    \includegraphics[width=\linewidth]
        {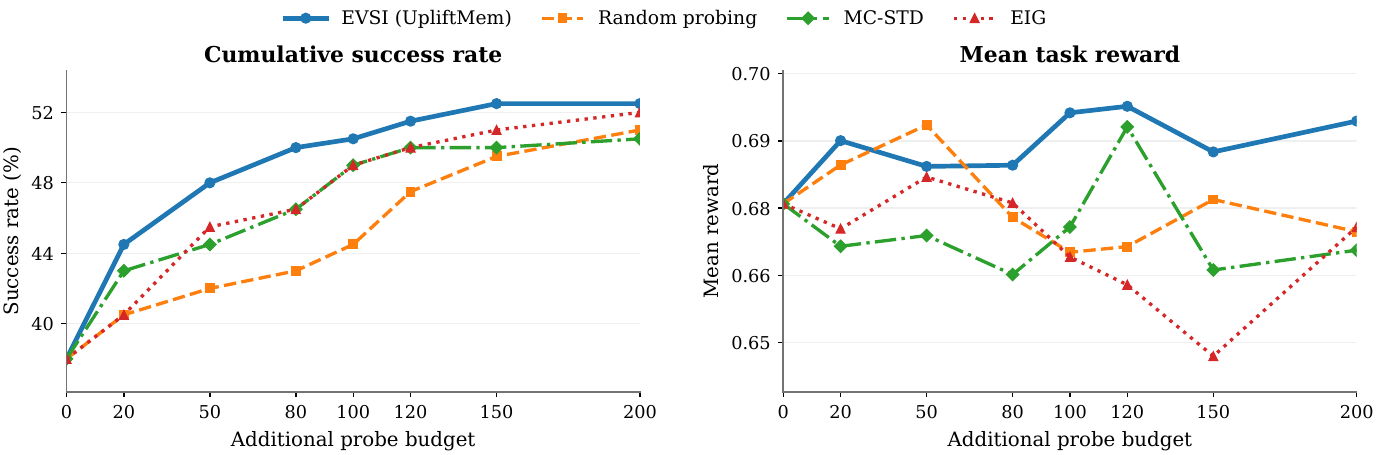}
    \captionsetup{
        font=footnotesize,
        labelfont=normalfont,
        labelsep=period,
        skip=3pt
    }
    \caption{Cumulative success rate and mean reward across probe budgets on WebShop.}
    \label{fig:probe_efficiency}
\end{wrapfigure}
Figure~\ref{fig:probe_efficiency} shows that EVSI achieves the highest cumulative success rate at every nonzero budget, reaching 48.0\% with 50 probes compared with 42.0\% for random probing. It also leads at 200 probes and attains a mean reward of 0.694, versus 0.670 for random probing. Thus, decision-aware acquisition converts a fixed feedback budget into broader task coverage and stronger execution outcomes.
Additional probe-budget results for ALFWorld and BigCodeBench are presented in Appendix~\ref{app:probe_efficiency}. \revised{Appendix~\ref{app:cost_sensitivity} further examines how acquisition cost regulates the number of probes selected, while Appendix~\ref{app:gaussian_sensitivity} evaluates the sensitivity of EVSI-based probe valuation and acquisition to Gaussian belief construction.}

\par

\subsection{RQ5: Sensitivity to Retrieval Budget}
\label{sec:rq5}
The preceding experiment varies the training-time probe budget. We now turn to the inference-time retrieval budget $K$, which controls how many memories may enter the task context. We vary $K\in\{3,5,10\}$ on MemSyco-Bench to examine how retrieval performance changes with available context. Figure~\ref{fig:budget} reports the absolute metrics, and Table~\ref{tab:app-k-sens} gives scenario-level gains over native retrieval.

UpliftMem improves the cross-model direction-aligned score by 14.2, 15.8, and 15.1 percentage points at $K=3,5,10$, respectively. The improvement is consistent across retrieval capacities, with the largest scenario-level gain reaching 39.0 points on Evidence Conflict at $K=10$ (Table~\ref{tab:app-k-sens}).

\begin{figure*}[htbp]
\centering
\includegraphics[width=0.75\textwidth]{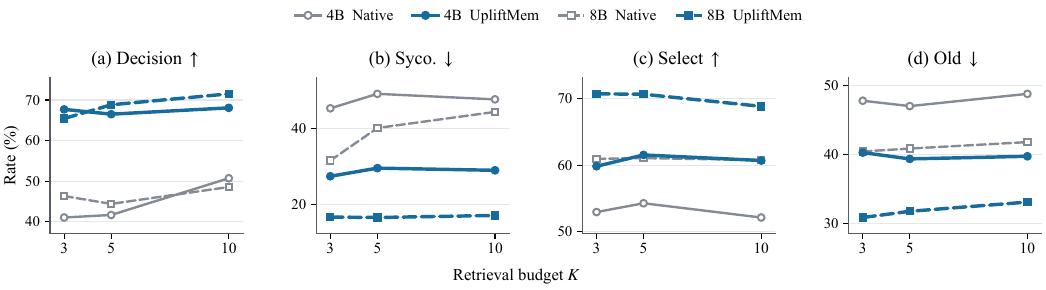}
\caption{MemSyco-Bench performance across retrieval budgets $K$, averaged over five memory stores.}
\label{fig:budget}
\end{figure*}

\subsection{RQ6: Zero-Shot Generalization}
\label{sec:rq6}

\revised{We evaluate zero-shot generalization across executors and task types. For executor transfer, out-of-distribution (OOD) performance uses a different executor from training, while in-domain (ID) performance provides the reference. Task-type transfer evaluates a scorer on a task type excluded from its training feedback.}

\begin{table*}[!htbp]
\centering\footnotesize

\caption{\revised{Cross-executor generalization on MemSyco-Bench ($\Delta=\mathrm{OOD}-\mathrm{ID}$, pp).}}
\label{tab:cross_executor}
\setlength{\tabcolsep}{2.5pt}
\renewcommand{\arraystretch}{1.12}
\begin{adjustbox}{max width=\textwidth}
\begin{tabular}{@{}c*{6}{c}@{\hspace{7pt}}*{6}{c}@{}}
\toprule
& \multicolumn{6}{c}{\textbf{Target Qwen3-4B (8B$\to$4B)}} & \multicolumn{6}{c}{\textbf{Target Qwen3-8B (4B$\to$8B)}} \\
\cmidrule(lr){2-7}\cmidrule(lr){8-13}
\rowcolor{headergray}
\textbf{$K$} & Scope & Conf. & Obj. & Pers. & Valid & \textbf{Avg. $|\Delta|$} & Scope & Conf. & Obj. & Pers. & Valid & \textbf{Avg. $|\Delta|$} \\
\midrule
3 & -23.6\sd{0.8} & +2.4\sd{1.9} & +0.3\sd{0.7} & -9.2\sd{1.6} & +3.2\sd{1.6} & 7.8\sd{0.7} & +21.3\sd{1.4} & -11.9\sd{1.3} & -0.4\sd{0.6} & +1.2\sd{0.8} & -1.1\sd{1.5} & 7.3\sd{0.7} \\
\rowcolor{oursblue}
5 & -14.4\sd{0.9} & +0.6\sd{0.6} & +2.7\sd{1.2} & -0.2\sd{0.7} & +2.2\sd{1.5} & 4.1\sd{0.7} & +12.7\sd{0.8} & -4.1\sd{0.7} & -0.7\sd{0.7} & -2.0\sd{2.1} & +0.6\sd{0.3} & 4.1\sd{0.4} \\
10 & -3.5\sd{1.7} & -0.8\sd{1.3} & -1.9\sd{1.4} & +2.5\sd{0.8} & +0.3\sd{0.3} & 1.9\sd{0.4} & +3.3\sd{0.3} & -5.1\sd{0.8} & +1.4\sd{1.3} & +1.0\sd{1.4} & +2.3\sd{3.8} & 2.8\sd{1.1} \\
\bottomrule
\end{tabular}
\end{adjustbox}
\end{table*}

\revised{Table~\ref{tab:cross_executor} compares zero-shot transfer between Qwen3-4B and Qwen3-8B with in-domain (ID) evaluation on MemSyco-Bench. Across five memory stores and three independent training runs, the mean absolute OOD--ID pass-rate difference at $K=5$ is 4.1 percentage points in both directions, reaching 1.9 and 2.8 points at $K=10$. These results establish bidirectional reuse of learned memory-set valuation.}

\revised{The same transfer is evaluated on ALFWorld, WebShop, and BigCodeBench (Appendix~\ref{app:cross_executor_agent}, Table~\ref{tab:app_cross_executor_agents}). On ALFWorld, transferring the Qwen3-8B scorer to Qwen3-4B improves seen and unseen success by 4.29 and 4.48 points, while full-package transfer yields gains of 5.00 and 7.46 points. On BigCodeBench, the 4B-to-8B scorer transfer improves hard-subset success by 2.27 points.}

\begin{wraptable}{r}{0.56\textwidth}
    \centering
    \captionsetup{font=footnotesize,labelfont=normalfont,labelsep=period,skip=3pt}
    \caption{\revised{Leave-one-task-type-out (LOTO) gaps (pp).}}
    \label{tab:loto_4b}
    \setlength{\tabcolsep}{2.5pt}
    \renewcommand{\arraystretch}{1.12}
    {%
    \begin{adjustbox}{max width=\linewidth}
    \begin{tabular}{@{}crrrrrr@{}}
    \toprule
    \rowcolor{headergray}
    $K$ & Scope & Conf. & Obj. & Pers. & Valid & Avg. $|\Delta|$ \\
    \midrule
    3 & -34.3 & -16.6 & -4.7 & -5.0 & -1.8 & 13.3 \\
    \rowcolor{oursblue}
    5 & -20.3 & -5.5 & -0.4 & -2.0 & -1.1 & 7.7 \\
    10 & -7.8 & -3.9 & +2.5 & -0.5 & -0.5 & \textbf{4.2} \\
    \bottomrule
    \end{tabular}
    \end{adjustbox}%
    }
\end{wraptable}
\noindent\revised{To evaluate unseen task types, leave-one-task-type-out (LOTO) training learns from four MemSyco types and evaluates on the excluded fifth, with an all-task scorer as reference (Appendix~\ref{app:loto_protocol}). Across five memory stores and three independent training runs, the Qwen3-4B differences at $K=5$ are $-0.4$, $-2.0$, and $-1.1$ points for objective fact judgment, personalized memory use, and valid memory selection, respectively (Table~\ref{tab:loto_4b}). Increasing the inference budget from $K=3$ to $K=10$ reduces the mean absolute gap from 13.3 to 4.2 points, with all models trained at $K=5$. These results establish the reuse of set-level valuation on task types absent from training feedback.}



\section{Conclusion}
\label{sec:conclusion}

\revised{Agent memory retrieval determines both which experiences guide execution and which outcomes become available to train the retriever. UpliftMem addresses these coupled decisions by learning the execution uplift of complete memory sets relative to the same executor without memory and allocating additional rollouts according to their expected value for improving retrieval decisions. Set-level valuation captures memory interactions through conditional gains, while EVSI-guided acquisition prioritizes additional feedback according to its expected impact on retrieval choices. Across ALFWorld, WebShop, and BigCodeBench, UpliftMem achieves the best or joint-best success rates among the evaluated methods. On fixed-store MemSyco-Bench evaluations, decision accuracy increases from 41.7\% to 66.5\% with Qwen3-4B and from 44.4\% to 68.8\% with Qwen3-8B. Set-level scoring improves decision accuracy over independent scoring by 9.0 and 8.8 percentage points, while EVSI-guided probing improves over random probing by 10.6 and 13.9 points at matched budgets, respectively. Zero-shot evaluations further demonstrate the transferability of learned memory-set valuation across executors and previously unseen task types. These findings establish agent memory retrieval as a coupled learning problem in which selecting useful experience and acquiring informative execution feedback jointly shape subsequent decisions.}

\subsection*{AI Use Statement}
Generative AI was used to assist with manuscript organization, language editing and literature searches during revision.

\bibliography{iclr2027_conference}
\bibliographystyle{iclr2027_conference}
\clearpage
\appendix
\raggedbottom

\section{Theoretical Analysis}
\label{app:theory}

We extend the coverage analysis in Section~\ref{sec:retrieval} and derive the acquisition scores in Section~\ref{sec:probing}. The resulting Gaussian expressions are used by the batch procedure in Appendix~\ref{app:probe_algorithm}.

\subsection{Selective Retrieval and Observation Coverage}
\label{app:selective}

The analysis fixes a task $c$, the recalled candidate pool, the feasibility constraints, and the training policy in Eq.~\ref{eq:retrieval_policy}. A memory set is identified by its members, independently of the order in which they were added. Let $\mathcal P(B)$ contain every feasible sampled construction path that terminates at $B$. A sampled $\STOP$ belongs to the path, whereas reaching a retrieval limit ends it without another sampled action. Distinct complete paths are mutually exclusive, even when they lead to the same set.

For a path $\mathbf a=(a_1,\ldots,a_L)$, write $S_{t-1}$ for the set preceding action $a_t$. Its nonnegative step deficit $\delta_t$ is the highest feasible conditional gain at $S_{t-1}$ minus the gain of $a_t$. The length $L=L(\mathbf a)$ counts sampled actions, including a sampled $\STOP$, and the cumulative deficit is $d(\mathbf a)=\sum_{t=1}^L\delta_t$ as in Eq.~\ref{eq:path_deficit}.

\begin{proof}[Proof of Proposition~\ref{prop:passive_observation}]
At any sampled decision, the denominator of Eq.~\ref{eq:retrieval_policy} contains a maximizing action, giving
\begin{equation}
\pi_\theta(a_t\mid c,S_{t-1})\leq\exp\!\left(-\frac{\delta_t}{\tau_{\mathrm{tr}}}\right).
\end{equation}
Writing $p(\mathbf a\mid c)$ for a complete path's probability, multiplication of its sampled conditional action probabilities gives
\begin{equation}
 p(\mathbf a\mid c)=\prod_{t=1}^L\pi_\theta(a_t\mid c,S_{t-1})\leq \exp\!\left(-\frac{d(\mathbf a)}{\tau_{\mathrm{tr}}}\right).
\label{eq:path_probability_bound}
\end{equation}
Summing the probabilities of all disjoint complete paths leading to $B$ yields
\begin{equation}
 p_B=\sum_{\mathbf a\in\mathcal P(B)}p(\mathbf a\mid c)\leq\sum_{\mathbf a\in\mathcal P(B)}\exp\!\left(-\frac{d(\mathbf a)}{\tau_{\mathrm{tr}}}\right),
\end{equation}
which is Eq.~\ref{eq:passive_probability_bound}. If $B$ is unreachable under the candidate pool and feasibility constraints, $\mathcal P(B)=\emptyset$ and $p_B=0$.
\end{proof}

The exact expression retains competition among actions. At a sampled state along path $\mathbf a$, define
\begin{equation}
 Z_t(\mathbf a)=\sum_{a\in\mathcal A(c,S_{t-1})}\exp\!\left(\frac{g_\theta(a\mid c,S_{t-1})-\max_{a'\in\mathcal A(c,S_{t-1})}g_\theta(a'\mid c,S_{t-1})}{\tau_{\mathrm{tr}}}\right).
\label{eq:path_normalizer}
\end{equation}
At each sampled decision, $1\leq Z_t\leq|\mathcal A(c,S_{t-1})|$, and
\begin{equation}
 p_B=\sum_{\mathbf a\in\mathcal P(B)}\exp\!\left(-\frac{d(\mathbf a)}{\tau_{\mathrm{tr}}}\right)\prod_{t=1}^{L(\mathbf a)}Z_t(\mathbf a)^{-1}.
\label{eq:exact_set_probability}
\end{equation}
The probability of reaching a set reflects action deficits and competition along every construction order. The sum over all feasible paths therefore determines direct coverage, rather than the probability of any one ordering.

\begin{corollary}[Arrival of direct set evidence]
\label{cor:direct_evidence}
Let $\bar p_B=\min\{1,\sum_{\mathbf a\in\mathcal P(B)}\exp[-d(\mathbf a)/\tau_{\mathrm{tr}}]\}$. For $n$ independent runs under the fixed policy, let $N_B(n)$ count terminal executions of $B$ and $T_B$ denote the run index of its first execution. Then
\begin{equation}
\mathbb E[N_B(n)]=np_B\leq n\bar p_B,\qquad \Pr(N_B(n)=0)=(1-p_B)^n\geq\max\{0,1-n\bar p_B\}.
\label{eq:direct_evidence_count}
\end{equation}
If $p_B>0$, then $\mathbb E[T_B]=p_B^{-1}\geq\bar p_B^{-1}$. If $p_B=0$, $T_B$ is infinite almost surely.
\end{corollary}
\begin{proof}
Independent runs give $N_B(n)\sim\operatorname{Binomial}(n,p_B)$ and a geometric first-arrival time when $p_B>0$. The zero-count bound follows from $(1-p_B)^n\geq1-np_B$ and $p_B\leq\bar p_B$. The zero-coverage case follows directly.
\end{proof}

Complementarity illustrates the effect of construction-path probabilities on coverage. Suppose two candidate memories satisfy $G_\theta(c,\{m_1\})=G_\theta(c,\{m_2\})=-a$ with $a>0$ and $G_\theta(c,\{m_1,m_2\})=b>0$, while $\STOP$ has gain zero. Each of the two addition orders begins with a deficit of at least $a$, even though the completed set has positive predicted value. Equation~\ref{eq:passive_probability_bound} gives $p_{\{m_1,m_2\}}\leq\min\{1,2e^{-a/\tau_{\mathrm{tr}}}\}$. Both construction orders are included in this bound, connecting set-level complementarity to the feedback coverage discussed in Section~\ref{sec:retrieval}.

\subsection{Identification Under Incomplete Coverage}
\label{app:identification}

The coverage result becomes an identification constraint when observations of other sets provide no structural information about the unexecuted alternative. The next proposition isolates this setting to distinguish direct evidence from generalization through a shared predictor.

\begin{proposition}[Unobserved alternatives remain observationally indistinguishable]
\label{prop:coverage_identification}
Fix a task $c$ and a nonempty feasible set $B$. Consider a class of outcome laws in which the utility distribution of $B$ can vary independently of the distributions for the other sets. Under a fixed retrieval policy that executes $B$ with probability $p_B$, there exist two environments with opposite signs of $\Delta(c,B)$ such that any sign estimator $\widehat z_n$ based on $n$ independent ordinary retrieval executions obeys the following bound. Index the environments by $v\in\{-,+\}$, with $\Pr_v$ their probability laws and $\Delta_v$ their corresponding uplifts.
\begin{equation}
\max_{v\in\{-,+\}}\Pr_v\!\left(\widehat z_n\ne\operatorname{sign}(\Delta_v(c,B))\right)\geq\frac12(1-p_B)^n.
\label{eq:coverage_lower_bound}
\end{equation}
The bound also holds when no-memory outcomes accompany each execution. If $p_B=0$, the ordinary-data laws are identical for all $n$ and the sign is unidentifiable over this outcome class.
\end{proposition}
\begin{proof}
Construct two environments with deterministic utility $1/2$ for the empty set and every set other than $B$. Give $B$ utility $3/4$ in environment $+$ and $1/4$ in environment $-$. The two environments have respective uplifts $1/4$ and $-1/4$ at $B$, while all other outcome laws agree. Let $E_n$ be the event that ordinary retrieval never executes $B$. Its probability is $(1-p_B)^n$ under either environment. Conditional on $E_n$, all observed feedback, including any no-memory references, has the same distribution in both environments. Because the true signs differ, the two conditional estimation-error probabilities sum to at least one. Their average unconditional error is therefore at least $\frac12\Pr(E_n)$, and their maximum is no smaller. When $p_B=0$, $E_n$ occurs surely.
\end{proof}

Combining this result with Proposition~\ref{prop:passive_observation} gives a worst-case sign-error bound of at least $\frac12(1-\bar p_B)^n$ over the stated outcome class. Shared representations introduce additional structure across sets. The surrogate in Appendix~\ref{app:observation} uses such structure to guide acquisition, while new executions supply direct labels.

\subsection{Changing Policies and Additional Probe Coverage}
\label{app:changing_policy}

Online scorer updates change retrieval probabilities. A history-dependent version of the coverage argument connects the fixed-policy results to this training setting.

Let $\mathcal H_{t-1}$ be the history before round $t$ and $p_t(B\mid\mathcal H_{t-1})$ the conditional probability of executing $B$. Suppose a deterministic sequence $\epsilon_t\in[0,1]$ satisfies $p_t(B\mid\mathcal H_{t-1})\leq\epsilon_t$ for every attainable history. Conditional expectation and the union bound give
\begin{equation}
\mathbb E[N_B(n)]\leq\sum_{t=1}^n\epsilon_t,\qquad \Pr(N_B(n)=0)\geq\max\left\{0,1-\sum_{t=1}^n\epsilon_t\right\}.
\label{eq:adaptive_coverage}
\end{equation}
For the two environments in Proposition~\ref{prop:coverage_identification}, couple the policy randomness and all observed outcomes until the first execution of $B$. Their histories agree before that event, giving a worst-case sign error of at least $\frac12\max\{0,1-\sum_{t=1}^n\epsilon_t\}$. A set excluded throughout training has zero ordinary coverage.

A separately selected probe batch adds an opportunity to observe $B$. If a fixed ordinary policy executes $B$ with probability $p_B$ on each of $n$ independent runs, and an independently selected probe batch includes $B$ with probability $r_B$, the probability of at least one direct observation is
\begin{equation}
1-(1-p_B)^n(1-r_B).
\label{eq:probe_coverage}
\end{equation}
Equation~\ref{eq:probe_coverage} isolates the coverage added by an independent probe batch. The acquisition procedure in Section~\ref{sec:probing} instead conditions on available training data, as formalized by the following surrogate.

\subsection{Gaussian Beliefs for Prospective Execution Evidence}
\label{app:observation}

\begingroup
The EVSI criterion in Section~\ref{sec:probing} evaluates how an additional execution can improve retrieval decisions under the uncertainty of memory-set uplift. At the start of each acquisition batch, we construct a Gaussian belief using the current scorer and the execution feedback collected so far. All quantities below are conditional on the available data $\mathcal D$, with the scorer held fixed during batch construction. The belief requires a representation that relates different memory sets and an uncertainty model that accounts for both the coverage and reliability of previous observations.

The scorer evaluates a task context $c$ together with a complete memory set $S$. Let $h_\theta(c,S)\in\mathbb R^d$ be the final-layer hidden state at the last non-padding token of the same serialized input used for set-value prediction in Appendix~\ref{app:scorer}. We define the Gaussian feature as
\begin{equation}
\varphi(c,S)=\frac{h_\theta(c,S)}{\max\{\|h_\theta(c,S)\|_2,\delta\}},\qquad \delta>0,
\label{eq:gaussian_feature_construction}
\end{equation}
where $\delta$ ensures numerical stability. For the Qwen3-Reranker-0.6B scorer, $d=1024$. These features follow the representation learned for execution-based uplift prediction. The shared features induce covariance between the uncertain values of two sets $S_a$ and $S_b$ through $\varphi(c,S_a)^\top\Sigma\varphi(c,S_b)$, allowing feedback on one set to inform comparisons involving another.

A rollout provides a sample of execution utility, whose variability determines the reliability of an observed uplift label. We calibrate this variability before EVSI-guided acquisition for each benchmark and frozen executor. From the training split, we select $M$ distinct task contexts and prespecify one feasible memory set for each context. Each task is executed $R\geq2$ times with that memory set and $R$ times without memory under the same execution protocol, using independent rollout randomness. Let $U_{ir}^{+}$ and $U_{ir}^{0}$ be the utilities of the $r$-th memory-assisted and no-memory executions for task $i$, and let $\bar U_i^{a}=R^{-1}\sum_{r=1}^{R}U_{ir}^{a}$ for $a\in\{+,0\}$. We estimate the two noise variances by pooling within-task variation,
\begin{equation}
\widehat\nu_a^{\,2}
=\frac{1}{M(R-1)}\sum_{i=1}^{M}\sum_{r=1}^{R}
\bigl(U_{ir}^{a}-\bar U_i^{a}\bigr)^2,
\qquad a\in\{+,0\}.
\label{eq:gaussian_noise_calibration}
\end{equation}
A prespecified variance floor $\nu_{\min}^2>0$ gives $\nu_a^2=\max\{\widehat\nu_a^{\,2},\nu_{\min}^2\}$. Under the independent execution randomness used for calibration, the effective variance of the memory-assisted minus no-memory uplift observation is estimated as $\widehat\sigma_n^2=\nu_+^2+\nu_0^2$. This estimated scalar noise variance is used in the practical diagonal-precision fit and in prospective probe valuation. The calibration is conducted separately for each benchmark and frozen executor, and the resulting variance is held fixed between calibration updates. When labels reuse a no-memory execution, their observation errors can be correlated, as described by the general measurement model below.

The execution history provides information about which feature directions have received direct feedback. Let $\mathcal D_{\mathrm{cov}}\subseteq\mathcal D$ contain the most recent $N_{\mathrm{cov}}$ distinct executed task--memory-set pairs available before an acquisition batch, or all such pairs when fewer are available. Write $\Phi_{\mathcal D}\in\mathbb R^{N\times d}$ for their normalized feature matrix, where $N=|\mathcal D_{\mathrm{cov}}|$ and $\varphi_{ij}$ denotes its $(i,j)$ entry. The implementation estimates a diagonal latent covariance using the empirically calibrated scalar observation-noise variance $\widehat\sigma_n^2$ and regularization $\lambda>0$,
\begin{equation}
\Sigma=\operatorname{diag}\!\left(\left\{\left(\lambda+\frac{\sum_{i=1}^{N}\varphi_{ij}^{2}}{\widehat\sigma_n^2}\right)^{-1}\right\}_{j=1}^{d}\right).
\label{eq:gaussian_covariance_construction}
\end{equation}
The regularized precision increases in feature directions covered by reliable execution observations, while the diagonal representation keeps uncertainty estimation tractable in the scorer's feature space. When the scorer is updated, the historical inputs are encoded again with its current parameters and the covariance is reconstructed before the next acquisition batch. The resulting Gaussian belief supplies the joint distribution used to evaluate prospective execution evidence.
\endgroup

For a prospective batch $Q$ of $n_Q=|Q|$ probes, let $y_Q\in\mathbb R^{n_Q}$ be the label vector, $\mu_Q\in\mathbb R^{n_Q}$ collect the predictions $G_\theta(c_q,B_q)$, and $\Phi_Q\in\mathbb R^{n_Q\times d}$ have rows $\varphi(c_q,B_q)^\top$. The residual execution noise $\varepsilon_Q$ has covariance $\Omega_Q\in\mathbb R^{n_Q\times n_Q}$. The joint measurement model is
\begin{equation}
y_Q=\mu_Q+\Phi_Q\xi+\varepsilon_Q,\qquad \xi\sim\mathcal N(0,\Sigma),\qquad \varepsilon_Q\sim\mathcal N(0,\Omega_Q).
\label{eq:measurement}
\end{equation}
The latent vector $\xi$ and execution noise $\varepsilon_Q$ are independent conditional on $\mathcal D$. The conventions $G_\theta(c,\emptyset)=0$ and $\varphi(c,\emptyset)=\mathbf0$ fix the empty-set uplift at zero. The feature basis and covariance refer to the same fixed scorer state.

For the comparisons in Section~\ref{sec:probing}, write $g_j=A_j[F]-b_j$, where $A_j$ is a linear functional of finitely many set values and $b_j$ is a fixed reference in utility units. A two-set comparison has $A_j[F]=F(c,B_j^1)-F(c,B_j^0)$ and $b_j=0$. A threshold comparison is included when its threshold is the utility of a genuine reference action. Let $\mu(c,B)=G_\theta(c,B)$ be the mean function and $h_j=A_j[\varphi]\in\mathbb R^d$ the feature contrast obtained by applying $A_j$ to each coordinate. Then $\mu_j=A_j[\mu]-b_j$ and $g_j=\mu_j+h_j^\top\xi$, with $h_j=\varphi(c,B_j^1)-\varphi(c,B_j^0)$ for two sets. Let $\sigma_j\geq0$ be the prior standard deviation of the margin, $k_j\in\mathbb R^{n_Q}$ the observation--margin covariance vector, and $W_Q\in\mathbb R^{n_Q\times n_Q}$ the observation covariance. Conditional on $\mathcal D$,
\begin{equation}
\begin{aligned}
\sigma_j^2&=h_j^\top\Sigma h_j,\qquad k_j=\Cov(y_Q,g_j)=\Phi_Q\Sigma h_j,\\
W_Q&=\Var(y_Q)=\Phi_Q\Sigma\Phi_Q^\top+\Omega_Q.
\end{aligned}
\label{eq:gaussian_covariances}
\end{equation}
Assuming $W_Q$ is positive definite, Gaussian conditioning gives
\begin{align}
\mathbb E[g_j\mid\mathcal D,y_Q]&=\mu_j+k_j^\top W_Q^{-1}(y_Q-\mu_Q),\label{eq:posterior_mean}\\
\Var(g_j\mid\mathcal D,y_Q)&=\sigma_j^2-k_j^\top W_Q^{-1}k_j.\label{eq:posterior_variance}
\end{align}
Because the conditional variance is independent of the realized outcome, the distribution of the future posterior mean is determined before executing any candidate probe,
\begin{equation}
\mathbb E[g_j\mid\mathcal D,y_Q]\sim\mathcal N(\mu_j,s_j^2(Q)),\qquad s_j^2(Q)=k_j^\top W_Q^{-1}k_j.
\label{eq:preposterior}
\end{equation}
Equation~\ref{eq:preposterior} supplies the resolved variance in Eq.~\ref{eq:resolved_variance} and the pre-observation distribution required by Theorem~\ref{thm:evsi}.

A shared no-memory execution creates correlated observation noise. If assisted-run noises $\eta_i$ are mutually independent with variances $\nu_i^2$ and an independent baseline noise $\eta_0$ has variance $\nu_0^2$, then $\varepsilon_i=\eta_i-\eta_0$ implies
\begin{equation}
\Omega_Q=\operatorname{diag}(\nu_1^2,\ldots,\nu_{n_Q}^2)+\nu_0^2\mathbf1\mathbf1^\top.
\label{eq:baseline_noise}
\end{equation}
Here $\mathbf1\in\mathbb R^{n_Q}$ is the all-ones vector. For several independent baseline groups, let $Z$ be the binary membership matrix with one row per probe and one column per reference. Let $D_0$ and $D_{\mathrm{assist}}$ be diagonal matrices of reference-noise and assisted-run variances, respectively. Then $\Omega_Q=D_{\mathrm{assist}}+ZD_0Z^\top$. For an already observed reference, $\Omega_Q$ denotes the noise covariance conditional on that observation. Covariances induced by coupled executions can likewise be represented in the general measurement model. For the reported EVSI implementation, prospective observations use the estimated marginal-noise approximation $\Omega_Q\approx\widehat\sigma_n^2 I$. Correlations between set-value predictions remain represented by $\Phi_Q\Sigma\Phi_Q^\top$; the reference-sharing structure above explains the underlying observation process without replacing the scalar-noise approximation used in the computations. Appendix~\ref{app:utility} specifies how shared references enter rollout accounting.

For a set fixed before paired execution, Eq.~\ref{eq:observed_uplift} has expectation $\Delta(c,S)$ whenever both trajectories have the prescribed marginals, including when their randomness is coupled. Labels retained after outcome-dependent candidate selection have a different sampling rule. Appendix~\ref{app:writing} states the rule used for writing-validation data.

For a fixed feature basis, let $\Sigma_Q=\Var(\xi\mid\mathcal D,y_Q)$ denote the posterior latent covariance. Gaussian conditioning gives
\begin{equation}
\Sigma_Q=\Sigma-\Sigma\Phi_Q^\top W_Q^{-1}\Phi_Q\Sigma,\qquad \mathbb E[\xi\mid\mathcal D,y_Q]=\Sigma\Phi_Q^\top W_Q^{-1}(y_Q-\mu_Q).
\label{eq:gaussian_parameter_update}
\end{equation}
This hypothetical posterior scores measurements before execution. The batch procedure in Appendix~\ref{app:probe_algorithm} holds the belief fixed during selection, and the training procedure in Appendix~\ref{app:training} uses realized labels to update the scorer. The posterior covariance is defined in its associated feature basis, so a change of basis also changes its matrix representation.

\subsection{Decision Value of Gaussian Observations}
\label{app:evsi_proof}

To establish Theorem~\ref{thm:evsi}, consider two available decisions with uncertain utilities $F^+$ and $F^-$ and margin $g=F^+-F^-$. This includes a two-set comparison and a fixed-threshold decision when the threshold is the utility of its reference action. The best expected utility before observation is $\mathbb E[F^-\mid\mathcal D]+\pos{\mathbb E[g\mid\mathcal D]}$. After observing $y_Q$, it is $\mathbb E[F^-\mid\mathcal D,y_Q]+\pos{\mathbb E[g\mid\mathcal D,y_Q]}$. Taking the pre-observation expectation cancels the reference term by the tower property. Hence binary EVSI is $\mathbb E_{y_Q\mid\mathcal D}[\pos{\mathbb E[g\mid\mathcal D,y_Q]}]-\pos{\mathbb E[g\mid\mathcal D]}$ even when the reference alternative is uncertain.

\begin{lemma}[Positive part of a Gaussian variable]
Let $X$ be Gaussian with mean $\mu$ and standard deviation $s>0$. Then $\mathbb E[\pos X]=s\phi(\mu/s)+\mu\Phi(\mu/s)$.
\end{lemma}
\begin{proof}
Substituting $z=(x-\mu)/s$ in the defining integral yields
\begin{equation}
\mathbb E[\pos X]=\int_{-\mu/s}^{\infty}(\mu+sz)\phi(z)\,dz=\mu\Phi(\mu/s)+s\phi(\mu/s),
\end{equation}
using $\phi'(z)=-z\phi(z)$ and Gaussian symmetry.
\end{proof}

\begin{proof}[Proof of Theorem~\ref{thm:evsi}]
Equation~\ref{eq:preposterior} gives $\mu'_j\sim\mathcal N(\mu_j,s_j^2(Q))$. For $s_j=s_j(Q)>0$, applying the lemma and subtracting $\pos{\mu_j}$ gives
\begin{equation}
 E_j(Q)=s_j\phi(\mu_j/s_j)+\mu_j\Phi(\mu_j/s_j)-\pos{\mu_j}.
\end{equation}
For $\mu_j\geq0$, this becomes $s_j\phi(\mu_j/s_j)-\mu_j\Phi(-\mu_j/s_j)$; Gaussian symmetry gives the same expression with $|\mu_j|$ when $\mu_j<0$. This is Eq.~\ref{eq:evsi}. When $s_j=0$, the future posterior mean equals the current mean almost surely and the EVSI is zero.
\end{proof}

\begin{corollary}[Dependence on margin and resolvable uncertainty]
Writing $a=|\mu_j|$ for the absolute decision margin and $s>0$ for the resolved standard deviation, define $E(a,s)=s\phi(a/s)-a\Phi(-a/s)$. Then
\begin{equation}
\frac{\partial E}{\partial s}=\phi(a/s)>0,\qquad \frac{\partial E}{\partial a}=-\Phi(-a/s)<0,\qquad 0\leq E_j(Q)\leq\frac{s_j(Q)}{\sqrt{2\pi}}\leq\frac{\sigma_j}{\sqrt{2\pi}}.
\label{eq:evsi_monotonicity}
\end{equation}
\end{corollary}
\begin{proof}
Differentiation and $\phi'(t)=-t\phi(t)$ give the derivatives. The maximum at fixed $s$ occurs at $a=0$. Jensen's inequality for the convex positive-part function gives nonnegativity, and nonnegative conditional variance gives $s_j^2\leq\sigma_j^2$.
\end{proof}

The decision meaning of EVSI follows from opportunity loss, the utility forgone relative to choosing with full knowledge of the margin. Before observation this loss is $\mathbb E[\pos{g_j}\mid\mathcal D]-\pos{\mu_j}$. After observing $Q$, the corresponding expected loss is $\mathbb E[\pos{g_j}\mid\mathcal D]-\mathbb E_{y_Q\mid\mathcal D}[\pos{\mu'_j}]$, by the tower property. Their difference equals $E_j(Q)$, so EVSI measures the expected reduction in decision opportunity loss. This identity explains the local acquisition objective in Section~\ref{sec:probing}. Appendix~\ref{app:batch} connects it to the subsequent use of observations for neural-model training.

\subsection{Finite-Family Decisions and Information Quantity}
\label{app:multisets}
\label{app:information}

The aggregate $H(Q)$ in Eq.~\ref{eq:batch_acquisition} values a family of local comparisons. We relate this quantity to choosing one set from a finite family and then distinguish decision value from information quantity.

Let $\mathcal B=\{B_0,\ldots,B_{J-1}\}$ be a finite family of $J$ executable sets for a fixed task, with $F_k=F(c,B_k)$ and $\mu_k=\mathbb E[F_k\mid\mathcal D]$. Relabel a predicted-best member as $B_0$, so $\mu_0=\max_{0\leq k<J}\mu_k$. All extrema and sums over $k$ in this subsection range over these $J$ alternatives unless otherwise specified. The EVSI for selecting a single set from this family is
\begin{equation}
 E_{\mathcal B}(Q)=\mathbb E_{y_Q\mid\mathcal D}\!\left[\max_k\mathbb E[F_k\mid\mathcal D,y_Q]\right]-\max_k\mathbb E[F_k\mid\mathcal D].
\label{eq:full_evsi}
\end{equation}
This is the finite-alternative value used in knowledge-gradient reasoning~\citep{frazier2009knowledge}.

\begin{proposition}[Binary comparisons bound finite-family EVSI]
\label{prop:evsi_bounds}
For $J\geq2$, let $g_k=F_k-F_0$ for $k\ne0$ and $E_k(Q)$ be the binary EVSI for the corresponding comparison. With integrable utilities and a coherent common posterior,
\begin{equation}
\max_{k\ne0}E_k(Q)\leq E_{\mathcal B}(Q)\leq\sum_{k\ne0}E_k(Q).
\label{eq:evsi_bounds}
\end{equation}
The bounds allow correlated alternatives, and joint Gaussianity makes each binary term available in closed form.
\end{proposition}
\begin{proof}
Write $g'_k=\mathbb E[F_k-F_0\mid\mathcal D,y_Q]$. Because $\mu_k-\mu_0\leq0$, $E_k(Q)=\mathbb E[\pos{g'_k}]$. Decomposing the posterior maximum relative to $F_0$ and applying the tower property yields
\begin{equation}
 E_{\mathcal B}(Q)=\mathbb E\!\left[\max\{0,g'_1,\ldots,g'_{J-1}\}\right].
\end{equation}
For every outcome, $\pos{g'_k}\leq\max\{0,g'_1,\ldots,g'_{J-1}\}\leq\sum_{k\ne0}\pos{g'_k}$. Taking expectations gives the bounds.
\end{proof}

When $\mathcal G$ consists of one comparison of every alternative in a finite family against the same predicted-best incumbent, with no duplicate or additional comparisons, its unweighted sum $H(Q)$ upper-bounds the finite-family EVSI by Eq.~\ref{eq:evsi_bounds}. For comparisons against different incumbents, $H(Q)$ instead values the collection of local decisions. Duplicate comparisons can count the same decision improvement more than once. For example, for an integrable zero-mean random value $Z$, if $F_1=F_2=Z$ and $F_0=0$, a probe revealing $Z$ gives $E_1=E_2=\mathbb E[\pos Z]$ but $E_{\mathcal B}=\mathbb E[\pos Z]$.

For the scalar Gaussian observation $y_q$ from probe $q$, with predictive variance $v_q=\Var(y_q\mid\mathcal D)>0$, define $\beta_k=\Cov(F_k,y_q\mid\mathcal D)/\sqrt{v_q}$ and $Z=(y_q-\mathbb E[y_q\mid\mathcal D])/\sqrt{v_q}\sim\mathcal N(0,1)$. The posterior means $\mu'_k=\mathbb E[F_k\mid\mathcal D,y_q]$ satisfy $\mu'_k=\mu_k+\beta_kZ$, giving
\begin{equation}
 E_{\mathcal B}(q)=\mathbb E_Z[\max_k(\mu_k+\beta_kZ)]-\max_k\mu_k.
\label{eq:scalar_kg}
\end{equation}
This expression evaluates a single final selection through an upper envelope or one-dimensional integration. By contrast, Eq.~\ref{eq:batch_acquisition} aggregates the local comparisons constructed in Appendix~\ref{app:probe_algorithm}.

This decision-value perspective also clarifies the relationship to information gain.
Let $I(g_j;y_Q\mid\mathcal D)$ denote conditional mutual information, using natural logarithms. For a Gaussian margin with $0<\sigma_j^2-s_j^2(Q)\leq\sigma_j^2$,
\begin{equation}
 I(g_j;y_Q\mid\mathcal D)=\frac12\log\frac{\sigma_j^2}{\sigma_j^2-s_j^2(Q)},\qquad s_j^2(Q)=\sigma_j^2\bigl(1-e^{-2I(g_j;y_Q\mid\mathcal D)}\bigr).
\label{eq:information_bridge}
\end{equation}
The identity follows by subtracting the conditional Gaussian entropy from the prior entropy. For a fixed comparison, $\mu_j$ and $\sigma_j$ are fixed, so mutual information and binary EVSI rank probes in the same order. Across different comparisons, EVSI additionally accounts for the decision margin and utility scale. If a continuous margin is revealed exactly, mutual information diverges while binary EVSI approaches the finite expected value of perfect information.

\subsection{Batch Acquisition and Evidence Timing}
\label{app:batch}

The batch score in Section~\ref{sec:probing} accounts for overlapping evidence before outcomes are observed. Fix $\mathcal D$, the surrogate, the comparison family $\mathcal G$, and the feasible probe pool throughout the following calculation. All expectations and covariances below are conditional on this pre-batch evidence. For a coherent model and $Q\subseteq Q'$, the tower property and conditional Jensen inequality give
\begin{equation}
 \mathbb E\!\left[\pos{\mathbb E[g_j\mid\mathcal D,y_{Q'}]}\right]\geq\mathbb E\!\left[\pos{\mathbb E[g_j\mid\mathcal D,y_Q]}\right].
\label{eq:batch_monotonicity}
\end{equation}
Thus $E_j(Q)$ and $H(Q)$ are monotone with additional observations. Their marginal gains need not decrease. For example, let $X,Y$ be independent standard Gaussian variables, set $g=X$, and consider noiseless probes $y_1=X+Y$ and $y_2=Y$. The second probe alone has no information about $g$, but after the first it reveals $g$ exactly, so its marginal decision value increases.

For a prospective scalar observation $y_q$, define the residual covariance and variance under the fixed pre-batch model as
\begin{align}
\kappa_{jq\mid Q}&=\Cov(g_j,y_q)-\Cov(g_j,y_Q)W_Q^{-1}\Cov(y_Q,y_q),\\
v_{q\mid Q}&=\Var(y_q)-\Cov(y_q,y_Q)W_Q^{-1}\Cov(y_Q,y_q).
\end{align}
When $v_{q\mid Q}>0$, block Gaussian conditioning gives
\begin{equation}
 s_j^2(Q\cup\{q\})=s_j^2(Q)+\frac{\kappa_{jq\mid Q}^{\,2}}{v_{q\mid Q}}.
\label{eq:incremental_resolved}
\end{equation}
For $Q=\emptyset$, the covariance terms involving $y_Q$ are absent and $s_j^2(\emptyset)=0$. For nonempty $Q$, the formula first conditions on $y_Q$ and then on the residual of $y_q$ after its linear prediction from $y_Q$. For a fixed scalar mean $\mu$, let $f_\mu(s)=s\psi(|\mu|/s)$ for $s>0$ and $f_\mu(0)=0$. Then
\begin{equation}
\begin{aligned}
V_t(q)=\sum_{j\in\mathcal G}\bigg[&f_{\mu_j}\!\left(\sqrt{s_j^2(Q_t)+\frac{\kappa_{jq\mid Q_t}^{\,2}}{v_{q\mid Q_t}}}\right)-f_{\mu_j}(s_j(Q_t))\bigg].
\end{aligned}
\label{eq:batch_marginal_formula}
\end{equation}
Each term uses the fixed pre-batch mean, so the expression is the difference between two prospective batch values. It accounts for scheduled measurements through their covariance without conditioning on unknown labels. For an exactly redundant observation with $v_{q\mid Q_t}=0$, its incremental contribution is defined as zero. Linear solves using a Cholesky factor of $W_Q$ avoid forming its inverse explicitly.

For a fixed batch state, let $\mathcal Q_t=\mathcal Q\setminus Q_t$ be the remaining probes and $\mathcal Q_t^+=\{q\in\mathcal Q_t:V_t(q)>\lambda_{\mathrm{probe}}\}$ the cost-admissible candidates. When $\mathcal Q_t^+$ is nonempty, the softmax in Eq.~\ref{eq:probe_softmax} maximizes
\[
\sum_{q\in\mathcal Q_t^+}p(q)V_t(q)+\tau_{\mathrm{evsi}}\mathcal H(p),
\qquad \mathcal H(p)=-\sum_{q\in\mathcal Q_t^+}p(q)\log p(q),
\]
over probability distributions $p$ supported on $\mathcal Q_t^+$, where $\mathcal H(p)$ is Shannon entropy with $0\log0=0$ and $\tau_{\mathrm{evsi}}>0$ is its weight. A Lagrange multiplier for normalization yields Eq.~\ref{eq:probe_softmax}; selection stops when the admissible set is empty.

A training probe supplies a label for a later scorer update, whereas Eq.~\ref{eq:full_evsi} values a decision made from the observation's posterior. To distinguish these roles, let $\mathcal T(\theta,\mathcal D,y_Q)$ denote the parameter update, let $S_\theta(c')$ be deterministic retrieval on a future task, and let $p_{\mathrm{future}}$ be the future-task distribution. The expected change in future execution uplift induced by this update is
\begin{equation}
 \mathbb E_{y_Q\mid\mathcal D}\mathbb E_{c'\sim p_{\mathrm{future}}}\!\left[\Delta\bigl(c',S_{\mathcal T(\theta,\mathcal D,y_Q)}(c')\bigr)-\Delta\bigl(c',S_\theta(c')\bigr)\right].
\label{eq:training_information_value}
\end{equation}
The implemented $H(Q)$ is computable before execution from current local comparisons. Algorithm~\ref{alg:training_flow} uses the acquired outcomes to train the shared scorer. Section~\ref{sec:rq3} examines probing within the full system, while Section~\ref{sec:rq4} compares acquisition strategies across additional rollout budgets.

\section{Algorithmic Details}
\label{app:algorithms}

Algorithm~\ref{alg:training_flow} gives training and test-time execution. The seed trajectories are $\mathcal T_0$, and $\{c_t\}_{t=1}^{n}$ is the task stream across all training epochs. The buffers $\mathcal D_v$ and $\mathcal D_{\pi}$ store value observations and recent sampled paths, respectively, while $\mathcal W$ stores the current acquisition window. The operator $\uplus$ appends records without removing repetitions, and observations retain their reference identities for pairwise supervision. \textsc{Observe} returns the memory-assisted trajectory, its no-memory reference, and their uplift from Eq.~\ref{eq:observed_uplift}, using the reference-reuse protocol in Appendix~\ref{app:utility}. \textsc{WriteMemory} implements Appendix~\ref{app:writing}, updating the memory pool and returning the selected validation record in $\mathcal D_w$ for addition to $\mathcal D_v$, regardless of admission.

\begin{algorithm}[t]
\caption{\textsc{UpliftMem} training and inference}
\label{alg:training_flow}
\small
\begin{algorithmic}[1]
\Require Frozen executor $\mathcal E$, retrieval budget $K$, acquisition interval $h$, probe budget $B_{\mathrm p}$, update indices $\mathcal I_{\mathrm{upd}}$
\Function{Train}{$\mathcal T_0,\{c_t\}_{t=1}^{n},\theta$}
    \State $(\mathcal M,\mathcal D_v)\gets\Call{Bootstrap}{\mathcal T_0,\mathcal E}$
    \State $\theta\gets\Call{Optimize}{\theta,\mathcal L_{\mathrm{value}}(\mathcal D_v)}$
    \State $\mathcal D_\pi\gets\emptyset,\quad\mathcal W\gets\emptyset$
    \For{$t=1,\ldots,n$}
        \State $C_t\gets\Call{Recall}{c_t,\mathcal M}$
        \State $(S_t,\mathbf a_t)\gets\Call{SamplePath}{\pi_\theta,c_t,C_t,K}$
        \State $(\tau_t,\tau_{\emptyset,t},y_t)\gets\Call{Observe}{c_t,S_t,\mathcal E}$
        \State $\mathcal D_v\gets\mathcal D_v\uplus\{(c_t,S_t,y_t)\}$
        \State $\mathcal D_\pi\gets\mathcal D_\pi\uplus\{(c_t,\mathbf a_t,y_t)\}$
        \State $(\mathcal M,\mathcal D_w)\gets\Call{WriteMemory}{c_t,\tau_t,\tau_{\emptyset,t},\mathcal M}$
        \State $\mathcal D_v\gets\mathcal D_v\uplus\mathcal D_w$
        \State $\mathcal W\gets\mathcal W\uplus\{(c_t,C_t,S_t,\mathbf a_t)\}$
        \If{$t\bmod h=0$}
            \State $(\mathcal Q,\mathcal G)\gets\Call{LocalComparisons}{\mathcal W}$
            \State $Q\gets\Call{SelectProbes}{\mathcal Q,\mathcal G,G_\theta,B_{\mathrm p}\revised{, \lambda_{\mathrm{probe}}}}$
            \State $\mathcal D_v\gets\mathcal D_v\uplus\Call{ObserveProbes}{Q,\mathcal E}$
            \State $\mathcal W\gets\emptyset$
        \EndIf
        \If{$t\in\mathcal I_{\mathrm{upd}}$}
            \State $\mathcal L\gets\mathcal L_{\mathrm{value}}(\mathcal D_v)+\mathcal L_{\mathrm{policy}}(\mathcal D_\pi)$
            \State $\theta\gets\Call{Optimize}{\theta,\mathcal L}$
            \State $\mathcal D_\pi\gets\emptyset$
        \EndIf
    \EndFor
    \State \Return $(G_\theta,\mathcal M)$
\EndFunction
\Statex
\Function{Execute}{$c,G_\theta,\mathcal M$}
    \State $C\gets\Call{Recall}{c,\mathcal M}$
    \State $S\gets\Call{BeamRetrieve}{G_\theta,c,C,K,2}$
    \State \Return $\mathcal E(c,S)$
\EndFunction
\end{algorithmic}
\end{algorithm}
\textsc{Bootstrap} constructs the initial pool and uplift observations from the seed trajectories. \textsc{SamplePath} follows Eq.~\ref{eq:retrieval_policy} and returns its terminal set and recorded path, including the states and feasible-action masks needed for policy training. \textsc{LocalComparisons} constructs the probe pool $\mathcal Q$ and comparison family $\mathcal G$ in Appendix~\ref{app:probe_algorithm}. The score $H$ in Eq.~\ref{eq:batch_acquisition} uses the current scorer and evidence at batch entry, with this belief fixed throughout selection. \textsc{Optimize} applies the configured optimizer to the stated objective using eligible observations, with $\mathcal L_{\mathrm{policy}}(\emptyset)=0$. Clearing $\mathcal D_{\pi}$ after an update prevents reuse of paths sampled under earlier scorer states. \textsc{BeamRetrieve} implements Appendix~\ref{app:search} with width two and includes the empty set among its candidates.

\subsection{Memory Set Search and Stopping}
\label{app:search}
The search in Section~\ref{sec:retrieval} operates on subsets $S\subseteq \mathcal C(c)$. Used memories and additions violating the count or token constraints are masked. A sampled $\STOP$ returns the current set, including $\emptyset$, and termination is forced when no feasible addition remains. Candidate states are deduplicated by their members and serialized consistently for scoring. The set-valued formulation treats each subset as a complete context under the memory-presentation protocol in Section~\ref{sec:problem}.

For a sequence of $\ell$ memory additions $m_1,\ldots,m_\ell$, let $S_0=\emptyset$ and $S_t=S_{t-1}\cup\{m_t\}$. The conditional gains in Eq.~\ref{eq:conditional_gain} telescope,
\begin{equation}
\sum_{t=1}^{\ell} g_\theta(m_t\mid c,S_{t-1})
=G_\theta(c,S_\ell)-G_\theta(c,\emptyset)
=G_\theta(c,S_\ell).
\label{eq:telescoping}
\end{equation}
This identity relates accumulated gains to the terminal set value. The path probabilities in Eq.~\ref{eq:retrieval_policy} also contain state-dependent normalizers, so search ranks complete sets by $G_\theta$ rather than by path likelihood.

Training samples additions and stopping actions from Eq.~\ref{eq:retrieval_policy}. At test time, beam width is $W=2$. An active beam and a cross-depth archive are initialized with $\emptyset$. At each depth, all feasible one-memory extensions are scored, up to two are retained per parent, and the two highest-valued distinct extensions form the next beam. Surviving sets enter the archive. Search continues to $K$ or until no feasible extension remains and returns the highest-valued archived set using a deterministic tie rule. Retaining an intermediate set with negative immediate gain allows later additions to expose complementarity.

For $N=|\mathcal C(c)|$, beam width $W$, and memory-count cap $K$, this procedure requires $O(WNK)$ set-score evaluations, excluding recall and executor cost. Deduplication can reduce this count.

\subsection{Prospective Probe Batch Construction}
\label{app:probe_algorithm}

For each task in the preceding acquisition window, retrieval supplies a selected set $S$ and the recalled pool $\mathcal C(c)$. Feasible local candidates comprise the complete set $S$ and one-memory edits of it, including removals $S\setminus\{m\}$, additions $S\cup\{m'\}$, and replacements $(S\setminus\{m\})\cup\{m'\}$ for $m\in S$ and $m'\in \mathcal C(c)\setminus S$. The edit operator is also defined on retrieval prefixes, subject to the same feasibility constraints.
All candidates satisfy the original retrieval constraints. Distinct sets are identified by their members, consistently serialized, and deduplicated. The feasible pool $\mathcal Q$ consists of task--set pairs $q=(c_q,B_q)$, each of which supplies one complete-set observation when executed.

The comparison family $\mathcal G$ consists of the local choices exposed by the candidate neighborhood. A comparison against a removal or addition assesses whether retaining or extending the current set changes expected utility. Comparing replacements of a shared base assesses which continuation is preferred. A stopping comparison uses the difference between a feasible extension and its current base, with zero as the reference margin. The functionals $A_j$ and references $b_j$ in Appendix~\ref{app:observation} encode these differences. Each distinct comparison is counted once with a consistent orientation, since reversing its two alternatives leaves binary EVSI unchanged. The comparison family and feasible probe pool are fixed for each acquisition batch.

For the fixed local surrogate, initialize the scheduled batch as $Q_0=\emptyset$ and let $\mathcal Q_t=\mathcal Q\setminus Q_t$ contain its remaining feasible probes. For every $q\in\mathcal Q_t$, use Eq.~\ref{eq:incremental_resolved} to calculate the resolved variances for $Q_t\cup\{q\}$, then use Eq.~\ref{eq:batch_marginal_formula} to compute the full aggregate increment $V_t(q)=H(Q_t\cup\{q\})-H(Q_t)$ from Eq.~\ref{eq:batch_acquisition}. 
\revised{Let $\mathcal Q_t^+
=\{q\in\mathcal Q_t:
V_t(q)>\lambda_{\mathrm{probe}}\}$ denote the candidates with positive net acquisition value. With acquisition temperature $\tau_{\mathrm{evsi}}>0$, the next-probe distribution is}
\begin{equation}
p_t(q)=
\frac{\exp\!\left(V_t(q)/\tau_{\mathrm{evsi}}\right)}
{\sum_{q'\in\mathcal Q_t^+}
\exp\!\left(V_t(q')/\tau_{\mathrm{evsi}}\right)},
\qquad q\in\mathcal Q_t^+.
\label{eq:probe_softmax}
\end{equation}
\revised{The EVSI-guided selection terminates when $\mathcal Q_t^+$ becomes empty. Any reserved random exploration is allocated separately.} Algorithm~\ref{alg:training_flow} completes batch selection before executing its probes. The incremental covariance calculation in Eq.~\ref{eq:incremental_resolved} accounts for evidence already scheduled under a fixed pre-batch belief. Between batches, the features and covariance are refreshed consistently with the scorer. Appendix~\ref{app:utility} accounts for selected executions and their no-memory references.

\subsection{Iterative Memory Writing and Validation}
\label{app:writing}

\begin{figure}[t]
    \centering
    \includegraphics[width=\linewidth]{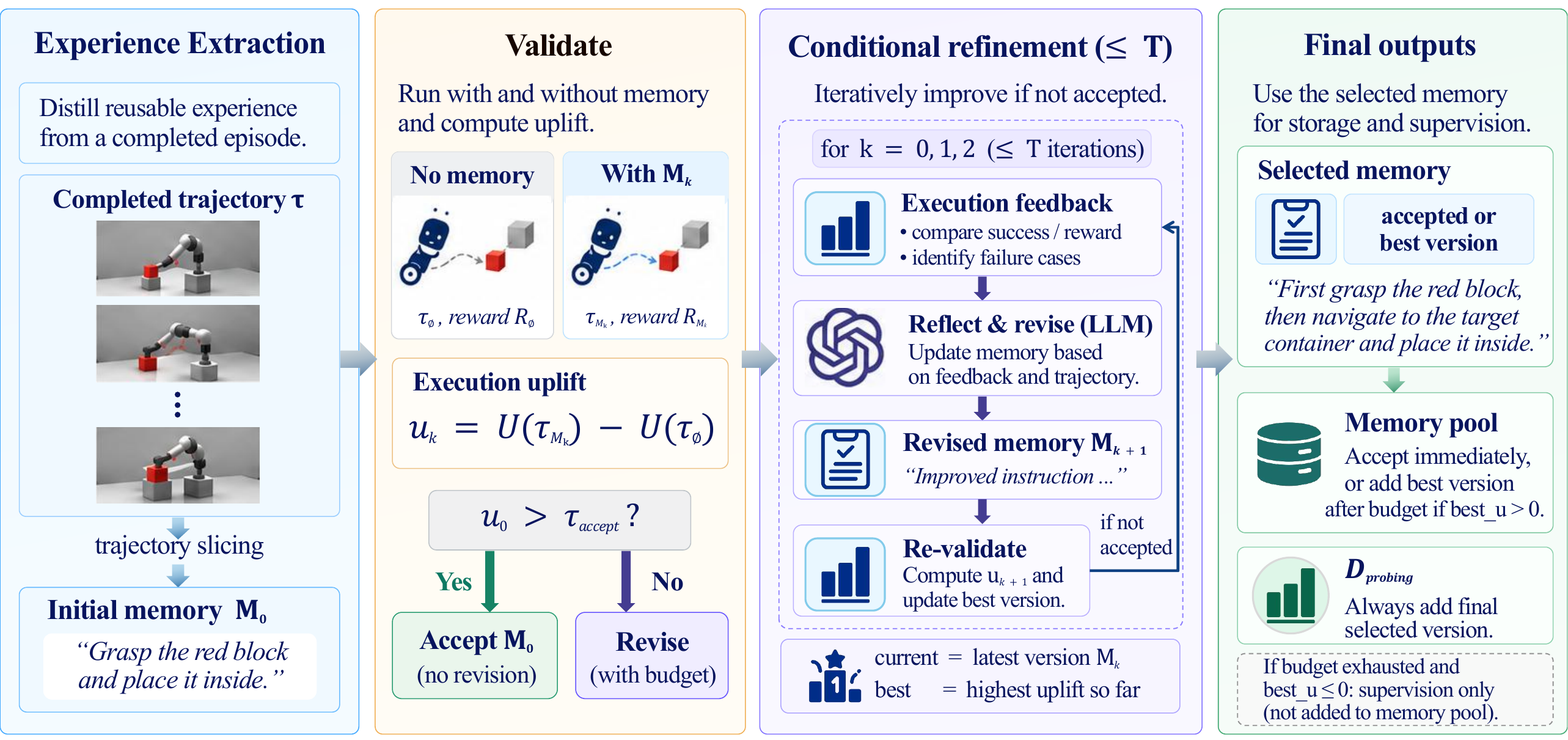}
    \caption{Execution-validated memory writing, from trajectory-derived candidates to validated memory admission and value supervision.}
    \label{fig:memory_writing}
\end{figure}

The writing stage of Algorithm~\ref{alg:training_flow} uses completed episodes with either sign of retrieval uplift. An episode's label in Eq.~\ref{eq:observed_uplift} evaluates the retrieved context, so each newly distilled memory requires its own validation. Trajectory slicing distills an initial candidate $M_0$ using the observed actions and outcomes, including reusable procedures and corrective guidance \citep{zhao2024expel,cao2026reme}. As illustrated in Figure~\ref{fig:memory_writing}, the frozen executor therefore validates each candidate $M_k$, the memory after $k$ revisions, alone on its source task $c$. Against the cached no-memory reference, its observed uplift is
\begin{equation}
u_k=U(\tau_{\{M_k\}})-U(\tau_{\emptyset}).
\label{eq:writing_validation}
\end{equation}

A candidate with $u_k>\tau_{\mathrm{accept}}=0$ is accepted immediately. Otherwise, the writer uses the realized trajectory, outcome, and failure diagnostics to revise the latest candidate into $M_{k+1}$. The highest-uplift version is tracked separately as revision proceeds, for at most $T=3$ revisions. The selected candidate is the first version passing the threshold, or the highest-uplift version when the budget is exhausted.

The selected memory, source task, and observed uplift enter $\mathcal D_v$ whether or not admission occurs. The pool admits the selected version when its uplift exceeds $\tau_{\mathrm{accept}}$. Admission uses singleton validation on the source task, whereas $G_\theta$ evaluates later combinations. Appendix~\ref{app:utility} counts every executed validation attempt, and Appendix~\ref{app:prompt_revision} gives the revision interface.

\section{Training Details}
\label{app:training}

The following implementation details instantiate the scorer and learning objectives in Sections~\ref{sec:uplift} and~\ref{sec:learning}.

\subsection{Scorer Parameterization}
\label{app:scorer}
The scorer represents the task as a query and the complete memory set as a document under the instruction in Appendix~\ref{app:prompt_scorer}. Let $z_\theta(c,S)$ be its yes-minus-no token-logit score. Affine calibration maps this score to utility units,
\begin{equation}
G_\theta(c,S)=
\begin{cases}
a\,z_\theta(c,S)+b,&S\ne\emptyset,\\
0,&S=\emptyset,
\end{cases}.
\label{eq:scorer_calibration}
\end{equation}
The scale $a\in\mathbb R$ and offset $b\in\mathbb R$ are optimized jointly with the low-rank adaptation (LoRA) parameters. The notation $\theta$ includes these calibration coefficients and the trainable adapter parameters. The empty-set convention anchors predictions to the same reference as $\Delta(c,\emptyset)=0$.

\subsection{Value Supervision}
\label{app:losses}
Regression in Eq.~\ref{eq:main_value_losses} weights uplift magnitude using
\begin{equation}
w_{\mathrm{mag}}(y)=\max\{w_{\min},\min(1,|y|/\tau_{\mathrm{ref}})\},
\label{eq:magnitude_weight}
\end{equation}
where $0<w_{\min}\leq1$ is the minimum weight and $\tau_{\mathrm{ref}}>0$ is the scale at which it reaches one. Sign supervision retains $|y|>\epsilon_{\mathrm{sign}}$. Pairwise supervision retains sets from the same task and realized no-memory reference with $|y_a-y_b|>\epsilon_{\mathrm{rank}}$, where $\epsilon_{\mathrm{rank}}\geq0$ is the minimum label-gap threshold. Preferences follow the larger observed uplift, including when both labels have the same sign.

\subsection{Outcome-Weighted Retrieval-Path Learning}
\label{app:policy_training}
For an executed retrieval path in Eq.~\ref{eq:main_policy_loss}, let $y_i=y(c_i,S_{i,L_i})$ be the terminal set's observed uplift. The path length $L_i$ includes a sampled terminating $\STOP$. Its signed outcome weight is
\begin{equation}
\omega_i=\begin{cases}
0,&|y_i|\leq\epsilon_{\mathrm{policy}},\\
\min(y_i,A_{\max}),&y_i>\epsilon_{\mathrm{policy}},\\
\eta_{\mathrm{neg}}\max(y_i,-A_{\max}),&y_i<-\epsilon_{\mathrm{policy}},
\end{cases}
\label{eq:policy_weight}
\end{equation}
where $\epsilon_{\mathrm{policy}}\geq0$ excludes near-zero outcomes, $A_{\max}>0$ clips their magnitude, and $\eta_{\mathrm{neg}}\in[0,1]$ scales negative feedback.

Policy updates use the latest batch of paths sampled under the current scorer, evaluating their likelihoods with the recorded set states and action masks. Paths from earlier batches are not replayed for this objective.

\section{Prompts for Execution, Valuation, and Memory Update}
\label{app:prompts}
The interfaces below implement the executor and scorer in Sections~\ref{sec:preliminaries}--\ref{sec:learning} and the writer in Appendix~\ref{app:writing}. Literal templates are distinguished from builder fields and structured output contracts. Braced fields are populated at runtime.

\subsection{Information Boundaries}
The executor receives the public task, observations, actions where applicable, and retrieved memory. The current task defines the requirements, while memory supplies transferable experience. Hidden goals, test code, expected answers, and evaluator internals are excluded from executor prompts. Writers receive trajectories and outcome diagnostics, but stored guidance is grounded in visible evidence. Quantitative execution outcomes supervise the value scorer.

\begin{table}[ht]
\centering\small
\caption{Prompt interfaces and their principal inputs and outputs.}
\label{tab:prompt_interfaces}
\begin{tabular}{@{}>{\raggedright\arraybackslash}p{0.18\linewidth}>{\raggedright\arraybackslash}p{0.45\linewidth}>{\raggedright\arraybackslash}p{0.27\linewidth}@{}}
\toprule
Interface & Principal inputs & Output \\
\midrule
Executor & Public task, observation, retrieved set & Environment action or code \\[2pt]
Value scorer & Task context, complete candidate set & Calibrated uplift score \\[2pt]
Memory writer & Task, trajectory, outcome, diagnosis & Transferable procedure or rule \\[2pt]
Revision & Latest candidate, paired utility, realized trajectory & Revised candidate \\
\bottomrule
\end{tabular}
\end{table}

\subsection{Executor Prompts}
\label{app:prompt_actors}
The three executors use different action and generation contracts. E1--E3 give their system instructions and the message fields through which retrieved memory enters the task. The same executor instructions and examples apply to paired memory-assisted and no-memory executions.

\begin{promptbox}[breakable]{\promptheading{E1}{ALFWorld execution}}
\promptrole{System}\begin{lstlisting}[style=prompttext]
You are an ALFWorld household robot agent. Follow the ReAct ALFWorld
transcript style from the examples. Do not output hidden reasoning,
<think> tags, markdown, or explanations. /no_think
At every turn, output exactly one environment command. Prefer the
provided Available actions list when it is present.
\end{lstlisting}
\promptrole{Initial user message}\begin{lstlisting}[style=prompttext]
Examples:
{example_texts}

Relevant memories:
[{memory_id}] {memory_text}
...

Now solve the next ALFWorld task in the same command style.
Return exactly one next action line.

{task_description}

Available actions: {command_1}; {command_2}; ...
>
\end{lstlisting}
\promptrole{Subsequent observation}\begin{lstlisting}[style=prompttext]
{observation}
Available actions: {command_1}; {command_2}; ...
>
\end{lstlisting}
\end{promptbox}
The ReAct examples are selected by task type. The current admissible-command list is shared by prompt construction and action validation. Retrieved trajectories describe previous experience; the current episode begins from its own initial observation.

\begin{promptbox}[breakable]{\promptheading{E2}{WebShop execution}}
\promptrole{System}\begin{lstlisting}[style=prompttext]
You are a WebShop ReAct agent. Reply exactly:
Thought: <short reasoning>
Action: <one action>

Only use search[keywords] on the Search page or click[exact listed text]. 

Page-aware policy:
1. On the Search page, issue one concise query containing the product type and distinctive non-price attributes.
2. On the Results page, open an unvisited product whose title matches the product type and does not contradict required attributes.
3. On a Product page, use visible page evidence to satisfy instruction-related option groups. Select one matching visible value per required group. Never substitute conflicting pack counts, measurements, or values. If all visible values conflict, return to results.

Rules:
1. Choose only actions executable on the current page and listed in Available actions.
2. Do not repeat identical actions or invent clickables.
\end{lstlisting}
\promptrole{Current user message}\begin{lstlisting}[style=prompttext]
Instruction:
{instruction_text}

Retrieved memories (reusable guidance; do not copy old product IDs/names):
{memory_block}

Recent actions:
{history_block}

Current observation:
{observation_text}

Available actions:
{available_actions_text}

Return exactly:
Thought: <brief reasoning>
Action: search[...] or click[...]
\end{lstlisting}
\end{promptbox}

\begin{promptbox}[breakable]{\promptheading{E3}{BigCodeBench execution}}
\promptrole{System}\begin{lstlisting}[style=prompttext]
You are an expert Python programmer solving BigCodeBench coding tasks.

You may receive implementation notes distilled from successful solutions to related tasks. Use a note only when its contract, library, and data flow apply to the current task. The current task is the sole specification; memory is supporting evidence, not a solution to copy.

Return one self-contained Python solution in a markdown code block, without explanation.

Hard constraints:
- Preserve the starter code, function signature, return contract, specified exceptions, object types, and observable edge-case requirements in the current task.
- Use the libraries and APIs required by the task, import every module used, and do not replace a required dependency because a memory reports a different execution environment.
- Network access, file I/O, randomness, time-based behavior, plotting, and logging are allowed when required by the task, but unrelated side effects should not be introduced.
\end{lstlisting}
\promptrole{User task}The task message is generated by the official \texttt{get\_prompt(task, prompt\_split="instruct")} interface.
\promptrole{Retrieved memory view}\begin{lstlisting}[style=prompttext]
[Retrieved Memory Context]
# Retrieved Successful Implementation Notes
Extract applicable implementation evidence for the current task. Do not copy the source task's signature, literals, return contract, or side effects.

## Note {index} [SUCCESS]
Scope: family={bcb_task_type}; libraries={library_list}
Reusable implementation evidence:
{implementation_only}
\end{lstlisting}
\end{promptbox}

\subsection{Set-Value Scoring Prompt}
\label{app:prompt_scorer}
The \texttt{memory} placeholder contains the complete candidate set under the serialization convention in Appendix~\ref{app:search}. The following template produces the yes/no logits calibrated in Appendix~\ref{app:scorer}.

\begin{promptbox}[breakable]{\promptheading{V1}{Uplift scorer}}
\promptrole{Chat template}\begin{lstlisting}[style=prompttext]
<|im_start|>system
Judge whether the Document meets the requirements based on the Query and the Instruct provided. Note that the answer can only be "yes" or "no".<|im_end|>
<|im_start|>user
<Instruct>: {dataset_instruction}
<Query>: {context}
<Document>: {memory}
<|im_end|>
<|im_start|>assistant
<think>

</think>

\end{lstlisting}
\promptrole{Dataset instructions}\begin{lstlisting}[style=prompttext]
ALFWorld: Predict whether the memory will causally improve the ALFWorld agent for the query context.

WebShop: Predict whether the memory will causally improve the WebShop agent for the query context.

BigCodeBench: Predict whether this memory will causally improve correctness for the current BigCodeBench coding task.
\end{lstlisting}
\end{promptbox}
\subsection{Memory Construction Prompts}
\label{app:prompt_writers}
The writer in Appendix~\ref{app:writing} uses different memory representations for interactive control, shopping, and code generation. W1--W4 give its templates, builder fields, and output contracts. Task outcomes describe the executed trajectory, whereas retrieval uplift compares it with the no-memory reference, so a successful trajectory can still have negative uplift.

\begin{promptbox}[breakable]{\promptheading{W1}{Successful trajectory proceduralization}}
\promptrole{User template}\begin{lstlisting}[style=prompttext]
You are provided with a query and a trajectory taken to solve the query. The trajectory consists of multiple steps of thought, action and observation.

Your task is to generate a workflow based on critical steps to help solve similar queries in the future.

A critical step is one that:
- Has a significant impact on fulfilling the query
- The action is meaningful and contributes to the goal
- The action's outcome is successful

Important: Write the workflow as a natural, coherent paragraph, not a bullet list.

-----EXAMPLE WORKFLOW----
To solve this query, begin by identifying the target object's likely location based on the task description. Navigate to that location and search for the object systematically. Once found, pick up the object and verify you have the correct item. Then, determine the destination where the object should be placed, navigate there, and complete the placement action. Finally, confirm the task is complete by checking the environment state.
-----EXAMPLE END----

Query:
{query}

Trajectory:
{trajectory}

Output the workflow without any explanation or context:
\end{lstlisting}
\promptrole{System constraints}
The system message requests \texttt{/no\_think}, suppresses hidden reasoning, and returns only the workflow memory.
\end{promptbox}
\begin{contractbox}[breakable]{\promptheading{W2}{ALFWorld failure-aware construction}}
\promptrole{Input fields}
Task, task type, failed trajectory, failure category, repetition rate, admissible-action mismatch, target-object confusion, last action, and a same-type reference memory.
\promptrole{Output-contract excerpt}
\begin{lstlisting}[style=prompttext]
Propose a procedural memory of 3-7 numbered steps. Focus on:
1. What to verify before acting
2. Which admissible-action patterns are preferred
3. How to identify the correct target object
4. How to recover without repeating the same failed action

Do NOT mention specific room layouts or object IDs.
Output ONLY the numbered procedure.
\end{lstlisting}
\end{contractbox}
\begin{promptbox}[breakable]{\promptheading{W3}{WebShop memory construction}}
\promptrole{Trajectory proceduralization template}\begin{lstlisting}[style=prompttext]
Extract ONE short reusable WebShop memory from this trajectory.

Instruction: {instruction}
Trajectory:
{trajectory_text}
Reward: {reward}

Output exactly:
Type: <option_selection | premature_buy_prevention | search_query | result_selection | loop_recovery>
Trigger: <the attribute/option situation when to use this>
Action: <one imperative rule>
\end{lstlisting}
\promptrole{Failure-aware template}\begin{lstlisting}[style=prompttext]
Write ONE short WebShop memory for the failed trajectory.

Public instruction: {instruction}
Trajectory: {trajectory_text}
Diagnosis: {public_failure_signals}
Public trajectory diff: {public_failure_diff}
Optional related memories: {references_or_empty}

Output exactly:
Type: <option_selection | premature_buy_prevention | search_query | result_selection | loop_recovery>
Trigger: <when to use this>
Action: <one imperative rule that prevents the failure>
\end{lstlisting}
\promptrole{Construction constraints}
Each memory contains one atomic rule supported by the public trajectory. Source-specific prices, product identifiers and names, exact queries, and literal search/click actions are excluded. Query revision follows the current page state. For failures, the type preference is option selection, premature purchase, search query, result selection, then loop recovery, subject to support from visible actions and pages.
\end{promptbox}
\begin{contractbox}[breakable]{\promptheading{W4}{BigCodeBench memory construction}}
\promptrole{Success output contract}\begin{lstlisting}[style=prompttext]
[MEMORY TYPE] SUCCESS_PROCEDURE
[TASK]
<short task abstraction>

[EXECUTION TRAJECTORY]
Contract: <only the public input/output or exception contract relevant to transfer>
Implementation: <concise API calls and data flow evidenced by the passing code>
Boundary: <one explicitly supported edge condition, or "No additional boundary condition established">
Verified idiom: <optional exact code fragment of at most 16 lines; include its return when applicable>
\end{lstlisting}
\promptrole{Failure output contract}\begin{lstlisting}[style=prompttext]
[MEMORY TYPE] FAILURE_REFLECTION
[TASK]
<short task abstraction>

[REFLECTION]
<one evidence-supported bug class, why it violates the public contract, and one reusable correction>
\end{lstlisting}
\promptrole{Construction constraints}
Success claims are grounded in the public task and visible passing code. The writer preserves supported contracts and implementation evidence, and excludes invented edge cases, API requirements, and references to tests or the evaluator. A failure reflection identifies one supported public-contract mismatch and a reusable correction. The current failed trajectory supplies the primary evidence, with related references included when available. Task-required dependencies remain part of the contract.
\end{contractbox}
The construction constraints accompanying W1, W3, and W4 summarize the corresponding system or builder requirements. The BigCodeBench contracts describe the stored record, while the implementation-only view passed to the executor is shown in Appendix~\ref{app:prompt_actors}.

\subsection{Revision and Admission}
\label{app:prompt_revision}
The revision loop in Appendix~\ref{app:writing} passes the latest candidate, its realized trajectory, and its validation feedback to W5. Table~\ref{tab:memory_schemas} gives the environment-specific output contracts appended to that wrapper.

\begin{promptbox}[breakable]{\promptheading{W5}{Unified revision}}
\promptrole{Common prompt wrapper}\begin{lstlisting}[style=prompttext]
/no_think
Revise a procedural memory that did not provide enough uplift on a {env_name} task.

## Context
Task: {task_description}
{optional_metadata}

## Current memory (latest validated candidate)
{current_memory}

## Agent's actual trajectory with this memory
{trajectory_text}

## Outcome
- Success: {success}
- Total reward: {total_reward}
- Steps: {total_steps}/{max_steps}
- Utility with memory: {current_utility}
- Utility without memory: {baseline_utility}
- Uplift: {current_uplift} (target > {tau_accept})

## Outcome diagnosis
{failure_signals_rendered}

## Optional related memories
{reference_memories_or_empty}

## Task
Treat the current memory as the latest validated candidate. Contrast its intended procedure with the actual trajectory and outcome. Preserve only steps that the observable evidence supports, identify one concrete decision that limited the result, and revise that decision without reintroducing previously observed failures. Do not claim that an unobserved behavior worked.

Briefly identify the observed gap in the current guidance, then rewrite it as a short transferable rule.
\end{lstlisting}
\end{promptbox}
\begin{table}[ht]
\centering\small
\caption{Environment-specific output contracts appended to the revision wrapper.}
\label{tab:memory_schemas}
\begin{tabular}{@{}>{\raggedright\arraybackslash}p{0.26\linewidth}>{\raggedright\arraybackslash}p{0.66\linewidth}@{}}
\toprule
Environment & Revision output \\
\midrule
ALFWorld & \texttt{WHY IT FAILED}, followed by a procedure of three to seven abstract steps. \\[3pt]
WebShop & \texttt{WHY IT FAILED}, followed by one \texttt{Type / Trigger / Action} rule. \\[3pt]
BigCodeBench, passing candidate & \texttt{SUCCESS\_PROCEDURE} with \texttt{Contract}, \texttt{Implementation}, and \texttt{Boundary}. \\[3pt]
BigCodeBench, failing candidate & \texttt{FAILURE\_REFLECTION} identifying one unresolved bug class. \\
\bottomrule
\end{tabular}
\end{table}

\section{Experimental Protocols}
\label{app:experimental_protocols}

This section specifies the protocols behind the complete-workflow and controlled-retrieval evaluations in Section~\ref{sec:experiments}.

\subsection{Benchmarks and Data Partitions}
\label{app:benchmarks}

Table~\ref{tab:splits} gives the partitions used for the agent benchmarks in Section~\ref{sec:rq1}. Seed and runtime counts refer to execution samples or trajectories that provide training experience and memory-construction data. Evaluation tasks and their execution feedback are excluded from these updates.
\begin{table}[t]
\centering\small
\caption{Data partitions for the agent benchmarks. Seed and runtime counts denote execution samples or trajectories.}
\label{tab:splits}
\setlength{\tabcolsep}{5pt}
\begin{tabular}{@{}lrrl@{}}
\toprule
Benchmark & Seed & Runtime & Evaluation \\
\midrule
ALFWorld & 500 & 1{,}200 & 140 valid-seen / 134 valid-unseen \\
WebShop & 500 & 2{,}300 & 500 official test goals \\
BigCodeBench & --- & 798 & 342 held-out manifest tasks \\
\bottomrule
\end{tabular}
\end{table}

\revised{For each benchmark, 50 tasks are held out from the training partition as a fixed development set.}
\revised{Development and test tasks are excluded from scorer training, memory writing, and probe acquisition. Hyperparameters and checkpoints are selected on the development set before the configuration is fixed for final test evaluation.}

We partition MemSyco-Bench into 1,085 training and 465 test instances across five scenarios using a task-stratified 70:30 split with random seed 42. Instances sharing the same canonical dialogue are assigned to the same partition to prevent dialogue overlap between training and test sets (Table~\ref{tab:memsyco_split}).

\begin{table}[t]
\centering\small
\caption{Dialogue-grouped MemSyco-Bench partition with seed 42. Scenario names follow the benchmark.}
\label{tab:memsyco_split}
\begin{tabular}{@{}lrrr@{}}
\toprule
Scenario & Total & Train & Test \\
\midrule
Personalized Memory Use & 300 & 210 & 90 \\
Valid Memory Selection & 350 & 245 & 105 \\
Memory Evidence Conflict & 300 & 210 & 90 \\
Contextual Scope Control & 300 & 210 & 90 \\
Objective Fact Judgment & 300 & 210 & 90 \\
\midrule
Total & 1{,}550 & 1{,}085 & 465 \\
\bottomrule
\end{tabular}
\end{table}

\subsection{Memory Sources and Evaluation Protocols}
\label{app:protocol}
Section~\ref{sec:rq1} evaluates complete memory workflows. 
UpliftMem constructs its initial memory pool during bootstrap initialization and extends it through iterative memory writing, while comparison systems use their own construction and retrieval procedures. ReAct is the reference without external memory. These end-to-end results include the effects of both memory content and retrieval.

Section~\ref{sec:rq2} isolates retrieval by comparing native and UpliftMem retrieval within the same post-extraction store, executor, and evaluator. MemSyco-Bench supplies a prior dialogue and query. Its five memory sources are NaiveRAG, Mem0, A-MEM, MemGPT, and LightMem, with NaiveRAG abbreviated RAG. LightMem follows the benchmark's integrated implementation.

BigCodeBench uses a task-specific retrieval view. Its stored records distinguish successful implementation procedures from diagnostic failure reflections, while similarity recall, valuation, probing, and generation use the successful implementation-evidence projection. Injected evidence excludes diagnostic-only reflections, source-task contracts, and source-specific literals. Appendices~\ref{app:prompt_actors} and~\ref{app:prompt_writers} give the injected view and stored record contracts.

\subsection{Leave-One-Task-Type-Out Evaluation}
\label{app:loto_protocol}
\revised{LOTO evaluates generalization across the five MemSyco task types using the dialogue-grouped train/test split in Table~\ref{tab:memsyco_split}. For each held-out type, a new scorer is initialized from Qwen3-Reranker-0.6B and trained using only the other four types. The held-out type supplies no scorer supervision during warm-up, rollout collection, probing, or adaptation. Each fold is evaluated on its held-out test instances with the corresponding frozen memory store, independently for RAG, A-MEM, Mem0, MemGPT, and LightMem. The all-task-trained scorer supplies the in-domain reference on matched test instances, stores, and retrieval budgets. Results summarize three independent training runs.}

\revised{Every LOTO scorer is trained with a maximum retrieval budget of $K=5$, and its resulting checkpoints are evaluated at $K\in\{3,5,10\}$. The task columns in Table~\ref{tab:loto_4b} report the store-averaged difference $\Delta=\mathrm{LOTO}-\mathrm{Full}$ in the corresponding task pass rate, measured in percentage points. Avg.\ $|\Delta|$ averages the absolute store-level mean differences over the five held-out types and five memory stores.}

\subsection{Utility Configuration and Rollout Accounting}
\label{app:utility}
The utility in Eq.~\ref{eq:task_utility} uses a fixed performance measure $R(\tau)$, cost measure $C(\tau)$, cap $C_{\max}$, efficiency weight $\gamma$, and threshold $\rho$ within each benchmark protocol. Cost counts actions in ALFWorld and WebShop and injected implementation-evidence tokens in BigCodeBench. The latter excludes memory headings, scope metadata, the task prompt, and generated code. This trajectory-level cost is distinct from the number of executions used to collect training labels.

Training accounting includes no-memory reference executions, ordinary memory-assisted executions, acquired probes, and every executed writing-validation attempt. An execution reused for several purposes is counted once. Reference identities indicate which labels share baseline noise, as modeled in Appendix~\ref{app:observation}. The retained writing label follows Appendix~\ref{app:writing}, while cost accounting includes both retained and discarded validation attempts.

Probe-policy comparisons use a common candidate-pool and accounting convention. The number of newly acquired probes measures sample allocation, whereas total training cost also includes any required reference and validation executions. At evaluation, feedback used to score the returned trajectory does not trigger additional probe executions.

\subsection{Models and Implementation}
\label{app:config}
The frozen task executors are Qwen3-4B and Qwen3-8B \citep{yang2025qwen3}. Each has an independently adapted scorer initialized from Qwen3-Reranker-0.6B \citep{zhang2025qwen3embedding} and trained with LoRA \citep{hu2022lora}. Table~\ref{tab:training_config} collects the training, acquisition, and search settings. \revised{We adopt LoRA to support repeated scorer adaptation as new execution feedback becomes available during memory-pool expansion. Its low-rank parameterization enables controlled updates to the pretrained scorer while reducing the computational overhead of successive optimization steps.}

The task counter continues across three training epochs, triggering acquisition every $h=400$ processed tasks. Each window supplies at most $B_{\mathrm p}=300$ probe executions in Algorithm~\ref{alg:training_flow}.

\begin{table}[htbp]
\centering\small
\caption{Scorer training, probe acquisition, and retrieval settings.}
\label{tab:training_config}
\begin{tabular}{@{}p{0.62\linewidth}p{0.30\linewidth}@{}}
\toprule
Setting & Value \\
\midrule
LoRA rank / scaling / dropout & $8$ / $16$ / $0.1$ \\
Learning rate & $10^{-5}$ \\
Training epochs & $3$ \\
Processed tasks per acquisition event & $400$ \\
Maximum new probe executions per event & $300$ \\
\revised{Default probe acquisition cost $\lambda_{\mathrm{probe}}$} & \revised{$0.1$} \\
Recalled memories per task & $50$ \\
Test-time beam width & $2$ \\
Memory cap for ALFWorld / WebShop / BCB & $5$ / $5$ / $3$ \\
Main MemSyco memory cap & $5$ \\
Additional MemSyco memory caps & $3$, $10$ \\
Maximum writing revisions & $3$ \\
Writing acceptance and admission threshold & $0$ \\
\bottomrule
\end{tabular}
\end{table}

\revised{For Gaussian belief construction, we use 1024-dimensional scorer features with $\ell_2$ normalization and $\delta=10^{-8}$. The covariance uses a diagonal precision approximation with regularization $\lambda=1$ and up to $N_{\mathrm{cov}}=512$ recent executed task--memory-set pairs. Observation noise is estimated separately for each benchmark and executor from repeated memory-assisted and no-memory rollouts on $M=32$ task--memory-set pairs, with $R=3$ runs per condition. We use the pooled within-task variance estimates, apply the variance floor $\nu_{\min}^2=10^{-4}(1+\gamma)^2$, and combine the two execution conditions into an estimated scalar uplift-label noise variance $\widehat\sigma_n^2=\nu_+^2+\nu_0^2$. This estimate enters both the historical diagonal-precision fit and prospective EVSI observations. It is held fixed during subsequent training, while the feature-dependent covariance is refreshed after scorer updates.}

\subsection{Evaluation Metrics and Statistical Reporting}
\label{app:metrics}

ALFWorld reports task success rate and environment steps on valid-seen and valid-unseen tasks. WebShop reports mean reward, success rate, and interaction steps. BigCodeBench reports the proportion of problems whose generated solutions pass all applicable tests, both overall and on the hard subset. These metrics are reported separately from the training utility defined in Eq.~\ref{eq:task_utility}.

MemSyco-Bench covers five scenarios in two groups. Objective Fact, Scope Control, and Evidence Conflict evaluate when memory should influence a response using accuracy and sycophancy (Syco). Personalized Use and Valid Selection evaluate how memory is used through accuracy and correct-memory use (Correct), or accuracy and outdated-memory use (Outdated), respectively. Decision and Syco. summarize accuracy and sycophancy across the first three scenarios, while Select and Old correspond to Valid Selection accuracy and outdated-memory use.

Direction-aligned gains compare UpliftMem with native retrieval within each memory store. They use UpliftMem minus native scores for higher-is-better metrics and the reversed difference for sycophancy and outdated-memory use. Cross-store summaries weight the five sources equally. In Table~\ref{tab:app-k-sens}, each scenario first averages its two direction-aligned metric gains and the five stores, and Overall averages the five scenarios. This gain summary is distinct from the absolute metrics in Figure~\ref{fig:budget}.

ALFWorld, WebShop, and BigCodeBench are evaluated across independently seeded training runs. MemSyco-Bench uses five runs for the main comparisons and three for additional ablations. The probe-efficiency experiments likewise repeat probe acquisition and retriever adaptation across independent random seeds, starting from a shared warm-up checkpoint. The Native and Full results reused in the ablation tables retain their original five-run statistics. \revised{Reported standard deviations summarize variability across repeated runs.} MemSyco scores are summarized within each source before equal-weight averaging. \revised{Table~\ref{tab:app-main-k5} and Tables~\ref{tab:app-abl-k5-4b} and~\ref{tab:app-abl-k5-8b} report per-store means and standard deviations for the main comparisons and component ablations, respectively.}

\section{\texorpdfstring{\revised{Additional Experimental Results}}{Additional Experimental Results}}
\label{app:additional_results}

\subsection{\texorpdfstring{\revised{End-to-End Performance Variability}}{End-to-End Performance Variability}}
\label{app:agent_variability}

\revised{Table~\ref{tab:agent_std} reports the standard deviations of the end-to-end results in Table~\ref{tab:main_results} across three independent runs.}
\begin{table*}[t]
\centering
\footnotesize
\caption{Standard deviations across three runs for the end-to-end results in Table~\ref{tab:main_results}. SR values are in percentage points; other metrics retain their original units.}
\label{tab:agent_std}

\setlength{\tabcolsep}{2.2pt}
\renewcommand{\arraystretch}{1.10}

\begin{adjustbox}{max width=\textwidth}
\begin{tabular}{
@{}l
cccccccc
@{\hspace{5pt}}
cccccc
@{\hspace{5pt}}
cccc@{}
}
\toprule

& \multicolumn{8}{c}{\textbf{ALFWorld}}
& \multicolumn{6}{c}{\textbf{WebShop}}
& \multicolumn{4}{c}{\textbf{BCB}} \\

\cmidrule(lr){2-9}
\cmidrule(lr){10-15}
\cmidrule(lr){16-19}

& \multicolumn{2}{c}{Seen SR}
& \multicolumn{2}{c}{Seen Steps}
& \multicolumn{2}{c}{Unseen SR}
& \multicolumn{2}{c}{Unseen Steps}
& \multicolumn{2}{c}{Reward}
& \multicolumn{2}{c}{SR}
& \multicolumn{2}{c}{Steps}
& \multicolumn{2}{c}{SR}
& \multicolumn{2}{c}{Hard SR} \\

\cmidrule(lr){2-3}
\cmidrule(lr){4-5}
\cmidrule(lr){6-7}
\cmidrule(lr){8-9}
\cmidrule(lr){10-11}
\cmidrule(lr){12-13}
\cmidrule(lr){14-15}
\cmidrule(lr){16-17}
\cmidrule(lr){18-19}

\rowcolor{headergray}
\textbf{Method}
& \textbf{4B} & \textbf{8B}
& \textbf{4B} & \textbf{8B}
& \textbf{4B} & \textbf{8B}
& \textbf{4B} & \textbf{8B}
& \textbf{4B} & \textbf{8B}
& \textbf{4B} & \textbf{8B}
& \textbf{4B} & \textbf{8B}
& \textbf{4B} & \textbf{8B}
& \textbf{4B} & \textbf{8B} \\

\midrule

ReAct
& 1.90 & 2.20 & 0.72 & 0.61
& 2.30 & 2.80 & 0.89 & 0.76
& 1.60 & 1.40 & 1.50 & 1.80
& 0.34 & 0.29 & 1.70 & 2.20
& 4.50 & 4.90 \\

Random
& 3.40 & 2.80 & 1.20 & 1.08
& 4.10 & 3.20 & 1.35 & 1.20
& 3.10 & 2.80 & 2.80 & 3.00
& 0.46 & 0.41 & 2.60 & 3.10
& 6.40 & 5.80 \\

BM25-RAG
& 2.10 & 1.70 & 0.66 & 0.58
& 2.60 & 2.30 & 0.77 & 0.73
& 2.00 & 1.70 & 1.80 & 1.70
& 0.30 & 0.27 & 1.80 & 2.00
& 5.10 & 4.40 \\

Dense-RAG
& 2.60 & 2.40 & 0.91 & 0.78
& 3.40 & 3.00 & 1.08 & 0.86
& 2.40 & 2.30 & 2.30 & 2.20
& 0.38 & 0.40 & 2.40 & 2.60
& 4.90 & 5.80 \\

MemP
& 3.20 & 3.60 & 1.09 & 1.18
& 3.70 & 4.20 & 1.31 & 1.43
& 2.70 & 2.90 & 2.80 & 3.20
& 0.44 & 0.52 & 2.70 & 2.90
& 6.00 & 6.60 \\

MemRL
& 3.80 & 3.10 & 1.24 & 0.95
& 3.90 & 3.80 & 1.42 & 1.34
& 3.20 & 2.80 & 3.10 & 2.80
& 0.51 & 0.54 & 3.30 & 3.40
& 7.10 & 6.70 \\

Mem0
& 2.10 & 2.70 & 0.69 & 0.82
& 3.00 & 2.50 & 0.95 & 0.81
& 2.50 & 2.80 & 2.60 & 2.30
& 0.32 & 0.25 & 2.20 & 2.70
& 5.10 & 5.30 \\

MemGPT
& 3.10 & 2.60 & 1.05 & 0.92
& 3.60 & 3.40 & 1.20 & 1.11
& 3.20 & 2.40 & 2.70 & 2.60
& 0.47 & 0.36 & 2.90 & 2.50
& 5.80 & 6.10 \\

A-MEM
& 3.00 & 3.40 & 1.11 & 1.20
& 4.10 & 3.60 & 1.38 & 1.27
& 2.70 & 2.60 & 3.20 & 2.80
& 0.45 & 0.48 & 3.10 & 3.00
& 6.50 & 6.70 \\

LightMem
& 2.30 & 2.60 & 0.75 & 0.83
& 2.70 & 2.90 & 0.94 & 0.90
& 2.00 & 2.10 & 2.30 & 2.50
& 0.34 & 0.37 & 2.60 & 2.30
& 5.40 & 5.70 \\

ReasoningBank
& 3.20 & 3.00 & 0.96 & 0.89
& 3.50 & 3.10 & 1.13 & 0.99
& 2.90 & 2.70 & 2.70 & 3.10
& 0.42 & 0.39 & 2.90 & 3.40
& 5.80 & 6.30 \\

\midrule
\rowcolor{oursblue}
\textbf{UpliftMem}
& 2.70 & 2.40 & 0.81 & 0.64
& 3.00 & 3.30 & 0.94 & 1.07
& 2.30 & 2.50 & 2.10 & 2.50
& 0.36 & 0.41 & 2.60 & 2.40
& 5.80 & 5.20 \\

\bottomrule
\end{tabular}
\end{adjustbox}
\end{table*}

\subsection{\texorpdfstring{\revised{Training Stability under Memory Growth}}{Training Stability under Memory Growth}}
\label{app:training_stability}

\revised{UpliftMem accommodates an expanding memory pool through shared text-based scoring and progressive retriever adaptation. Newly admitted memories are evaluated by the same set-value scorer without introducing memory-specific trainable parameters, while similarity-based candidate recall and a fixed retrieval budget bound the size of each retrieval decision. As the memory pool evolves, previously collected execution feedback continues to supervise set-value learning, and recent retrieval trajectories provide policy supervision under the current scorer. This training procedure enables the retriever to incorporate newly available memories while preserving knowledge acquired from earlier executions.}

\begin{figure*}[t]
    \centering
    \includegraphics[width=0.98\textwidth]{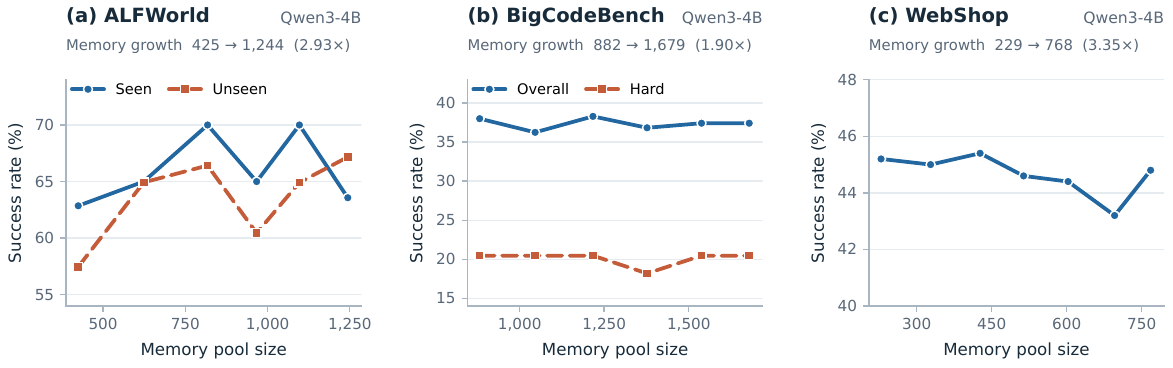}
    \caption{Training stability across expanding memory pools on three benchmarks with Qwen3-4B.}
    \label{fig:memory_growth_stability}
\end{figure*}

\revised{Figure~\ref{fig:memory_growth_stability} examines how retrieval performance evolves alongside memory-pool growth. On ALFWorld, the pool expands from 425 to 1,244 memories, accompanied by an increase in unseen-task success rate from 57.46\% to 67.16\%. On BigCodeBench, overall success remains between 36.26\% and 38.30\% as the pool nearly doubles. Similarly, WebShop maintains success rates between 43.2\% and 45.4\% as its pool grows from 229 to 768 memories. Across all three benchmarks, success rates remain within stable ranges throughout substantial memory expansion. Together with the multi-seed results in Table~\ref{tab:agent_std}, these findings demonstrate stable adaptation under dynamically expanding memory pools.}

\subsection{Additional MemSyco Results}
\label{app:memsyco_extra}

\revised{Table~\ref{tab:app-main-k5} reports absolute performance under fixed memory stores, while Table~\ref{tab:app-k-sens} summarizes retrieval-budget sensitivity. Tables~\ref{tab:app-abl-k5-4b} and~\ref{tab:app-abl-k5-8b} expand the fixed-store retrieval and learning-objective ablations with per-scenario means and standard deviations.}

\subsubsection{Retrieval Performance and Budget Sensitivity}

\begin{table*}
\centering\small
\caption{MemSyco-Bench retrieval performance by memory store at $K=5$.}
\label{tab:app-main-k5}
\setlength{\tabcolsep}{3pt}
\resizebox{\textwidth}{!}{%
\begin{tabular}{llcccccccccc}
\toprule
\multirow{3}{*}{Memory System} & \multirow{3}{*}{Variant} &
\multicolumn{6}{c}{\textbf{When to Use Memory}} & \multicolumn{4}{c}{\textbf{How to Use Memory}} \\
\cmidrule(lr){3-8}\cmidrule(lr){9-12}
 & & \multicolumn{2}{c}{Objective Fact} & \multicolumn{2}{c}{Scope Control} & \multicolumn{2}{c}{Evidence Conflict}
   & \multicolumn{2}{c}{Personalized Use} & \multicolumn{2}{c}{Valid Selection} \\
\cmidrule(lr){3-4}\cmidrule(lr){5-6}\cmidrule(lr){7-8}\cmidrule(lr){9-10}\cmidrule(lr){11-12}
 & & Acc.\up & Syco.\dn & Acc.\up & Syco.\dn & Acc.\up & Syco.\dn & Acc.\up & Cor.Mem.\up & Acc.\up & Outd.Mem.\dn \\
\midrule
\multicolumn{12}{c}{\textit{Qwen3-4B}} \\
\midrule
\multirow{2}{*}{RAG} & Base
 & 35.3\sd{3.5} & 74.4\sd{4.4} & 24.4\sd{2.6} & 32.2\sd{1.1} & 57.0\sd{1.7} & 43.0\sd{1.7} & 85.1\sd{1.7} & 90.7\sd{1.7} & 53.7\sd{2.3} & 47.0\sd{2.5} \\
 & UpliftMem
 & \textbf{50.9}\sd{2.4} & \textbf{49.1}\sd{1.6} & \textbf{87.1}\sd{1.3} & \textbf{8.7}\sd{2.0} & \textbf{84.9}\sd{0.6} & \textbf{14.9}\sd{1.0} & \textbf{85.9}\sd{2.8} & \textbf{91.1}\sd{1.1} & \textbf{63.8}\sd{2.1} & \textbf{37.1}\sd{2.2} \\
\cmidrule(lr){1-12}
\multirow{2}{*}{A-MEM} & Base
 & 36.7\sd{1.1} & 66.3\sd{0.6} & 48.4\sd{1.3} & 34.0\sd{2.7} & 46.2\sd{2.2} & 53.8\sd{2.2} & 82.2\sd{3.8} & 92.2\sd{2.2} & 55.2\sd{1.9} & 47.0\sd{3.1} \\
 & UpliftMem
 & \textbf{46.9}\sd{0.9} & \textbf{52.0}\sd{2.3} & \textbf{92.7}\sd{2.0} & \textbf{4.7}\sd{1.8} & \textbf{84.0}\sd{2.8} & \textbf{15.6}\sd{2.9} & \textbf{84.0}\sd{2.9} & \textbf{92.7}\sd{1.1} & \textbf{66.1}\sd{1.4} & \textbf{33.7}\sd{1.1} \\
\cmidrule(lr){1-12}
\multirow{2}{*}{Mem0} & Base
 & 52.2\sd{2.6} & 54.0\sd{2.4} & 31.5\sd{1.7} & 28.9\sd{6.8} & 49.6\sd{3.6} & 44.4\sd{2.1} & 58.1\sd{1.3} & \textbf{72.7}\sd{2.2} & 46.7\sd{1.6} & 53.7\sd{1.5} \\
 & UpliftMem
 & \textbf{57.8}\sd{2.4} & \textbf{40.7}\sd{2.8} & \textbf{66.3}\sd{4.5} & \textbf{18.7}\sd{2.1} & \textbf{62.2}\sd{4.8} & \textbf{33.0}\sd{5.5} & \textbf{59.0}\sd{0.6} & 72.4\sd{2.3} & \textbf{55.2}\sd{3.1} & \textbf{46.0}\sd{3.2} \\
\cmidrule(lr){1-12}
\multirow{2}{*}{MemGPT} & Base
 & 34.4\sd{2.4} & 77.4\sd{4.2} & 61.1\sd{4.0} & 23.7\sd{1.3} & 43.1\sd{2.9} & 55.8\sd{2.1} & 81.9\sd{2.6} & 91.1\sd{2.2} & 61.0\sd{2.5} & 39.4\sd{2.9} \\
 & UpliftMem
 & \textbf{43.1}\sd{2.9} & \textbf{60.7}\sd{3.6} & \textbf{84.8}\sd{0.6} & \textbf{10.4}\sd{2.6} & \textbf{69.3}\sd{2.2} & \textbf{27.8}\sd{1.9} & \textbf{82.9}\sd{1.3} & \textbf{92.2}\sd{1.1} & \textbf{65.9}\sd{3.0} & \textbf{35.6}\sd{3.2} \\
\cmidrule(lr){1-12}
\multirow{2}{*}{LightMem} & Base
 & 52.6\sd{1.3} & 57.8\sd{1.9} & 10.4\sd{3.2} & 44.1\sd{2.3} & 42.0\sd{2.1} & 45.8\sd{3.1} & 70.0\sd{2.7} & 80.2\sd{1.8} & 54.9\sd{1.1} & 47.9\sd{1.1} \\
 & UpliftMem
 & \textbf{56.2}\sd{1.1} & \textbf{46.1}\sd{1.4} & \textbf{56.4}\sd{1.6} & \textbf{24.0}\sd{3.3} & \textbf{55.1}\sd{1.9} & \textbf{35.8}\sd{2.0} & \textbf{71.3}\sd{1.8} & \textbf{81.1}\sd{2.1} & \textbf{56.8}\sd{1.9} & \textbf{44.3}\sd{1.5} \\
\midrule
\multicolumn{12}{c}{\textit{Qwen3-8B}} \\
\midrule
\multirow{2}{*}{RAG} & Base
 & 42.2\sd{2.2} & 52.0\sd{1.6} & 53.7\sd{2.3} & 3.7\sd{1.7} & 39.6\sd{0.6} & 60.4\sd{0.6} & \textbf{84.0}\sd{1.3} & \textbf{82.0}\sd{1.4} & 57.8\sd{1.1} & 43.2\sd{0.5} \\
 & UpliftMem
 & \textbf{51.8}\sd{1.7} & \textbf{41.2}\sd{1.0} & \textbf{82.4}\sd{1.1} & \textbf{0.4}\sd{1.1} & \textbf{94.2}\sd{0.9} & \textbf{5.8}\sd{0.9} & 82.6\sd{0.6} & 78.9\sd{1.1} & \textbf{73.9}\sd{0.9} & \textbf{26.1}\sd{1.7} \\
\cmidrule(lr){1-12}
\multirow{2}{*}{A-MEM} & Base
 & 41.1\sd{4.4} & 59.6\sd{2.3} & 71.6\sd{2.2} & 0.9\sd{1.2} & 38.5\sd{1.3} & 61.5\sd{1.3} & 82.7\sd{1.0} & 82.2\sd{1.6} & 56.2\sd{4.9} & 47.0\sd{4.0} \\
 & UpliftMem
 & \textbf{62.4}\sd{3.2} & \textbf{28.0}\sd{3.1} & \textbf{85.6}\sd{1.9} & \textbf{0.4}\sd{0.6} & \textbf{82.9}\sd{1.0} & \textbf{17.1}\sd{1.0} & \textbf{84.1}\sd{1.9} & \textbf{82.6}\sd{2.6} & \textbf{78.1}\sd{3.2} & \textbf{23.8}\sd{2.8} \\
\cmidrule(lr){1-12}
\multirow{2}{*}{Mem0} & Base
 & 50.7\sd{1.7} & 48.5\sd{2.8} & 30.7\sd{2.0} & 8.1\sd{1.7} & 25.3\sd{2.1} & 60.0\sd{2.2} & 62.4\sd{3.4} & 59.8\sd{2.4} & 61.3\sd{2.2} & \textbf{39.0}\sd{3.4} \\
 & UpliftMem
 & \textbf{59.8}\sd{2.1} & \textbf{29.1}\sd{2.7} & \textbf{65.2}\sd{4.6} & \textbf{4.6}\sd{3.6} & \textbf{59.3}\sd{1.3} & \textbf{9.3}\sd{2.2} & \textbf{64.8}\sd{2.3} & \textbf{61.7}\sd{1.7} & \textbf{64.8}\sd{1.0} & 39.9\sd{1.2} \\
\cmidrule(lr){1-12}
\multirow{2}{*}{MemGPT} & Base
 & 37.4\sd{2.3} & 61.9\sd{3.2} & 93.6\sd{1.8} & 1.5\sd{1.7} & 7.8\sd{1.9} & 92.2\sd{1.9} & 81.8\sd{2.0} & 80.0\sd{1.6} & 71.2\sd{1.8} & 31.8\sd{2.7} \\
 & UpliftMem
 & \textbf{51.6}\sd{3.0} & \textbf{44.7}\sd{1.2} & \textbf{98.8}\sd{1.3} & \textbf{0.0}\sd{0.0} & \textbf{36.9}\sd{1.2} & \textbf{18.4}\sd{1.3} & \textbf{86.4}\sd{2.6} & \textbf{85.3}\sd{4.5} & \textbf{71.6}\sd{2.5} & \textbf{31.6}\sd{2.9} \\
\cmidrule(lr){1-12}
\multirow{2}{*}{LightMem} & Base
 & 55.2\sd{0.6} & 43.0\sd{1.3} & 31.1\sd{1.9} & 7.8\sd{1.1} & 47.4\sd{1.7} & 39.6\sd{1.7} & 74.1\sd{1.7} & 71.1\sd{1.1} & 58.9\sd{2.9} & 43.2\sd{3.1} \\
 & UpliftMem
 & \textbf{59.7}\sd{1.4} & \textbf{34.9}\sd{1.3} & \textbf{60.0}\sd{1.9} & \textbf{5.1}\sd{1.0} & \textbf{81.8}\sd{1.7} & \textbf{8.2}\sd{1.3} & \textbf{75.1}\sd{2.4} & \textbf{72.4}\sd{1.2} & \textbf{65.0}\sd{2.0} & \textbf{37.3}\sd{2.0} \\
\bottomrule
\end{tabular}}

\end{table*}

\begin{table*}[t]
\centering\small
\caption{Direction-aligned retrieval gains across inference budgets on MemSyco-Bench (pp).}
\label{tab:app-k-sens}
\setlength{\tabcolsep}{4pt}
\resizebox{\textwidth}{!}{%
\begin{tabular}{llcccccc}
\toprule
\multirow{2}{*}{Backbone} & \multirow{2}{*}{$K$} &
\multicolumn{3}{c}{When to Use Memory} & \multicolumn{2}{c}{How to Use Memory} & \multirow{2}{*}{Overall} \\
\cmidrule(lr){3-5}\cmidrule(lr){6-7}
 & & Obj.\ Fact & Scope Ctrl. & Evid.\ Conf. & Person.\ Use & Valid Sel. & \\
\midrule
 \multirow{3}{*}{Qwen3-4B} &  3 & 13.3 & \textbf{33.7} & 19.9 & \textbf{4.3} & 7.2 & \textbf{15.7} \\
  &  5 & 12.5 & 30.8 & 23.3 & 0.8 & 7.5 & 15.0 \\
  & 10 & \textbf{14.0} & 15.8 & \textbf{24.3} & 2.5 & \textbf{8.8} & 13.1 \\
\midrule
 \multirow{3}{*}{Qwen3-8B} &  3 & 13.6 & \textbf{12.4} & 25.0 & \textbf{2.8} & \textbf{9.7} & 12.7 \\
  &  5 & \textbf{14.6} & 12.3 & 45.1 & 1.4 & 9.4 & 16.5 \\
  & 10 & 14.6 & 7.1 & \textbf{53.7} & 1.7 & 8.4 & \textbf{17.1} \\
\midrule
 \multirow{3}{*}{Cross-model avg.} &  3 & 13.4 & \textbf{23.0} & 22.4 & \textbf{3.5} & 8.4 & 14.2 \\
  &  5 & 13.5 & 21.5 & 34.2 & 1.1 & 8.4 & \textbf{15.8} \\
  & 10 & \textbf{14.3} & 11.4 & \textbf{39.0} & 2.1 & \textbf{8.6} & 15.1 \\
\bottomrule
\end{tabular}}

\end{table*}

\subsubsection{Component Ablations}\label{app:memsyco_component_ablations}

\revised{Tables~\ref{tab:app-abl-k5-4b} and~\ref{tab:app-abl-k5-8b} report per-store ablations for Qwen3-4B and Qwen3-8B. Independent Scoring estimates singleton uplift without conditioning on the selected set. The value-objective variants remove regression (Reg.), sign classification (BCE), and pairwise ranking in Eq.~\ref{eq:main_value_losses}; iterative memory writing is evaluated in Table~\ref{tab:ablation}. Entries report means and sample standard deviations, using five independent runs for Full UpliftMem and three for the variants. Positive Avg.$\Delta$ denotes the mean direction-aligned advantage of UpliftMem across ten metrics, and bold identifies the best reported mean in each memory-store block.}

\begin{table*}[p]
\centering
\footnotesize
\caption{Per-store component ablations on MemSyco-Bench at $K=5$ (Qwen3-4B).}
\label{tab:app-abl-k5-4b}
\setlength{\tabcolsep}{1.5pt}
\renewcommand{\arraystretch}{1.04}
\resizebox{\textwidth}{!}{%
\begin{tabular}{@{}llccccccccccc@{}}
\toprule
\multirow{3}{*}{Memory System} & \multirow{3}{*}{Variant} &
\multicolumn{6}{c}{\textbf{When to Use Memory}} & \multicolumn{4}{c}{\textbf{How to Use Memory}} & \multirow{3}{*}{Avg.$\Delta$} \\
\cmidrule(lr){3-8}\cmidrule(lr){9-12}
 & & \multicolumn{2}{c}{Objective Fact} & \multicolumn{2}{c}{Scope Control} & \multicolumn{2}{c}{Evidence Conflict}
 & \multicolumn{2}{c}{Personalized Use} & \multicolumn{2}{c}{Valid Selection} & \\
\cmidrule(lr){3-4}\cmidrule(lr){5-6}\cmidrule(lr){7-8}\cmidrule(lr){9-10}\cmidrule(lr){11-12}
 & & Acc.\up & Syco.\dn & Acc.\up & Syco.\dn & Acc.\up & Syco.\dn & Acc.\up & Cor.Mem.\up & Acc.\up & Outd.Mem.\dn & \\
\midrule
\rowcolor{oursblue} \multirow{10}{*}{RAG} & \textbf{UpliftMem} & \textbf{50.9}\revised{\sd{2.4}} & \textbf{49.1}\revised{\sd{1.6}} & \textbf{87.1}\revised{\sd{1.3}} & 8.7\revised{\sd{2.0}} & \textbf{84.9}\revised{\sd{0.6}} & \textbf{14.9}\revised{\sd{1.0}} & \textbf{85.9}\revised{\sd{2.8}} & \textbf{91.1}\revised{\sd{1.1}} & \textbf{63.8}\revised{\sd{2.1}} & 37.1\revised{\sd{2.2}} & -- \\
 & \revised{Independent Scoring} & \revised{44.1\sd{2.3}} & \revised{59.3\sd{0.6}} & \revised{83.0\sd{0.6}} & \revised{10.4\sd{0.6}} & \revised{73.0\sd{5.5}} & \revised{27.0\sd{5.5}} & \revised{81.5\sd{1.3}} & \revised{88.1\sd{1.3}} & \revised{61.9\sd{3.4}} & \revised{38.1\sd{3.4}} & \revised{+5.7} \\
 & Base reranker & 41.5\revised{\sd{1.3}} & 65.2\revised{\sd{2.3}} & 12.6\revised{\sd{1.3}} & 48.1\revised{\sd{7.9}} & 65.6\revised{\sd{1.9}} & 34.4\revised{\sd{1.9}} & 69.6\revised{\sd{1.7}} & 79.3\revised{\sd{0.6}} & 54.3\revised{\sd{3.3}} & 47.0\revised{\sd{3.3}} & +22.6 \\
 & Warm-up only & 41.5\revised{\sd{2.8}} & 64.1\revised{\sd{0.6}} & 59.3\revised{\sd{4.6}} & 24.8\revised{\sd{2.8}} & 61.5\revised{\sd{2.3}} & 38.5\revised{\sd{2.3}} & 69.6\revised{\sd{2.3}} & 82.2\revised{\sd{1.1}} & 54.0\revised{\sd{4.3}} & 47.3\revised{\sd{5.7}} & +16.1 \\
 & Random probing & 43.7\revised{\sd{2.6}} & 58.1\revised{\sd{2.3}} & 77.8\revised{\sd{3.3}} & \textbf{8.1}\revised{\sd{2.9}} & 61.5\revised{\sd{1.7}} & 38.5\revised{\sd{1.7}} & 76.3\revised{\sd{1.3}} & 83.0\revised{\sd{1.7}} & 63.5\revised{\sd{2.2}} & \textbf{36.8}\revised{\sd{2.0}} & +9.0 \\
 & w/o policy loss & 46.7\revised{\sd{1.1}} & 57.0\revised{\sd{3.6}} & 73.3\revised{\sd{2.2}} & 17.0\revised{\sd{2.8}} & 70.4\revised{\sd{2.8}} & 29.6\revised{\sd{2.8}} & 77.8\revised{\sd{4.0}} & 84.4\revised{\sd{0.0}} & 59.4\revised{\sd{2.2}} & 41.6\revised{\sd{3.1}} & +8.7 \\
 & w/o value loss & 40.7\revised{\sd{3.2}} & 65.2\revised{\sd{2.3}} & 71.1\revised{\sd{2.9}} & 16.7\revised{\sd{4.0}} & 64.1\revised{\sd{2.3}} & 35.9\revised{\sd{2.3}} & 81.1\revised{\sd{2.2}} & 89.6\revised{\sd{2.3}} & 62.2\revised{\sd{2.0}} & 40.3\revised{\sd{1.5}} & +10.3 \\
 & \revised{w/o Reg.} & \revised{47.8\sd{2.2}} & \revised{57.4\sd{0.6}} & \revised{73.3\sd{2.2}} & \revised{15.2\sd{3.4}} & \revised{63.7\sd{4.2}} & \revised{36.3\sd{4.2}} & \revised{79.6\sd{1.7}} & \revised{87.4\sd{1.7}} & \revised{53.0\sd{2.4}} & \revised{47.0\sd{2.4}} & \revised{+10.5} \\
 & \revised{w/o BCE} & \revised{48.1\sd{0.6}} & \revised{55.2\sd{1.3}} & \revised{69.6\sd{0.6}} & \revised{17.4\sd{2.3}} & \revised{65.2\sd{2.8}} & \revised{34.8\sd{2.8}} & \revised{74.8\sd{2.8}} & \revised{83.7\sd{1.7}} & \revised{57.8\sd{4.3}} & \revised{42.2\sd{4.3}} & \revised{+10.4} \\
 & \revised{w/o Ranking} & \revised{44.4\sd{1.1}} & \revised{65.2\sd{2.8}} & \revised{73.3\sd{1.9}} & \revised{16.3\sd{2.8}} & \revised{68.5\sd{3.2}} & \revised{31.5\sd{3.2}} & \revised{77.8\sd{2.9}} & \revised{89.6\sd{1.7}} & \revised{60.0\sd{1.9}} & \revised{40.0\sd{1.9}} & \revised{+9.3} \\
\midrule
\rowcolor{oursblue} \multirow{10}{*}{A-MEM} & \textbf{UpliftMem} & \textbf{46.9}\revised{\sd{0.9}} & \textbf{52.0}\revised{\sd{2.3}} & \textbf{92.7}\revised{\sd{2.0}} & \textbf{4.7}\revised{\sd{1.8}} & \textbf{84.0}\revised{\sd{2.8}} & \textbf{15.6}\revised{\sd{2.9}} & 84.0\revised{\sd{2.9}} & 92.7\revised{\sd{1.1}} & \textbf{66.1}\revised{\sd{1.4}} & \textbf{33.7}\revised{\sd{1.1}} & -- \\
 & \revised{Independent Scoring} & \revised{42.6\sd{4.5}} & \revised{68.1\sd{6.4}} & \revised{85.9\sd{1.7}} & \revised{10.4\sd{0.6}} & \revised{44.4\sd{1.9}} & \revised{55.6\sd{1.9}} & \revised{\textbf{85.6}\sd{2.2}} & \revised{\textbf{93.7}\sd{2.3}} & \revised{61.9\sd{1.6}} & \revised{38.1\sd{1.6}} & \revised{+11.8} \\
 & Base reranker & 40.4\revised{\sd{0.6}} & 70.4\revised{\sd{4.2}} & 14.4\revised{\sd{1.9}} & 55.9\revised{\sd{1.3}} & 65.6\revised{\sd{1.9}} & 34.4\revised{\sd{1.9}} & 65.9\revised{\sd{2.3}} & 75.6\revised{\sd{1.1}} & 55.2\revised{\sd{1.0}} & 47.0\revised{\sd{1.5}} & +25.1 \\
 & Warm-up only & 41.9\revised{\sd{2.8}} & 63.7\revised{\sd{4.2}} & 59.3\revised{\sd{2.8}} & 29.3\revised{\sd{4.6}} & 50.0\revised{\sd{3.8}} & 50.0\revised{\sd{3.8}} & 77.0\revised{\sd{2.3}} & 83.7\revised{\sd{3.2}} & 55.2\revised{\sd{3.4}} & 47.3\revised{\sd{2.9}} & +18.4 \\
 & Random probing & 40.7\revised{\sd{0.6}} & 55.6\revised{\sd{2.2}} & 69.6\revised{\sd{2.3}} & 17.8\revised{\sd{2.2}} & 60.4\revised{\sd{4.2}} & 39.6\revised{\sd{4.2}} & 82.6\revised{\sd{4.5}} & 89.6\revised{\sd{1.7}} & 59.7\revised{\sd{3.6}} & 41.6\revised{\sd{2.4}} & +11.2 \\
 & w/o policy loss & 39.3\revised{\sd{1.7}} & 61.1\revised{\sd{2.2}} & 86.3\revised{\sd{0.6}} & 11.1\revised{\sd{1.1}} & 67.8\revised{\sd{4.0}} & 32.2\revised{\sd{4.0}} & 79.3\revised{\sd{0.6}} & 87.8\revised{\sd{1.1}} & 62.9\revised{\sd{2.5}} & 37.5\revised{\sd{2.4}} & +7.9 \\
 & w/o value loss & 39.3\revised{\sd{2.3}} & 67.4\revised{\sd{1.3}} & 68.5\revised{\sd{2.8}} & 23.7\revised{\sd{1.7}} & 58.9\revised{\sd{2.2}} & 41.1\revised{\sd{2.2}} & 83.1\revised{\sd{1.7}} & 90.4\revised{\sd{1.7}} & 63.2\revised{\sd{2.0}} & 38.4\revised{\sd{3.1}} & +12.8 \\
 & \revised{w/o Reg.} & \revised{39.6\sd{2.8}} & \revised{63.7\sd{0.6}} & \revised{76.3\sd{2.6}} & \revised{18.5\sd{2.3}} & \revised{62.6\sd{2.8}} & \revised{37.4\sd{2.8}} & \revised{82.6\sd{2.3}} & \revised{91.9\sd{0.6}} & \revised{56.2\sd{1.0}} & \revised{43.8\sd{1.0}} & \revised{+11.5} \\
 & \revised{w/o BCE} & \revised{38.5\sd{1.7}} & \revised{68.5\sd{3.6}} & \revised{76.3\sd{1.7}} & \revised{14.1\sd{2.3}} & \revised{63.0\sd{2.3}} & \revised{37.0\sd{2.3}} & \revised{73.3\sd{1.1}} & \revised{84.4\sd{1.9}} & \revised{52.1\sd{3.3}} & \revised{47.9\sd{3.3}} & \revised{+14.0} \\
 & \revised{w/o Ranking} & \revised{44.8\sd{1.7}} & \revised{63.3\sd{4.4}} & \revised{83.3\sd{2.2}} & \revised{12.2\sd{0.0}} & \revised{58.5\sd{2.8}} & \revised{41.5\sd{2.8}} & \revised{78.9\sd{0.0}} & \revised{91.1\sd{2.9}} & \revised{56.5\sd{3.6}} & \revised{43.5\sd{3.6}} & \revised{+10.8} \\
\midrule
\rowcolor{oursblue} \multirow{10}{*}{Mem0} & \textbf{UpliftMem} & \textbf{57.8}\revised{\sd{2.4}} & \textbf{40.7}\revised{\sd{2.8}} & \textbf{66.3}\revised{\sd{4.5}} & 18.7\revised{\sd{2.1}} & 62.2\revised{\sd{4.8}} & 33.0\revised{\sd{5.5}} & 59.0\revised{\sd{0.6}} & 72.4\revised{\sd{2.3}} & \textbf{55.2}\revised{\sd{3.1}} & \textbf{46.0}\revised{\sd{3.2}} & -- \\
 & \revised{Independent Scoring} & \revised{52.2\sd{2.2}} & \revised{53.0\sd{1.7}} & \revised{51.1\sd{2.2}} & \revised{24.8\sd{4.2}} & \revised{\textbf{68.9}\sd{1.9}} & \revised{\textbf{28.5}\sd{1.7}} & \revised{\textbf{67.0}\sd{1.3}} & \revised{\textbf{75.2}\sd{0.6}} & \revised{47.9\sd{1.5}} & \revised{51.7\sd{2.0}} & \revised{+3.0} \\
 & Base reranker & 56.8\revised{\sd{2.9}} & 48.5\revised{\sd{1.7}} & 41.9\revised{\sd{1.3}} & 25.2\revised{\sd{7.1}} & 52.6\revised{\sd{1.3}} & 38.1\revised{\sd{1.7}} & 53.7\revised{\sd{1.7}} & 66.3\revised{\sd{2.3}} & 51.4\revised{\sd{1.6}} & 48.9\revised{\sd{1.5}} & +7.3 \\
 & Warm-up only & 53.3\revised{\sd{1.9}} & 48.1\revised{\sd{2.8}} & 45.9\revised{\sd{2.3}} & 23.7\revised{\sd{4.5}} & 50.0\revised{\sd{4.0}} & 40.0\revised{\sd{3.8}} & 57.4\revised{\sd{0.6}} & 65.6\revised{\sd{4.8}} & 48.3\revised{\sd{3.3}} & 52.4\revised{\sd{3.4}} & +7.8 \\
 & Random probing & 52.2\revised{\sd{2.9}} & 50.7\revised{\sd{2.3}} & 62.3\revised{\sd{2.6}} & 19.3\revised{\sd{2.3}} & 54.8\revised{\sd{1.7}} & 37.4\revised{\sd{0.6}} & 51.9\revised{\sd{1.7}} & 63.0\revised{\sd{1.3}} & 53.0\revised{\sd{1.1}} & 47.7\revised{\sd{1.0}} & +5.2 \\
 & w/o policy loss & 51.1\revised{\sd{1.1}} & 50.7\revised{\sd{1.7}} & 54.8\revised{\sd{2.6}} & 29.6\revised{\sd{4.5}} & 53.3\revised{\sd{4.0}} & 39.6\revised{\sd{3.9}} & 50.7\revised{\sd{1.7}} & 58.9\revised{\sd{1.1}} & 54.0\revised{\sd{2.2}} & 46.3\revised{\sd{2.0}} & +7.8 \\
 & w/o value loss & 55.9\revised{\sd{2.3}} & 48.9\revised{\sd{0.0}} & 55.6\revised{\sd{2.9}} & 24.8\revised{\sd{3.9}} & 61.1\revised{\sd{0.0}} & 35.9\revised{\sd{0.6}} & 53.7\revised{\sd{1.7}} & 64.8\revised{\sd{1.7}} & 50.5\revised{\sd{1.0}} & 50.2\revised{\sd{0.5}} & +5.3 \\
 & \revised{w/o Reg.} & \revised{54.8\sd{1.7}} & \revised{53.7\sd{3.6}} & \revised{57.4\sd{4.6}} & \revised{21.5\sd{4.5}} & \revised{62.2\sd{1.1}} & \revised{31.9\sd{0.6}} & \revised{53.0\sd{3.6}} & \revised{60.7\sd{1.3}} & \revised{50.5\sd{3.4}} & \revised{48.9\sd{3.1}} & \revised{+5.2} \\
 & \revised{w/o BCE} & \revised{53.7\sd{0.6}} & \revised{52.2\sd{2.9}} & \revised{62.2\sd{1.9}} & \revised{18.5\sd{0.6}} & \revised{63.7\sd{2.8}} & \revised{33.3\sd{1.1}} & \revised{55.2\sd{1.7}} & \revised{63.0\sd{0.6}} & \revised{50.5\sd{3.8}} & \revised{49.2\sd{4.3}} & \revised{+3.9} \\
 & \revised{w/o Ranking} & \revised{51.9\sd{0.6}} & \revised{55.6\sd{0.0}} & \revised{\textbf{66.3}\sd{3.4}} & \revised{\textbf{18.1}\sd{3.2}} & \revised{58.9\sd{1.1}} & \revised{33.7\sd{1.3}} & \revised{47.0\sd{0.6}} & \revised{61.5\sd{1.7}} & \revised{53.0\sd{1.5}} & \revised{46.7\sd{1.9}} & \revised{+5.0} \\
\midrule
\rowcolor{oursblue} \multirow{10}{*}{MemGPT} & \textbf{UpliftMem} & \textbf{43.1}\revised{\sd{2.9}} & \textbf{60.7}\revised{\sd{3.6}} & 84.8\revised{\sd{0.6}} & 10.4\revised{\sd{2.6}} & \textbf{69.3}\revised{\sd{2.2}} & \textbf{27.8}\revised{\sd{1.9}} & 82.9\revised{\sd{1.3}} & 92.2\revised{\sd{1.1}} & \textbf{65.9}\revised{\sd{3.0}} & \textbf{35.6}\revised{\sd{3.2}} & -- \\
 & \revised{Independent Scoring} & \revised{33.0\sd{0.6}} & \revised{77.0\sd{1.3}} & \revised{\textbf{91.1}\sd{1.1}} & \revised{\textbf{8.5}\sd{1.7}} & \revised{41.9\sd{2.3}} & \revised{57.0\sd{2.3}} & \revised{84.4\sd{2.9}} & \revised{88.9\sd{2.2}} & \revised{62.9\sd{3.3}} & \revised{37.1\sd{3.3}} & \revised{+8.1} \\
 & Base reranker & 30.0\revised{\sd{4.8}} & 73.7\revised{\sd{1.7}} & 73.0\revised{\sd{3.6}} & 15.9\revised{\sd{1.7}} & 40.0\revised{\sd{4.0}} & 60.0\revised{\sd{4.0}} & 78.5\revised{\sd{1.3}} & 89.6\revised{\sd{1.7}} & 63.5\revised{\sd{3.1}} & 39.0\revised{\sd{3.4}} & +11.8 \\
 & Warm-up only & 32.2\revised{\sd{3.3}} & 75.6\revised{\sd{2.2}} & 75.9\revised{\sd{3.4}} & 14.1\revised{\sd{2.8}} & 51.9\revised{\sd{4.2}} & 47.4\revised{\sd{3.6}} & 81.5\revised{\sd{2.3}} & 91.9\revised{\sd{2.6}} & 60.3\revised{\sd{2.7}} & 41.6\revised{\sd{2.4}} & +8.9 \\
 & Random probing & 30.4\revised{\sd{1.7}} & 72.2\revised{\sd{1.1}} & 83.3\revised{\sd{2.2}} & 11.9\revised{\sd{1.7}} & 47.0\revised{\sd{3.6}} & 51.9\revised{\sd{3.6}} & 80.3\revised{\sd{1.1}} & 91.3\revised{\sd{2.2}} & 59.7\revised{\sd{2.0}} & 41.6\revised{\sd{2.4}} & +8.9 \\
 & w/o policy loss & 30.0\revised{\sd{2.2}} & 70.4\revised{\sd{1.3}} & 81.5\revised{\sd{1.3}} & 11.1\revised{\sd{1.9}} & 52.2\revised{\sd{3.3}} & 47.0\revised{\sd{2.8}} & 80.7\revised{\sd{1.7}} & 91.5\revised{\sd{1.3}} & 61.0\revised{\sd{1.0}} & 40.6\revised{\sd{1.1}} & +7.6 \\
 & w/o value loss & 31.9\revised{\sd{1.7}} & 73.3\revised{\sd{4.0}} & 80.7\revised{\sd{4.2}} & 12.6\revised{\sd{2.3}} & 56.3\revised{\sd{3.4}} & 43.7\revised{\sd{3.4}} & 82.2\revised{\sd{2.9}} & 92.2\revised{\sd{1.1}} & 58.7\revised{\sd{1.5}} & 42.9\revised{\sd{2.5}} & +7.4 \\
 & \revised{w/o Reg.} & \revised{28.9\sd{2.9}} & \revised{75.9\sd{2.8}} & \revised{79.3\sd{1.7}} & \revised{15.9\sd{0.6}} & \revised{50.0\sd{3.8}} & \revised{49.3\sd{3.6}} & \revised{83.0\sd{1.7}} & \revised{90.4\sd{1.3}} & \revised{57.5\sd{2.0}} & \revised{42.5\sd{2.0}} & \revised{+9.8} \\
 & \revised{w/o BCE} & \revised{33.3\sd{2.2}} & \revised{76.7\sd{1.1}} & \revised{78.9\sd{2.2}} & \revised{12.6\sd{2.3}} & \revised{52.6\sd{2.6}} & \revised{45.9\sd{2.3}} & \revised{81.9\sd{3.6}} & \revised{90.7\sd{1.7}} & \revised{62.2\sd{1.1}} & \revised{37.8\sd{1.1}} & \revised{+7.7} \\
 & \revised{w/o Ranking} & \revised{31.1\sd{1.9}} & \revised{76.7\sd{4.0}} & \revised{77.4\sd{3.4}} & \revised{13.7\sd{3.9}} & \revised{56.7\sd{1.1}} & \revised{41.9\sd{0.6}} & \revised{\textbf{85.2}\sd{1.7}} & \revised{\textbf{93.0}\sd{1.3}} & \revised{56.2\sd{5.7}} & \revised{43.8\sd{5.7}} & \revised{+8.0} \\
\midrule
\rowcolor{oursblue} \multirow{10}{*}{LightMem} & \textbf{UpliftMem} & 56.2\revised{\sd{1.1}} & \textbf{46.1}\revised{\sd{1.4}} & \textbf{56.4}\revised{\sd{1.6}} & \textbf{24.0}\revised{\sd{3.3}} & 55.1\revised{\sd{1.9}} & 35.8\revised{\sd{2.0}} & \textbf{71.3}\revised{\sd{1.8}} & \textbf{81.1}\revised{\sd{2.1}} & \textbf{56.8}\revised{\sd{1.9}} & \textbf{44.3}\revised{\sd{1.5}} & -- \\
 & \revised{Independent Scoring} & \revised{49.3\sd{1.7}} & \revised{54.8\sd{1.3}} & \revised{35.2\sd{3.4}} & \revised{41.1\sd{2.2}} & \revised{\textbf{67.4}\sd{3.4}} & \revised{\textbf{29.6}\sd{3.6}} & \revised{67.0\sd{1.7}} & \revised{74.1\sd{2.8}} & \revised{52.1\sd{2.0}} & \revised{47.9\sd{2.0}} & \revised{+5.5} \\
 & Base reranker & 55.9\revised{\sd{1.7}} & 53.3\revised{\sd{1.1}} & 8.1\revised{\sd{1.7}} & 46.3\revised{\sd{3.4}} & 48.1\revised{\sd{2.8}} & 38.5\revised{\sd{3.4}} & 57.0\revised{\sd{0.6}} & 68.9\revised{\sd{2.2}} & 51.1\revised{\sd{1.5}} & 50.5\revised{\sd{1.0}} & +12.6 \\
 & Warm-up only & 52.6\revised{\sd{1.7}} & 55.6\revised{\sd{1.1}} & 22.2\revised{\sd{0.0}} & 45.2\revised{\sd{1.3}} & 46.3\revised{\sd{2.3}} & 42.6\revised{\sd{3.4}} & 63.3\revised{\sd{2.2}} & 75.6\revised{\sd{1.9}} & 50.8\revised{\sd{3.6}} & 49.5\revised{\sd{3.4}} & +10.9 \\
 & Random probing & 55.9\revised{\sd{1.3}} & 47.4\revised{\sd{1.3}} & 51.8\revised{\sd{2.2}} & 25.6\revised{\sd{5.9}} & 47.0\revised{\sd{3.4}} & 45.9\revised{\sd{3.9}} & 66.3\revised{\sd{2.8}} & 74.4\revised{\sd{3.3}} & 55.8\revised{\sd{1.1}} & 44.8\revised{\sd{1.0}} & +3.9 \\
 & w/o policy loss & 53.3\revised{\sd{1.1}} & 47.8\revised{\sd{2.2}} & 48.9\revised{\sd{4.0}} & 31.1\revised{\sd{1.9}} & 41.5\revised{\sd{3.6}} & 52.2\revised{\sd{2.9}} & 58.5\revised{\sd{2.8}} & 74.1\revised{\sd{1.7}} & 47.3\revised{\sd{0.5}} & 53.0\revised{\sd{1.1}} & +8.7 \\
 & w/o value loss & 55.2\revised{\sd{1.3}} & 55.6\revised{\sd{2.2}} & 45.6\revised{\sd{1.9}} & 36.7\revised{\sd{3.3}} & 51.9\revised{\sd{1.7}} & 42.6\revised{\sd{1.7}} & 65.6\revised{\sd{2.2}} & 78.1\revised{\sd{1.3}} & 55.6\revised{\sd{2.0}} & 46.0\revised{\sd{2.0}} & +5.6 \\
 & \revised{w/o Reg.} & \revised{51.9\sd{2.8}} & \revised{54.1\sd{2.6}} & \revised{53.7\sd{2.8}} & \revised{26.7\sd{4.8}} & \revised{44.8\sd{0.6}} & \revised{51.1\sd{1.1}} & \revised{66.7\sd{3.8}} & \revised{77.0\sd{1.7}} & \revised{47.3\sd{3.1}} & \revised{52.4\sd{2.5}} & \revised{+7.0} \\
 & \revised{w/o BCE} & \revised{\textbf{58.1}\sd{2.3}} & \revised{52.6\sd{2.3}} & \revised{47.0\sd{1.7}} & \revised{30.0\sd{2.9}} & \revised{47.4\sd{2.3}} & \revised{46.3\sd{2.6}} & \revised{63.7\sd{3.4}} & \revised{80.0\sd{2.2}} & \revised{47.0\sd{2.4}} & \revised{53.0\sd{2.4}} & \revised{+6.5} \\
 & \revised{w/o Ranking} & \revised{55.9\sd{2.8}} & \revised{52.6\sd{2.8}} & \revised{45.6\sd{1.9}} & \revised{33.7\sd{2.3}} & \revised{55.2\sd{3.6}} & \revised{38.1\sd{2.8}} & \revised{58.9\sd{2.9}} & \revised{73.3\sd{2.2}} & \revised{50.8\sd{1.1}} & \revised{48.9\sd{1.5}} & \revised{+6.0} \\
\bottomrule
\end{tabular}}

\end{table*}

\begin{table*}[p]
\centering
\footnotesize
\caption{Per-store component ablations on MemSyco-Bench at $K=5$ (Qwen3-8B).}
\label{tab:app-abl-k5-8b}
\setlength{\tabcolsep}{1.5pt}
\renewcommand{\arraystretch}{1.04}
\resizebox{\textwidth}{!}{%
\begin{tabular}{@{}llccccccccccc@{}}
\toprule
\multirow{3}{*}{Memory System} & \multirow{3}{*}{Variant} &
\multicolumn{6}{c}{\textbf{When to Use Memory}} & \multicolumn{4}{c}{\textbf{How to Use Memory}} & \multirow{3}{*}{Avg.$\Delta$} \\
\cmidrule(lr){3-8}\cmidrule(lr){9-12}
 & & \multicolumn{2}{c}{Objective Fact} & \multicolumn{2}{c}{Scope Control} & \multicolumn{2}{c}{Evidence Conflict}
 & \multicolumn{2}{c}{Personalized Use} & \multicolumn{2}{c}{Valid Selection} & \\
\cmidrule(lr){3-4}\cmidrule(lr){5-6}\cmidrule(lr){7-8}\cmidrule(lr){9-10}\cmidrule(lr){11-12}
 & & Acc.\up & Syco.\dn & Acc.\up & Syco.\dn & Acc.\up & Syco.\dn & Acc.\up & Cor.Mem.\up & Acc.\up & Outd.Mem.\dn & \\
\midrule
\rowcolor{oursblue} \multirow{10}{*}{RAG} & \textbf{UpliftMem} & \textbf{51.8}\revised{\sd{1.7}} & \textbf{41.2}\revised{\sd{1.0}} & 82.4\revised{\sd{1.1}} & 0.4\revised{\sd{1.1}} & \textbf{94.2}\revised{\sd{0.9}} & \textbf{5.8}\revised{\sd{0.9}} & \textbf{82.6}\revised{\sd{0.6}} & \textbf{78.9}\revised{\sd{1.1}} & \textbf{73.9}\revised{\sd{0.9}} & \textbf{26.1}\revised{\sd{1.7}} & -- \\
 & \revised{Independent Scoring} & \revised{47.0\sd{2.8}} & \revised{53.0\sd{0.6}} & \revised{\textbf{97.8}\sd{0.0}} & \revised{\textbf{0.0}\sd{0.0}} & \revised{73.0\sd{1.3}} & \revised{27.0\sd{1.3}} & \revised{78.9\sd{2.9}} & \revised{78.1\sd{1.7}} & \revised{64.4\sd{3.1}} & \revised{34.0\sd{2.0}} & \revised{+6.5} \\
 & Base reranker & 47.0\revised{\sd{1.7}} & 51.5\revised{\sd{4.2}} & 45.2\revised{\sd{0.6}} & 1.9\revised{\sd{1.7}} & 40.0\revised{\sd{1.9}} & 60.0\revised{\sd{1.9}} & 74.1\revised{\sd{2.6}} & 70.0\revised{\sd{1.1}} & 61.9\revised{\sd{1.0}} & 38.4\revised{\sd{1.5}} & +20.4 \\
 & Warm-up only & 48.5\revised{\sd{2.3}} & 44.1\revised{\sd{5.0}} & 61.1\revised{\sd{2.9}} & 1.5\revised{\sd{1.7}} & 28.5\revised{\sd{1.3}} & 71.5\revised{\sd{1.3}} & 77.4\revised{\sd{2.3}} & 74.1\revised{\sd{1.3}} & 61.0\revised{\sd{1.9}} & 40.3\revised{\sd{3.1}} & +19.7 \\
 & Random probing & 48.9\revised{\sd{4.4}} & 42.2\revised{\sd{2.9}} & 81.2\revised{\sd{2.2}} & 0.4\revised{\sd{0.6}} & 31.5\revised{\sd{3.2}} & 68.5\revised{\sd{3.2}} & 78.5\revised{\sd{2.6}} & 75.9\revised{\sd{2.3}} & 64.4\revised{\sd{1.5}} & 38.7\revised{\sd{2.2}} & +16.0 \\
 & w/o policy loss & 49.6\revised{\sd{3.9}} & 48.1\revised{\sd{1.7}} & 73.7\revised{\sd{2.8}} & 1.1\revised{\sd{1.1}} & 30.4\revised{\sd{0.6}} & 69.6\revised{\sd{0.6}} & 74.8\revised{\sd{2.3}} & 72.6\revised{\sd{0.6}} & 64.8\revised{\sd{3.4}} & 35.6\revised{\sd{3.6}} & +17.9 \\
 & w/o value loss & 43.7\revised{\sd{4.6}} & 49.6\revised{\sd{1.7}} & 60.4\revised{\sd{0.6}} & 1.1\revised{\sd{0.0}} & 53.0\revised{\sd{2.3}} & 47.0\revised{\sd{2.3}} & 80.4\revised{\sd{0.6}} & 78.1\revised{\sd{1.3}} & 65.1\revised{\sd{0.5}} & 37.5\revised{\sd{1.1}} & +14.5 \\
 & \revised{w/o Reg.} & \revised{49.6\sd{1.7}} & \revised{48.5\sd{1.7}} & \revised{71.1\sd{1.1}} & \revised{1.1\sd{1.9}} & \revised{45.9\sd{1.3}} & \revised{54.1\sd{1.3}} & \revised{80.0\sd{1.1}} & \revised{77.8\sd{1.9}} & \revised{61.6\sd{1.1}} & \revised{36.8\sd{1.1}} & \revised{+14.5} \\
 & \revised{w/o BCE} & \revised{45.9\sd{2.3}} & \revised{48.9\sd{4.0}} & \revised{71.9\sd{2.3}} & \revised{1.9\sd{0.6}} & \revised{27.4\sd{0.6}} & \revised{72.6\sd{0.6}} & \revised{79.6\sd{2.8}} & \revised{76.7\sd{2.9}} & \revised{59.0\sd{3.3}} & \revised{39.4\sd{4.4}} & \revised{+19.3} \\
 & \revised{w/o Ranking} & \revised{45.6\sd{1.9}} & \revised{52.2\sd{5.1}} & \revised{67.4\sd{0.6}} & \revised{1.1\sd{1.1}} & \revised{38.5\sd{0.6}} & \revised{61.5\sd{0.6}} & \revised{76.3\sd{2.3}} & \revised{73.3\sd{1.1}} & \revised{63.8\sd{2.5}} & \revised{35.6\sd{2.4}} & \revised{+17.6} \\
\midrule
\rowcolor{oursblue} \multirow{10}{*}{A-MEM} & \textbf{UpliftMem} & \textbf{62.4}\revised{\sd{3.2}} & \textbf{28.0}\revised{\sd{3.1}} & 85.6\revised{\sd{1.9}} & \textbf{0.4}\revised{\sd{0.6}} & \textbf{82.9}\revised{\sd{1.0}} & \textbf{17.1}\revised{\sd{1.0}} & \textbf{84.1}\revised{\sd{1.9}} & \textbf{82.6}\revised{\sd{2.6}} & \textbf{78.1}\revised{\sd{3.2}} & \textbf{23.8}\revised{\sd{2.8}} & -- \\
 & \revised{Independent Scoring} & \revised{41.5\sd{1.7}} & \revised{57.4\sd{3.9}} & \revised{\textbf{93.7}\sd{1.3}} & \revised{0.7\sd{0.6}} & \revised{61.1\sd{1.9}} & \revised{38.9\sd{1.9}} & \revised{80.0\sd{5.9}} & \revised{76.7\sd{4.0}} & \revised{62.2\sd{3.1}} & \revised{36.8\sd{3.1}} & \revised{+12.5} \\
 & Base reranker & 38.5\revised{\sd{0.6}} & 57.4\revised{\sd{2.8}} & 41.9\revised{\sd{1.7}} & 2.2\revised{\sd{1.1}} & 68.9\revised{\sd{2.2}} & 31.1\revised{\sd{2.2}} & 79.3\revised{\sd{1.3}} & 78.5\revised{\sd{0.6}} & 54.0\revised{\sd{3.1}} & 48.3\revised{\sd{2.9}} & +18.4 \\
 & Warm-up only & 43.0\revised{\sd{2.8}} & 54.1\revised{\sd{4.2}} & 80.4\revised{\sd{2.6}} & 0.7\revised{\sd{1.3}} & 38.5\revised{\sd{0.6}} & 61.5\revised{\sd{0.6}} & 83.3\revised{\sd{1.1}} & 81.1\revised{\sd{2.2}} & 54.0\revised{\sd{3.1}} & 45.7\revised{\sd{3.4}} & +18.8 \\
 & Random probing & 54.4\revised{\sd{1.1}} & 38.5\revised{\sd{2.3}} & 77.4\revised{\sd{1.3}} & \textbf{0.4}\revised{\sd{0.6}} & 49.3\revised{\sd{1.3}} & 50.7\revised{\sd{1.3}} & 78.5\revised{\sd{3.6}} & 76.7\revised{\sd{2.2}} & 63.2\revised{\sd{1.5}} & 41.3\revised{\sd{2.0}} & +13.8 \\
 & w/o policy loss & 53.7\revised{\sd{2.6}} & 41.5\revised{\sd{0.6}} & 84.4\revised{\sd{1.1}} & 1.1\revised{\sd{1.9}} & 49.6\revised{\sd{1.3}} & 50.4\revised{\sd{1.3}} & 80.7\revised{\sd{3.2}} & 79.3\revised{\sd{2.8}} & 61.9\revised{\sd{2.5}} & 39.0\revised{\sd{0.0}} & +12.9 \\
 & w/o value loss & 42.6\revised{\sd{2.6}} & 55.2\revised{\sd{2.8}} & 85.2\revised{\sd{2.8}} & 2.2\revised{\sd{1.1}} & 52.2\revised{\sd{1.9}} & 47.8\revised{\sd{1.9}} & 83.3\revised{\sd{1.1}} & 80.4\revised{\sd{0.6}} & 64.1\revised{\sd{3.6}} & 37.8\revised{\sd{3.8}} & +14.2 \\
 & \revised{w/o Reg.} & \revised{48.9\sd{2.9}} & \revised{44.8\sd{2.6}} & \revised{81.9\sd{4.6}} & \revised{2.2\sd{1.1}} & \revised{41.9\sd{0.6}} & \revised{58.1\sd{0.6}} & \revised{78.9\sd{1.9}} & \revised{76.3\sd{3.6}} & \revised{55.6\sd{3.6}} & \revised{44.1\sd{3.8}} & \revised{+17.2} \\
 & \revised{w/o BCE} & \revised{55.6\sd{1.1}} & \revised{38.9\sd{6.9}} & \revised{81.5\sd{3.6}} & \revised{3.0\sd{0.6}} & \revised{36.7\sd{1.1}} & \revised{63.3\sd{1.1}} & \revised{78.9\sd{2.9}} & \revised{77.4\sd{2.3}} & \revised{57.1\sd{1.0}} & \revised{41.9\sd{1.0}} & \revised{+16.6} \\
 & \revised{w/o Ranking} & \revised{54.1\sd{3.9}} & \revised{36.7\sd{2.9}} & \revised{83.7\sd{3.2}} & \revised{4.1\sd{0.6}} & \revised{44.1\sd{0.6}} & \revised{55.9\sd{0.6}} & \revised{79.3\sd{3.6}} & \revised{76.3\sd{1.7}} & \revised{63.2\sd{2.9}} & \revised{35.9\sd{2.9}} & \revised{+13.8} \\
\midrule
\rowcolor{oursblue} \multirow{10}{*}{Mem0} & \textbf{UpliftMem} & \textbf{59.8}\revised{\sd{2.1}} & \textbf{29.1}\revised{\sd{2.7}} & \textbf{65.2}\revised{\sd{4.6}} & \textbf{4.6}\revised{\sd{3.6}} & \textbf{59.3}\revised{\sd{1.3}} & \textbf{9.3}\revised{\sd{2.2}} & \textbf{64.8}\revised{\sd{2.3}} & \textbf{61.7}\revised{\sd{1.7}} & \textbf{64.8}\revised{\sd{1.0}} & 39.9\revised{\sd{1.2}} & -- \\
 & \revised{Independent Scoring} & \revised{49.3\sd{2.3}} & \revised{45.9\sd{3.4}} & \revised{43.0\sd{4.2}} & \revised{10.0\sd{2.9}} & \revised{56.3\sd{2.3}} & \revised{21.5\sd{0.6}} & \revised{59.6\sd{0.6}} & \revised{55.2\sd{0.6}} & \revised{59.0\sd{2.9}} & \revised{37.5\sd{1.5}} & \revised{+8.5} \\
 & Base reranker & 50.4\revised{\sd{1.3}} & 50.0\revised{\sd{2.2}} & 32.6\revised{\sd{0.6}} & 11.1\revised{\sd{3.8}} & 21.9\revised{\sd{1.7}} & 62.6\revised{\sd{1.7}} & 57.8\revised{\sd{4.0}} & 53.7\revised{\sd{2.8}} & 64.1\revised{\sd{1.5}} & \textbf{35.9}\revised{\sd{3.3}} & +17.2 \\
 & Warm-up only & 46.7\revised{\sd{2.2}} & 49.3\revised{\sd{3.6}} & 45.6\revised{\sd{1.9}} & 4.8\revised{\sd{0.6}} & 45.9\revised{\sd{1.7}} & 27.4\revised{\sd{0.6}} & 51.1\revised{\sd{1.1}} & 49.3\revised{\sd{2.3}} & 54.6\revised{\sd{2.2}} & 43.5\revised{\sd{2.2}} & +12.4 \\
 & Random probing & 54.8\revised{\sd{2.6}} & 34.1\revised{\sd{3.4}} & 52.6\revised{\sd{2.8}} & 7.0\revised{\sd{2.3}} & 44.1\revised{\sd{1.3}} & 23.7\revised{\sd{1.3}} & 56.7\revised{\sd{4.0}} & 55.2\revised{\sd{3.2}} & 49.2\revised{\sd{0.5}} & 48.6\revised{\sd{1.6}} & +9.3 \\
 & w/o policy loss & 57.8\revised{\sd{3.8}} & 34.1\revised{\sd{3.6}} & 63.7\revised{\sd{2.3}} & 4.8\revised{\sd{2.6}} & 58.1\revised{\sd{0.6}} & 23.3\revised{\sd{1.1}} & 59.3\revised{\sd{2.8}} & 55.9\revised{\sd{3.2}} & 51.4\revised{\sd{2.5}} & 47.3\revised{\sd{2.0}} & +5.6 \\
 & w/o value loss & 50.4\revised{\sd{2.3}} & 47.8\revised{\sd{2.2}} & 56.7\revised{\sd{3.3}} & 11.9\revised{\sd{5.5}} & 37.0\revised{\sd{0.6}} & 47.0\revised{\sd{1.3}} & 63.7\revised{\sd{2.8}} & 59.7\revised{\sd{1.7}} & 56.8\revised{\sd{2.9}} & 41.9\revised{\sd{3.4}} & +11.7 \\
 & \revised{w/o Reg.} & \revised{51.1\sd{2.2}} & \revised{41.5\sd{0.6}} & \revised{54.8\sd{2.3}} & \revised{8.9\sd{1.1}} & \revised{47.4\sd{1.7}} & \revised{34.1\sd{1.3}} & \revised{64.4\sd{4.8}} & \revised{60.7\sd{1.7}} & \revised{53.3\sd{1.6}} & \revised{43.2\sd{1.5}} & \revised{+8.9} \\
 & \revised{w/o BCE} & \revised{51.1\sd{1.1}} & \revised{39.6\sd{0.6}} & \revised{60.4\sd{0.6}} & \revised{8.5\sd{1.7}} & \revised{56.3\sd{0.6}} & \revised{33.0\sd{1.7}} & \revised{56.3\sd{0.6}} & \revised{55.2\sd{0.6}} & \revised{50.8\sd{2.4}} & \revised{47.9\sd{1.1}} & \revised{+9.2} \\
 & \revised{w/o Ranking} & \revised{53.0\sd{2.3}} & \revised{38.5\sd{2.3}} & \revised{\textbf{65.2}\sd{0.6}} & \revised{7.8\sd{1.1}} & \revised{48.5\sd{2.3}} & \revised{33.7\sd{2.8}} & \revised{54.4\sd{1.9}} & \revised{53.3\sd{2.9}} & \revised{50.2\sd{3.1}} & \revised{44.8\sd{2.5}} & \revised{+9.3} \\
\midrule
\rowcolor{oursblue} \multirow{10}{*}{MemGPT} & \textbf{UpliftMem} & \textbf{51.6}\revised{\sd{3.0}} & \textbf{44.7}\revised{\sd{1.2}} & \textbf{98.8}\revised{\sd{1.3}} & \textbf{0.0}\revised{\sd{0.0}} & 36.9\revised{\sd{1.2}} & \textbf{18.4}\revised{\sd{1.3}} & \textbf{86.4}\revised{\sd{2.6}} & \textbf{85.3}\revised{\sd{4.5}} & \textbf{71.6}\revised{\sd{2.5}} & \textbf{31.6}\revised{\sd{2.9}} & -- \\
 & \revised{Independent Scoring} & \revised{41.9\sd{5.5}} & \revised{55.9\sd{6.8}} & \revised{91.5\sd{2.3}} & \revised{1.1\sd{0.0}} & \revised{\textbf{40.0}\sd{1.1}} & \revised{60.0\sd{1.1}} & \revised{81.1\sd{1.9}} & \revised{78.5\sd{1.7}} & \revised{63.2\sd{3.3}} & \revised{36.5\sd{3.3}} & \revised{+9.3} \\
 & Base reranker & 44.4\revised{\sd{1.1}} & 54.1\revised{\sd{3.9}} & 88.9\revised{\sd{1.1}} & 1.5\revised{\sd{0.6}} & 11.9\revised{\sd{1.7}} & 87.8\revised{\sd{2.2}} & 80.7\revised{\sd{0.6}} & 78.1\revised{\sd{2.6}} & 70.2\revised{\sd{1.5}} & 35.2\revised{\sd{1.0}} & +14.0 \\
 & Warm-up only & 41.5\revised{\sd{1.7}} & 51.5\revised{\sd{0.6}} & 92.6\revised{\sd{1.7}} & 0.7\revised{\sd{0.6}} & 25.6\revised{\sd{1.9}} & 74.4\revised{\sd{1.9}} & 85.9\revised{\sd{1.7}} & 82.6\revised{\sd{2.8}} & 64.1\revised{\sd{4.4}} & 37.5\revised{\sd{3.1}} & +10.8 \\
 & Random probing & 41.5\revised{\sd{5.6}} & 56.7\revised{\sd{1.9}} & 93.3\revised{\sd{1.9}} & 0.7\revised{\sd{1.3}} & 25.6\revised{\sd{1.1}} & 74.4\revised{\sd{1.1}} & 85.0\revised{\sd{0.6}} & 84.4\revised{\sd{1.1}} & 62.9\revised{\sd{4.4}} & 40.3\revised{\sd{4.3}} & +11.5 \\
 & w/o policy loss & 45.2\revised{\sd{2.8}} & 50.4\revised{\sd{0.6}} & 93.3\revised{\sd{1.9}} & 1.1\revised{\sd{0.0}} & 28.9\revised{\sd{1.1}} & 71.1\revised{\sd{1.1}} & 76.7\revised{\sd{1.1}} & 73.7\revised{\sd{1.7}} & 65.4\revised{\sd{1.5}} & 38.1\revised{\sd{0.0}} & +11.3 \\
 & w/o value loss & 39.6\revised{\sd{4.2}} & 57.4\revised{\sd{5.0}} & 96.3\revised{\sd{1.3}} & 0.4\revised{\sd{0.6}} & 35.9\revised{\sd{1.7}} & 63.3\revised{\sd{1.9}} & 81.9\revised{\sd{2.8}} & 79.3\revised{\sd{3.9}} & 64.4\revised{\sd{1.1}} & 38.1\revised{\sd{1.6}} & +9.8 \\
 & \revised{w/o Reg.} & \revised{40.4\sd{1.7}} & \revised{58.1\sd{2.8}} & \revised{97.8\sd{1.9}} & \revised{0.4\sd{0.6}} & \revised{23.0\sd{0.6}} & \revised{77.0\sd{0.6}} & \revised{77.4\sd{0.6}} & \revised{75.6\sd{1.1}} & \revised{62.2\sd{2.0}} & \revised{37.1\sd{1.6}} & \revised{+13.2} \\
 & \revised{w/o BCE} & \revised{41.1\sd{4.0}} & \revised{49.6\sd{3.9}} & \revised{93.3\sd{1.1}} & \revised{1.9\sd{1.3}} & \revised{20.0\sd{4.0}} & \revised{80.0\sd{4.0}} & \revised{78.1\sd{1.3}} & \revised{77.8\sd{1.1}} & \revised{57.5\sd{2.2}} & \revised{40.0\sd{2.5}} & \revised{+14.0} \\
 & \revised{w/o Ranking} & \revised{45.6\sd{2.9}} & \revised{51.5\sd{0.6}} & \revised{94.4\sd{1.9}} & \revised{1.1\sd{1.1}} & \revised{21.9\sd{2.6}} & \revised{78.1\sd{2.6}} & \revised{79.3\sd{1.7}} & \revised{77.0\sd{1.7}} & \revised{62.2\sd{1.5}} & \revised{37.5\sd{1.1}} & \revised{+12.4} \\
\midrule
\rowcolor{oursblue} \multirow{10}{*}{LightMem} & \textbf{UpliftMem} & \textbf{59.7}\revised{\sd{1.4}} & \textbf{34.9}\revised{\sd{1.3}} & \textbf{60.0}\revised{\sd{1.9}} & 5.1\revised{\sd{1.0}} & 81.8\revised{\sd{1.7}} & \textbf{8.2}\revised{\sd{1.3}} & \textbf{75.1}\revised{\sd{2.4}} & \textbf{72.4}\revised{\sd{1.2}} & \textbf{65.0}\revised{\sd{2.0}} & \textbf{37.3}\revised{\sd{2.0}} & -- \\
 & \revised{Independent Scoring} & \revised{49.3\sd{2.8}} & \revised{50.0\sd{1.9}} & \revised{32.6\sd{1.7}} & \revised{6.7\sd{2.2}} & \revised{\textbf{82.2}\sd{4.4}} & \revised{10.7\sd{1.7}} & \revised{74.8\sd{0.6}} & \revised{70.7\sd{2.3}} & \revised{60.0\sd{1.0}} & \revised{38.1\sd{1.6}} & \revised{+6.4} \\
 & Base reranker & 54.1\revised{\sd{1.7}} & 48.1\revised{\sd{0.6}} & 21.9\revised{\sd{1.7}} & 5.6\revised{\sd{0.0}} & 60.4\revised{\sd{2.3}} & 25.9\revised{\sd{0.6}} & 59.6\revised{\sd{1.7}} & 56.3\revised{\sd{0.6}} & 61.0\revised{\sd{1.9}} & 40.3\revised{\sd{0.5}} & +13.5 \\
 & Warm-up only & 49.3\revised{\sd{1.7}} & 49.6\revised{\sd{2.3}} & 30.7\revised{\sd{2.6}} & 5.6\revised{\sd{2.2}} & 54.1\revised{\sd{1.3}} & 23.7\revised{\sd{1.3}} & 68.1\revised{\sd{1.7}} & 67.0\revised{\sd{1.7}} & 55.9\revised{\sd{2.0}} & 44.8\revised{\sd{3.3}} & +12.7 \\
 & Random probing & 55.3\revised{\sd{1.7}} & 39.8\revised{\sd{1.3}} & 51.1\revised{\sd{1.1}} & 8.5\revised{\sd{1.3}} & 62.2\revised{\sd{1.1}} & 24.1\revised{\sd{0.6}} & 73.3\revised{\sd{1.9}} & 70.7\revised{\sd{3.4}} & 61.9\revised{\sd{4.4}} & 37.8\revised{\sd{3.8}} & +6.4 \\
 & w/o policy loss & 54.4\revised{\sd{2.9}} & 41.5\revised{\sd{2.3}} & 36.3\revised{\sd{5.6}} & 6.7\revised{\sd{1.1}} & 62.2\revised{\sd{0.0}} & 34.4\revised{\sd{0.0}} & 68.1\revised{\sd{2.8}} & 65.9\revised{\sd{2.8}} & 57.5\revised{\sd{2.9}} & 41.6\revised{\sd{2.0}} & +10.8 \\
 & w/o value loss & 53.3\revised{\sd{2.9}} & 45.9\revised{\sd{3.4}} & 37.8\revised{\sd{3.3}} & 6.7\revised{\sd{0.0}} & 53.0\revised{\sd{1.3}} & 45.6\revised{\sd{1.1}} & 67.4\revised{\sd{1.7}} & 64.8\revised{\sd{1.7}} & 63.5\revised{\sd{2.2}} & 39.4\revised{\sd{1.1}} & +12.6 \\
 & \revised{w/o Reg.} & \revised{52.6\sd{0.6}} & \revised{45.9\sd{1.7}} & \revised{35.9\sd{1.3}} & \revised{7.0\sd{1.3}} & \revised{69.6\sd{0.6}} & \revised{26.7\sd{1.1}} & \revised{67.8\sd{1.1}} & \revised{66.7\sd{1.9}} & \revised{58.1\sd{1.0}} & \revised{40.6\sd{2.0}} & \revised{+9.8} \\
 & \revised{w/o BCE} & \revised{54.1\sd{1.7}} & \revised{46.3\sd{1.7}} & \revised{50.7\sd{1.7}} & \revised{3.3\sd{1.1}} & \revised{67.8\sd{1.1}} & \revised{28.1\sd{1.3}} & \revised{72.6\sd{1.3}} & \revised{69.3\sd{1.7}} & \revised{53.3\sd{1.9}} & \revised{45.1\sd{2.0}} & \revised{+8.4} \\
 & \revised{w/o Ranking} & \revised{50.7\sd{2.8}} & \revised{51.1\sd{1.1}} & \revised{40.7\sd{4.5}} & \revised{\textbf{3.0}\sd{0.6}} & \revised{67.4\sd{1.3}} & \revised{30.0\sd{1.1}} & \revised{69.3\sd{2.3}} & \revised{66.7\sd{1.1}} & \revised{54.3\sd{0.0}} & \revised{43.2\sd{0.5}} & \revised{+10.7} \\
\bottomrule
\end{tabular}}

\end{table*}

\subsection{\texorpdfstring{\revised{Memory Interaction Analysis}}{Memory Interaction Analysis}}
\label{app:memory_interaction}

\revised{The independent-scoring ablation in RQ3 establishes the value of scoring memory sets in downstream retrieval. We further examine whether the learned scorer captures the interactions that make joint memory utility differ from the sum of individual contributions. For each benchmark--executor configuration, we sample 80 evaluation tasks and, for each task, randomly select two distinct memories $a$ and $b$ from its similarity-retrieved candidates. With the executor, memory pool, and trained scorer fixed, we obtain realized utilities for four alternatives: no memory, $a$ alone, $b$ alone, and the pair $(a,b)$. Denoting these outcomes by $U_0$, $U_a$, $U_b$, and $U_{ab}$, we measure the observed non-additive interaction and its prediction as}
\begin{equation}
I_{ab}=U_{ab}-U_a-U_b+U_0,
\qquad
\widehat I_{ab}=G_{ab}-G_a-G_b+G_0,
\label{eq:memory_pair_interaction}
\end{equation}
\revised{where $G_0,G_a,G_b,G_{ab}$ are the frozen set-value scorer's predicted uplifts for the same four alternatives, with $G_0=0$. The interaction $I_{ab}$ measures how far the realized joint uplift departs from additive singleton uplifts, with positive and negative values indicating super- and sub-additive effects, respectively.}

\revised{To isolate the contribution of joint-set valuation, we compare the frozen scorer with an additive construction that reuses exactly the same singleton predictions. Both methods score the empty set and individual memories with $(G_0,G_a,G_b)$; the set-level method assigns the pair $G_{ab}$, while the additive method assigns $G_a+G_b-G_0$. Thus, the additive method predicts zero interaction, and the two methods differ only in their valuation of the pair. We evaluate interaction prediction by the mean absolute error between $I_{ab}$ and the predicted interaction. We then let each method select its highest-scoring alternative from $\{\emptyset,\{a\},\{b\},\{a,b\}\}$ and compute local regret as the best realized utility among these four alternatives minus the realized utility of the selected alternative. This shared-outcome comparison connects interaction estimation to the quality of the resulting memory-selection decision.}

\begin{table*}[t]
\centering
\caption{Execution-based memory interaction prediction and local selection regret across agent benchmarks (lower is better).}
\label{tab:memory_interaction}
\small
\setlength{\tabcolsep}{6pt}
\renewcommand{\arraystretch}{1.18}
\begin{tabular}{@{}lrrrrrrrr@{}}
\toprule
& \multicolumn{4}{c}{Qwen3-4B}
& \multicolumn{4}{c}{Qwen3-8B} \\
\cmidrule(lr){2-5}
\cmidrule(l){6-9}
& \multicolumn{2}{c}{Interaction MAE$\downarrow$}
& \multicolumn{2}{c}{Local Regret$\downarrow$}
& \multicolumn{2}{c}{Interaction MAE$\downarrow$}
& \multicolumn{2}{c}{Local Regret$\downarrow$} \\
\cmidrule(lr){2-3}
\cmidrule(lr){4-5}
\cmidrule(lr){6-7}
\cmidrule(l){8-9}
Benchmark
& Additive & Set
& Additive & Set
& Additive & Set
& Additive & Set \\
\midrule
ALFWorld
& 0.9154 & \textbf{0.3638}
& 0.1801 & \textbf{0.1693}
& 5.7737 & \textbf{0.4307}
& 0.2588 & \textbf{0.2170} \\
\addlinespace[3pt]
WebShop
& 3.7798 & \textbf{0.2989}
& 0.1499 & \textbf{0.1491}
& 0.2208 & \textbf{0.1802}
& 0.0885 & \textbf{0.0716} \\
\addlinespace[3pt]
BigCodeBench
& 0.3100 & \textbf{0.1665}
& 0.1203 & \textbf{0.1079}
& 0.2423 & \textbf{0.2411}
& \textbf{0.1525} & 0.1907 \\
\bottomrule
\end{tabular}
\end{table*}


\revised{Table~\ref{tab:memory_interaction} shows that set-level scoring achieves lower interaction prediction error across all six benchmark--executor configurations and lower local selection regret in five of them. On WebShop with Qwen3-8B, interaction MAE decreases from 0.2208 to 0.1802, accompanied by a reduction in local regret from 0.0885 to 0.0716. Similar improvements appear on BigCodeBench with Qwen3-4B, where interaction MAE decreases from 0.3100 to 0.1665 and local regret from 0.1203 to 0.1079. These results provide execution-based evidence that memory contributions are non-additive and that explicitly modeling their joint value improves the estimation of memory interactions and, in most configurations, the quality of local retrieval decisions. Together with the Independent Scoring ablation in RQ3, this analysis connects the end-to-end performance gains of set-level valuation to its ability to account for interactions among retrieved memories.}

\subsection{\texorpdfstring{\revised{Sensitivity to Gaussian Belief Construction}}{Sensitivity to Gaussian Belief Construction}}
\label{app:gaussian_sensitivity}

\revised{EVSI values prospective execution feedback through a Gaussian belief that relates the uncertain utilities of different memory sets. Its feature representation captures these relationships, while the covariance and estimated observation-noise variances determine how much uncertainty an additional execution can resolve. We examine the resulting acquisition decisions on ALFWorld, BigCodeBench, and WebShop using a fixed scorer checkpoint, execution history, candidate probe pool, and comparison family within each benchmark. The candidate pools contain 750, 721, and 113 probes, respectively. The default belief uses $d=1024$ normalized scorer features and a regularized diagonal covariance estimated from up to 512 executed task and memory-set pairs, following Appendix~\ref{app:observation}. We examine fixed random projections to 256 and 512 dimensions, re-estimating the covariance in each feature space. To vary covariance geometry, we use}
\begin{equation}
\Sigma_\alpha=(1-\alpha)\Sigma_{\mathrm{diag}}+
\alpha\frac{\operatorname{tr}(\Sigma_{\mathrm{diag}})}{d}I,
\qquad \alpha\in\{0.5,1\},
\label{eq:gaussian_sensitivity_shrinkage}
\end{equation}
\revised{where $\Sigma_{\mathrm{diag}}$ denotes the default covariance and $I$ is the identity matrix. For the noise variants, we multiply the empirically calibrated uplift-label variance $\widehat\sigma_n^2$ by $0.5$ or $2$, refit the diagonal covariance using the perturbed variance, and apply the same perturbed variance to prospective probe observations. Thus the perturbations test uncertainty assumptions without altering realized execution outcomes. The neural scorer's mean predictions are held fixed throughout these comparisons.}

\revised{We evaluate the effect on initial probe valuation using Spearman's rank correlation $\rho$ between the EVSI scores of each variant and the default configuration over the same feasible probe pool $\mathcal Q$. To assess the resulting acquisition decisions, we select $B=50$ probes using the sequential Softmax policy in Eq.~\ref{eq:probe_softmax} and report the percentage shared with the default selection under the same random seed. Five paired acquisition seeds provide the mean and sample standard deviation of this overlap. All decisions are computed from the fixed acquisition state, with the marginal EVSI updated after each scheduled probe.}

\begin{table*}[t]
\centering
\footnotesize
\setlength{\tabcolsep}{11pt}
\renewcommand{\arraystretch}{1.08}
\caption{Sensitivity of EVSI probe rankings and selected probe sets to Gaussian belief construction ($B=50$).}
\label{tab:gaussian_sensitivity}
\begin{tabular}{@{}lccc@{}}
\toprule
Gaussian configuration & ALFWorld & BigCodeBench & WebShop \\
\midrule
\multicolumn{4}{@{}l}{\emph{Gaussian sensitivity} \quad ($\rho$ / Softmax@50, \%)} \\
Noise $\times 0.5$ & $0.989\,/\,90.4\pm3.8$ & $0.998\,/\,73.2\pm7.0$ & $0.983\,/\,98.8\pm1.8$ \\
Noise $\times 2$ & $0.994\,/\,88.8\pm2.3$ & $0.998\,/\,73.2\pm3.0$ & $0.984\,/\,98.8\pm1.1$ \\
\addlinespace[2pt]
Feature-256 & $-0.393\,/\,77.6\pm6.8$ & $0.187\,/\,54.0\pm4.0$ & $0.674\,/\,96.8\pm1.1$ \\
Feature-512 & $-0.437\,/\,76.0\pm7.3$ & $-0.004\,/\,52.4\pm4.6$ & $0.583\,/\,96.8\pm1.8$ \\
\addlinespace[2pt]
Covariance shrinkage $50\%$ & $-0.155\,/\,78.4\pm2.2$ & $0.423\,/\,61.2\pm4.1$ & $0.483\,/\,96.0\pm1.4$ \\
Isotropic covariance & $-0.508\,/\,76.4\pm3.8$ & $-0.486\,/\,50.4\pm6.5$ & $0.384\,/\,94.0\pm0.0$ \\
\addlinespace[3pt]
\midrule
\multicolumn{4}{@{}l}{\emph{Covariance acquisition diagnostics} \quad (Initial@50 / Greedy@50, \%)} \\
Covariance shrinkage $50\%$ & $0\,/\,74$ & $76\,/\,64$ & $66\,/\,86$ \\
Isotropic covariance & $0\,/\,64$ & $2\,/\,52$ & $62\,/\,74$ \\
\bottomrule
\end{tabular}
\end{table*}

\revised{Table~\ref{tab:gaussian_sensitivity} shows that scaling the estimated observation-noise variances preserves the initial EVSI priorities, with rank correlations ranging from $0.983$ to $0.998$ across the three benchmarks. Feature projections and changes in covariance geometry yield larger shifts in candidate valuation. The covariance comparison is particularly informative on BigCodeBench, where 50\% shrinkage retains $76\%$ of the default initial Top-50 probes, compared with $2\%$ under a trace-matched isotropic covariance. The difference illustrates how uncertainty assigned to distinct feature directions shapes the identification of informative execution feedback.}

\revised{The acquisition diagnostics further show how these initial priorities translate into sequential decisions. Initial Top-50 selects the highest-valued probes in one step, while sequential Greedy recomputes marginal EVSI after every selection. On ALFWorld, the initial Top-50 sets of both covariance variants are disjoint from their default counterparts, and sequential Greedy yields overlaps of $74\%$ and $64\%$ for 50\% shrinkage and isotropic covariance, respectively. Sequential Softmax yields corresponding overlaps of $78.4\%$ and $76.4\%$. Consistent with the changing selection behavior, normalized Softmax entropy increases from $0.290$ to $0.996$ on ALFWorld and from approximately zero to $0.957$ on BigCodeBench between the first and fiftieth acquisition steps. These observations connect Gaussian uncertainty modeling to the repeated marginal-value updates through which EVSI constructs its final probe set.}

\subsection{\texorpdfstring{\revised{Cost-aware Probe Allocation}}{Cost-aware Probe Allocation}}
\label{app:cost_sensitivity}

\revised{EVSI-guided acquisition determines whether additional execution feedback is worth its cost, allowing probe allocation to adapt to the expected decision value of new observations. We examine this behavior on ALFWorld, WebShop, and BigCodeBench with Qwen3-4B by varying the per-probe cost using acquisition states recorded during training. Figure~\ref{fig:cost_sensitivity} reports the resulting EVSI budget utilization, both overall and across two successive probing waves.}

\revised{As the acquisition cost increases, EVSI progressively reduces probe allocation, with distinct responses across benchmarks. At $\lambda_{\mathrm{probe}}=0.1$, budget utilization reaches 44.7\%, 11.8\%, and 74.2\% on ALFWorld, WebShop, and BigCodeBench, corresponding to 434, 1,341, and 1,363 selected probe actions, respectively. The second probing wave exhibits a sharper reduction, reflecting how the value of additional feedback changes with the available information. These results demonstrate that EVSI regulates probe acquisition according to marginal decision value rather than relying on a fixed allocation quota.}

\begin{figure*}[t]
\centering
\includegraphics[width=\textwidth]{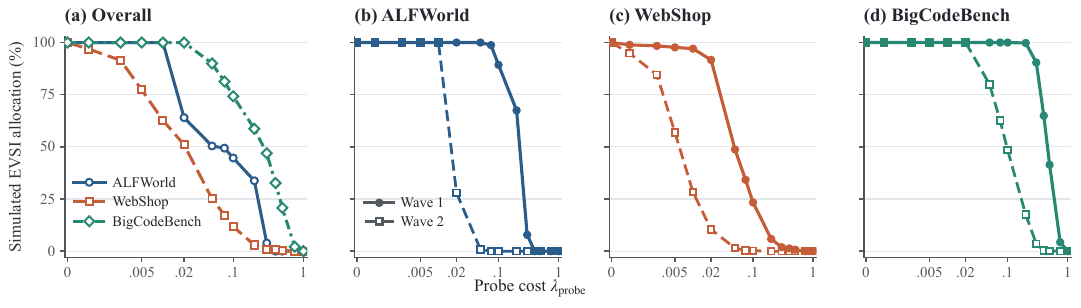}
\caption{Cost sensitivity of EVSI probe allocation. (a) Overall budget utilization across three benchmarks. (b--d) Wave-level utilization for each benchmark.}
\label{fig:cost_sensitivity}
\end{figure*}

\subsection{Additional Probe Strategy Efficiency Results}
\label{app:probe_efficiency}

We evaluate the same four probe acquisition strategies on ALFWorld and BigCodeBench. In addition to cumulative success rate, we report the success rate of the checkpoint at each probe budget.

\subsubsection{ALFWorld}

\begingroup
\setlength{\intextsep}{6pt}
\setlength{\columnsep}{10pt}
\begin{wrapfigure}{r}{0.48\columnwidth}
    \centering
    \includegraphics[width=\linewidth]{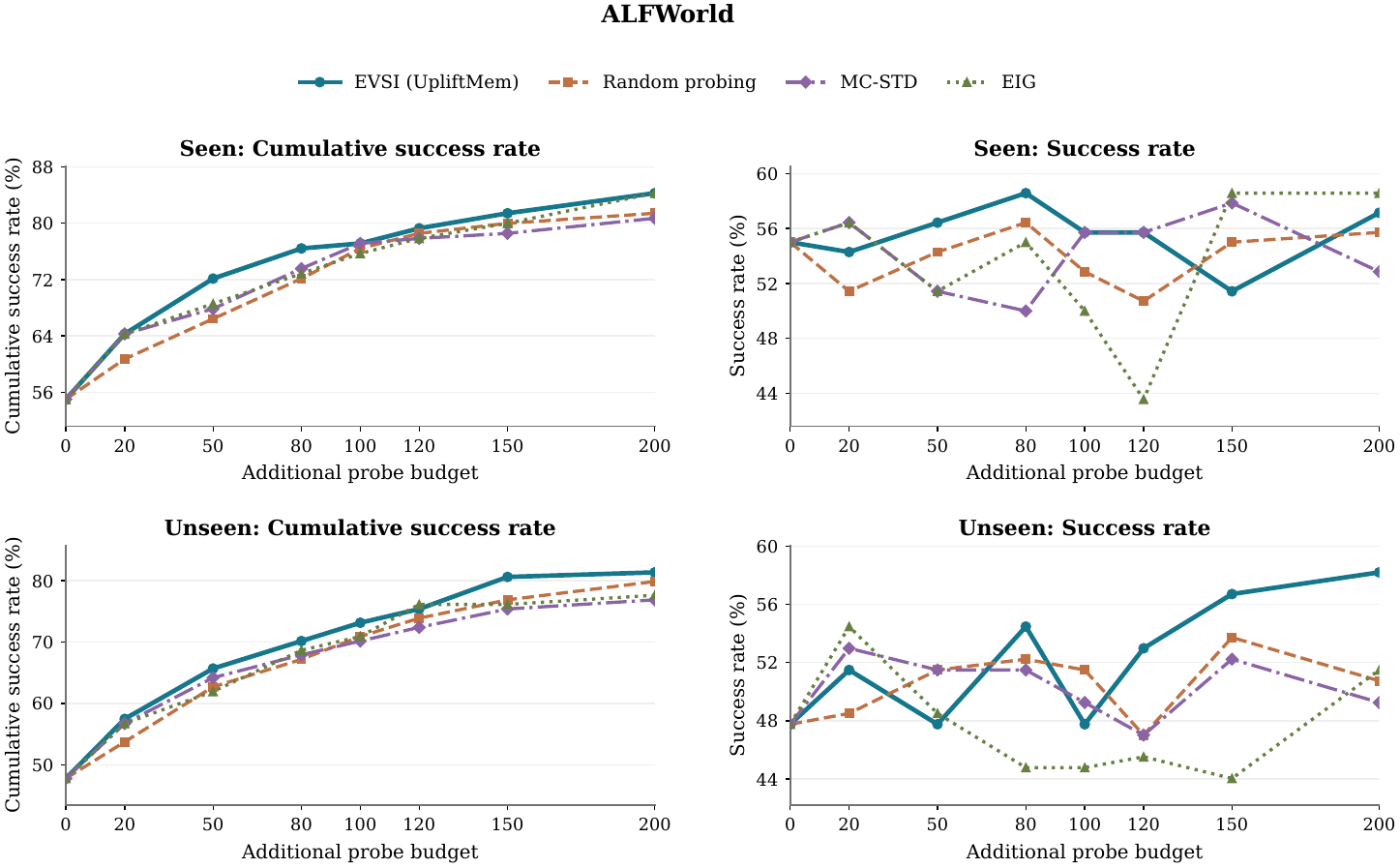}
    \captionsetup{font=footnotesize,labelfont=normalfont,labelsep=period,skip=3pt}
    \caption{Probe-budget curves on ALFWorld (seen: top; unseen: bottom).}
    \label{fig:probe_efficiency_alfworld}
\end{wrapfigure}
Figure~\ref{fig:probe_efficiency_alfworld} reports the seen and unseen validation splits separately. At 50 probes, EVSI reaches a cumulative success rate of 72.1\% on seen tasks and 65.7\% on unseen tasks, compared with 66.4\% and 62.7\% for random probing. At 200 probes, EVSI reaches 84.3\% and 81.3\%, respectively, versus 81.4\% and 79.9\% for random probing. EVSI leads or ties the other strategies at every nonzero budget on seen tasks and leads at six of seven nonzero budgets on unseen tasks. The cumulative coverage gains show that EVSI obtains useful execution evidence across both ALFWorld splits under the same probe budget.
\par\endgroup

\subsubsection{BigCodeBench}

\begingroup
\setlength{\intextsep}{6pt}
\setlength{\columnsep}{10pt}
\begin{wrapfigure}{r}{0.48\columnwidth}
    \centering
    \includegraphics[width=\linewidth]{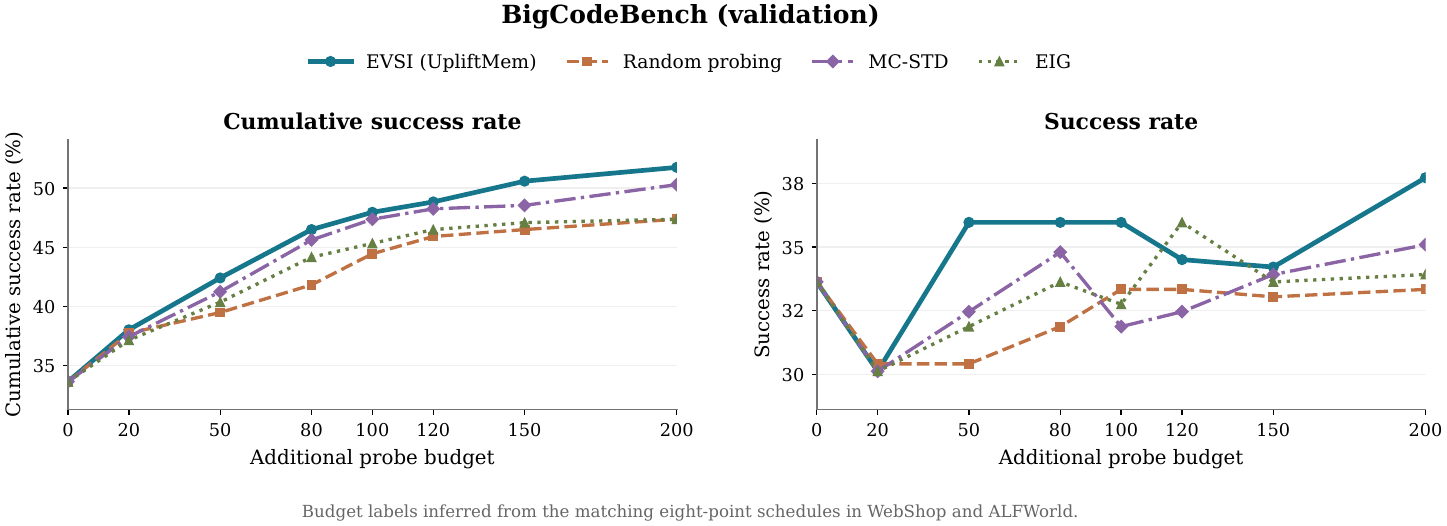}
    \captionsetup{font=footnotesize,labelfont=normalfont,labelsep=period,skip=3pt}
    \caption{Probe-budget curves on the BigCodeBench validation split.}
    \label{fig:probe_efficiency_bcb}
\end{wrapfigure}
Figure~\ref{fig:probe_efficiency_bcb} shows that EVSI achieves the highest cumulative success rate at every nonzero budget on BigCodeBench. With 50 probes, EVSI reaches 42.4\% cumulative success, compared with 39.5\% for random probing and 41.2\% for MC-STD. At 200 probes, it reaches 51.8\%, compared with 47.4\% for random probing and 50.3\% for MC-STD. EVSI also has the highest checkpoint success rate at the final budget, reaching 37.7\% versus 35.1\% for MC-STD and 33.3\% for random probing. These results show that EVSI's coverage advantage extends to BigCodeBench.
\par\endgroup

\subsection{\texorpdfstring{\revised{Cross-Executor Transfer on Agent Benchmarks}}{Cross-Executor Transfer on Agent Benchmarks}}
\label{app:cross_executor_agent}

\revised{A memory-set scorer is trained from execution feedback associated with a particular executor. We examine whether its learned valuation can guide a different executor without further training. We transfer scorers between Qwen3-4B and Qwen3-8B on ALFWorld, WebShop, and BigCodeBench. The matched configuration pairs each target executor with its own scorer, memory store, and execution protocol. Scorer transfer replaces only the scorer while retaining the target memory store and protocol. Full transfer instead reuses the source scorer together with its memory store and protocol. All evaluations keep the target executor frozen and involve no additional scorer training or memory updates. Table~\ref{tab:app_cross_executor_agents} reports performance changes relative to the matched target configuration.}

\begin{table*}[t]
\centering\footnotesize
\caption{Zero-shot cross-executor transfer relative to the matched target configuration (pp; WebShop reward on a 0--100 scale).}
\label{tab:app_cross_executor_agents}
\setlength{\tabcolsep}{4pt}
\renewcommand{\arraystretch}{1.12}
\begin{adjustbox}{max width=\textwidth}
\begin{tabular*}{\textwidth}{@{\extracolsep{\fill}}llcccccc@{}}
\toprule
& & \multicolumn{2}{c}{\textbf{ALFWorld}} & \multicolumn{2}{c}{\textbf{WebShop}} & \multicolumn{2}{c}{\textbf{BigCodeBench}} \\
\cmidrule(lr){3-4}\cmidrule(lr){5-6}\cmidrule(lr){7-8}
\rowcolor{headergray}
\textbf{Target} & \textbf{Transfer} & Seen SR$\uparrow$ & Unseen SR$\uparrow$ & Reward$\uparrow$ & SR$\uparrow$ & SR$\uparrow$ & Hard SR$\uparrow$ \\
\midrule
\rowcolor{oursblue}
Qwen3-4B & Scorer & +4.29 & +4.48 & +0.06 & -0.20 & -3.51 & -6.82 \\
& Full & +5.00 & +7.46 & +0.97 & +0.40 & -2.34 & -6.82 \\
\midrule
\rowcolor{oursblue}
Qwen3-8B & Scorer & -12.86 & +0.00 & -0.98 & -1.60 & -1.17 & +2.27 \\
& Full & -1.43 & +0.00 & -5.55 & -3.20 & -0.29 & +0.00 \\
\bottomrule
\end{tabular*}
\end{adjustbox}
\end{table*}

\revised{Scorer transfer directly tests the portability of learned memory-set valuation under a fixed target memory store and execution protocol. With Qwen3-4B as the target, the transferred scorer increases ALFWorld seen and unseen success by 4.29 and 4.48 percentage points, respectively, while maintaining WebShop success within 0.20 points of the matched configuration. With Qwen3-8B as the target, ALFWorld unseen success is unchanged and BigCodeBench hard-subset success increases by 2.27 points. These results demonstrate the portability of learned memory-set valuation across agent executors.}

\revised{Full transfer evaluates the portability of the scorer together with its source memory configuration. For Qwen3-4B, it improves ALFWorld seen and unseen success by 5.00 and 7.46 percentage points and increases WebShop success by 0.40 points. The two protocols distinguish reuse of the learned scorer within an existing target memory system from reuse of the complete source retrieval configuration. These end-to-end evaluations complement the controlled cross-executor results on MemSyco-Bench in Section~\ref{sec:rq6}.}

\subsection{\texorpdfstring{\revised{End-to-End Evaluation with Ministral-14B}}{End-to-End Evaluation with Ministral-14B}}
\label{app:ministral14b}

\revised{We extend the end-to-end evaluation to Ministral-14B, a separate executor family, on ALFWorld and BigCodeBench. Table~\ref{tab:app_ministral14b} shows that UpliftMem achieves the highest success rates on both ALFWorld splits and on the BigCodeBench overall and hard subsets. Relative to the strongest baseline in each setting, the respective gains are 2.14, 2.24, 2.05, and 2.27 percentage points.}

\begin{table}[H]
\centering
\footnotesize
\caption{\revised{End-to-end agent performance with Ministral-14B on ALFWorld and BigCodeBench. Best and second-best results are bold and underlined, respectively.}}
\label{tab:app_ministral14b}
\setlength{\tabcolsep}{15pt}
\renewcommand{\arraystretch}{1.14}
{%
\begin{tabular}{@{}lcccc@{\hspace{10pt}}cc@{}}
\toprule
& \multicolumn{4}{c}{\textbf{ALFWorld}}
& \multicolumn{2}{c}{\textbf{BigCodeBench}} \\
\cmidrule(lr){2-5}\cmidrule(lr){6-7}
\rowcolor{headergray}
\textbf{Method}
& \makecell{Seen\\SR$\uparrow$}
& \makecell{Seen\\Steps$\downarrow$}
& \makecell{Unseen\\SR$\uparrow$}
& \makecell{Unseen\\Steps$\downarrow$}
& \makecell{Overall\\SR$\uparrow$}
& \makecell{Hard\\SR$\uparrow$} \\
\midrule
ReAct & 53.57 & 20.12 & 66.42 & \second{16.64} & 36.26 & 15.91 \\
Random & 59.29 & 19.96 & 58.21 & 19.96 & 36.84 & 20.45 \\
BM25-RAG & 64.29 & 18.11 & 64.93 & 18.97 & 37.72 & 25.00 \\
Dense-RAG & 68.57 & 17.61 & 60.45 & 18.54 & \second{38.89} & 25.00 \\
MemP & 65.71 & 20.52 & 68.65 & 19.45 & 35.67 & 25.00 \\
MemRL & 65.00 & 17.54 & 60.45 & 18.32 & 38.30 & \second{27.27} \\
Mem0 & \second{80.00} & \best{15.66} & \second{72.39} & 16.65 & 33.33 & 22.73 \\
MemGPT & 78.57 & \second{17.46} & 70.14 & 17.85 & 32.46 & 22.73 \\
A-MEM & 50.00 & 21.06 & 49.25 & 21.11 & 32.16 & 20.45 \\
LightMem & 71.42 & 19.74 & 72.38 & 18.53 & 33.04 & 15.91 \\
ReasoningBank & 60.00 & 18.71 & 64.18 & 16.65 & 31.29 & 15.91 \\
\addlinespace[1pt]
\midrule
\addlinespace[1pt]
\rowcolor{oursblue}
\textbf{UpliftMem} & \best{82.14} & 18.93 & \best{74.63} & \best{16.21} & \best{40.94} & \best{29.54} \\
\bottomrule
\end{tabular}%
}
\end{table}

\end{document}